\RequirePackage{fix-cm}
\documentclass[twocolumn]{svjour3}          
\smartqed  
\usepackage{graphicx}
\usepackage{amssymb}
\usepackage{amsmath}
\usepackage{natbib}
\usepackage{microtype}
\usepackage{xcolor}
\usepackage[caption=false]{subfig}
\usepackage{algorithm}
\usepackage{algorithmicx}
\usepackage{algpseudocode}
\usepackage{balance}
\usepackage{gensymb}
\usepackage{hyperref}
\hypersetup{colorlinks=true}
\hypersetup{citecolor=blue}

\newcommand{\tn}[1]{\textnormal{#1}}

\newcommand{\mat}[0]{\begin{bmatrix}}
\newcommand{\mate}[0]{\end{bmatrix}}
\newcommand{\va}{\mathbf{a}}

\newcommand{\vc}{\mathbf{c}}
\newcommand{\vd}{\mathbf{d}}

\newcommand{\vf}{\mathbf{f}}
\newcommand{\vg}{\mathbf{g}}

\newcommand{\vp}{\mathbf{p}}

\newcommand{\vu}{\mathbf{u}}
\newcommand{\vv}{\mathbf{v}}

\newcommand{\vx}{\mathbf{x}}
\newcommand{\vy}{\mathbf{y}}

\newcommand{\cD}{\mathcal{D}}

\newcommand{\cG}{\mathcal{G}}

\newcommand{\cI}{\mathcal{I}}

\newcommand{\cN}{\mathcal{N}}
\newcommand{\cO}{\mathcal{O}}

\newcommand{\cS}{\mathcal{S}}

\newcommand{\cU}{\mathcal{U}}
\newcommand{\cV}{\mathcal{V}}
\newcommand{\cW}{\mathcal{W}}

\newcommand{\R}{\mathbb{R}}

\newcommand{\hx}{\hat{\mathbf{x}}}
\newcommand{\hp}{\hat{\mathbf{p}}}

\newcommand{\pr}{\textnormal{Pr}}
\newcommand{\Gau}{\mathcal{N}}
\newcommand\norm[1]{\left\|#1\right\|}              

\DeclareMathOperator*{\argmin}{arg\,min}            

\newcommand{\half}{\frac{1}{2}}

\newcommand{\rebuttal}[1]{{#1}}

\newcommand{\subparagraph}{}
\usepackage{titlesec}
\titlespacing*{\section}{0pt}{1.0\baselineskip}{0.7\baselineskip}
\titlespacing*{\subsection}{0pt}{0.9\baselineskip}{0.7\baselineskip}
\titlespacing*{\subsubsection}{0pt}{0.7\baselineskip}{0.5\baselineskip}
\titlespacing*{\paragraph}{0pt}{0.7\baselineskip}{0.5\baselineskip}

\begin{document}

\title{Decentralized Probabilistic Multi-Robot Collision Avoidance Using Buffered Uncertainty-Aware Voronoi Cells
\thanks{This work was supported in part by the Netherlands Organization for Scientific Research (NWO) domain Applied Sciences (Veni 15916) and the U.S. Office of Naval Research Global (ONRG) NICOP-grant N62909-19-1-2027. We are grateful for their support.}
\thanks{A video of the experimental results is available at \url{https://youtu.be/5F3fjjgwCSs}}
}


\author{Hai Zhu         \and
        Bruno Brito     \and
        Javier Alonso-Mora
}


\institute{ Hai Zhu \at
                Department of Cognitive Robotics, Delft University of Technology, Mekelweg 2, 2628 CD Delft, The Netherlands \\
                \email{h.zhu@tudelft.nl}           
           \and
            Bruno Brito \at
                Department of Cognitive Robotics, Delft University of Technology, Mekelweg 2, 2628 CD Delft, The Netherlands \\
                \email{bruno.debrito@tudelft.nl} 
            \and 
            Javier Alonso-Mora \at 
                Department of Cognitive Robotics, Delft University of Technology, Mekelweg 2, 2628 CD Delft, The Netherlands \\
                \email{j.alonsomora@tudelft.nl}
}

\date{Received: date / Accepted: date}

\maketitle              

\begin{abstract}
    In this paper, we present a decentralized and communication-free collision avoidance approach for multi-robot systems that accounts for both robot localization and sensing uncertainties. The approach relies on the computation of an uncertainty-aware safe region for each robot to navigate among other robots and static obstacles in the environment, \rebuttal{under the assumption of Gaussian-distributed uncertainty}. In particular, at each time step, we construct a chance-constrained buffered uncertainty-aware Voronoi cell (B-UAVC) for each robot given a specified collision probability threshold. Probabilistic collision avoidance is achieved by constraining the motion of each robot to be within its corresponding B-UAVC, i.e. the collision probability between the robots and obstacles remains below the specified threshold. The proposed approach is decentralized, communication-free, scalable with the number of robots and robust to robots' localization and sensing uncertainties. We applied the approach to single-integrator, double-integrator, differential-drive robots, and robots with general nonlinear dynamics. Extensive simulations and experiments with a team of ground vehicles, quadrotors, and heterogeneous robot teams are performed to analyze and validate the proposed approach. 
\end{abstract}

\section{Introduction}\label{sec:intro}

Multi-robot collision avoidance in cluttered environments is a fundamental problem when deploying a team of autonomous robots for applications such as coverage \citep{Breitenmoser2016}, target tracking \citep{zhou2018resilient}, formation flying \citep{Zhu2019ICRA} and multi-view cinematography \citep{Nageli2017multiple}. Given the robot current states and goal locations, the objective is to plan a local motion for each robot to navigate towards its goal while avoiding collisions with other robots and obstacles in the environment. Most existing algorithms solve the problem in a deterministic manner, where the robot states and obstacle locations are perfectly known. Practically, however, robot states and obstacle locations are generally obtained by an estimator based on sensor measurements that have noise and uncertainty. Taking this uncertainty into consideration is of utmost importance for safe and robust multi-robot collision avoidance. 

In this paper, we present a decentralized probabilistic approach for multi-robot collision avoidance under localization and sensing uncertainty that does not rely on communication. Our approach is built on the buffered Voronoi cell (BVC) method developed by \citet{Zhou2017}. The BVC method is designed for collision avoidance among multiple single-integrator robots, where each robot only needs to know the positions of neighboring robots. We extend the method into probabilistic scenarios considering robot localization and sensing uncertainties by mathematically formalizing a buffered uncertainty-aware Voronoi cell (B-UAVC). Furthermore, we consider static obstacles with uncertain locations in the environment. We apply our approach to double-integrator dynamics, differential-drive robots, and general high-order dynamical robots.

\subsection{Related Works}\label{sec:relatedWork}

\subsubsection{Multi-robot collision avoidance}
The problem of multi-robot collision avoidance has been well studied for deterministic scenarios, where the robots' states are precisely known. One of the state-of-the-art approaches is the reciprocal velocity obstacle (RVO) method \citep{VandenBerg2008}, which builds on the concept of velocity obstacles (VO) \citep{Fiorini1998}. The method models robot interaction pairwise in a distributed manner and estimates future collisions as a function of relative velocity. Based on the basic framework, RVO has been extended towards several revisions: the optimal reciprocal collision-avoidance (ORCA) method \citep{VanDenBerg2011} casting the problem into a linear programming formulation which can be solved efficiently, the generalized RVO method \citep{Bareiss2015} applying for heterogeneous teams of robots, and the $\varepsilon$-cooperative collision avoidance ($\varepsilon$CCA) method \citep{Alonso-Mora2018} accounting for the cooperation of nonholonomic robots. In addition to those RVO-based methods, the model predictive control (MPC) framework has also been widely used for multi-robot collision avoidance, which includes decentralized MPC \citep{Shim2003}, decoupled MPC \citep{Chen2015}, and sequential MPC \citep{Morgan2016, Luis2020}. 
While those approaches typically require the robots position and velocity, or more detailed future trajectory information to be known among neighboring robots, the recent developed buffered Voronoi cell (BVC) method \citep{Zhou2017,pierson2020weighted} only requires the robots to know the 
positions of other robots. 
In this paper, we build upon the concept of BVC and extend it to probabilistic scenarios, where each robot only needs to estimate the positions of its neighboring robots. 

\subsubsection{Collision avoidance under uncertainty}
Some of the above deterministic collision avoidance approaches have been extended to scenarios where robot localization or sensing uncertainty is considered. 
Based on RVO, the COCALU method \citep{Claes2012} takes into account bounded localization uncertainty of the robots by constructing an error-bounded convex hull of the VO of each robot. 
\citet{Gopalakrishnan2017} presents a probabilistic RVO method for single-integrator robots. 
\citet{Kamel2017} presents a decentralized MPC where robot motion uncertainty is taken into account by enlarging the robots with their 3-$\sigma$ confidence ellipsoids. 
A chance constrained MPC problem was formulated by \citet{Lyons2012} for planar robots, where rectangular regions were computed and inter-robot collision avoidance was transformed to avoid overlaps of those regions. 
Using local linearization, \citet{Zhu2019RAL} proposed a chance constrained nonlinear MPC (CCNMPC) method to ensure that the probability of inter-robot collision is below a specified threshold. 

Among these attempts to incorporate uncertainty into multi-robot collision avoidance, several limitations are observed. 
Probabilistic VO-based methods are limited to systems with simple first-order dynamics, or limited to homogeneous teams of robots. 
Probabilistic MPC-based methods typically demand communication of the planned trajectory of each robot to guarantee collision avoidance, which does not scale well with the number of robots in the system. 
An alternative to communicating trajectories is to assume that all other robots move with constant velocity \citep{Kamel2017}, which has been shown to lead to collisions in cluttered environments \citep{Zhu2019RAL}. 
Recently, \citet{luo2020multi} proposes probabilistic safety barrier certificates (PrSBC) to define the space of admissible control actions that are probabilistic safe, but it is only designed for single-integrator robots. In this paper, we define the probabilistic safe region for each robot directly based on the concept of buffered Voronoi cell (BVC). 

\rebuttal{
The BVC method has also been extended to probabilistic scenarios by \cite{Wang2019}. Taking into account the robot measurement uncertainty of other robots, they present the probabilistic buffered Voronoi cell (PBVC) to assure a safety level given a collision probability threshold. However, since the PBVC of each robot does not have an analytic solution, they employ a sampling-based approach to approximate it. In contrast, our proposed B-UAVC has an explicit and analytical form, which is more efficient to be computed. Moreover, our B-UAVC can be incorporated with MPC to handle general nonlinear systems, while the PBVC method developed by \cite{Wang2019} cannot be directly applied within a MPC framework. }

\subsubsection{Spatial decomposition in motion planning}
Our method constructs a set of local safe regions for the robots, which decompose the workspace. Spatial decomposition is broadly used in robot motion planning. \citet{Deits2015} proposes the IRIS (iterative regional inflation by semi-definite programming) algorithm to compute safe convex regions among obstacles given a set of seed points. The algorithm is then used for UAV path planning \citep{deits2015efficient} and multi-robot formation control \citep{Zhu2019ICRA}. \citet{Liu2017} presents a simpler but more efficient iteratively inflation algorithm to compute a convex polytope around a line segment among obstacles and utilizes it to construct a safe flight corridor for UAV navigation \citep{tordesillas2019faster}. Similar safe flight corridors are constructed for trajectory planning of quadrotor swarms \rebuttal{\citep{Honig2018}}, by computing a set of max-margin separating hyperplanes between a line segment and convex polygonal obstacles. The max-margin separating hyperplanes are also used by \citet{Arslan2019} to construct a local robot-centric safe region in convex sphere worlds for sensor-based reactive navigation. While those spatial decomposition methods have shown successful application in robot motion planning, they all assume perfect knowledge on robots and obstacles positions. In this paper, we consider both the robot localization and obstacle position uncertainty and construct a local uncertainty-aware safe region for each robot.

\subsection{Contribution}
The main contribution of this paper is a decentralized and communication-free method for probabilistic multi-robot collision avoidance in cluttered environments. The method considers robot localization and sensing uncertainties and relies on the computation of buffered uncertainty-aware Voronoi cells (B-UAVC). At each time step, each robot computes its B-UAVC based on the estimated position and uncertainty covariance of itself, neighboring robots and obstacles, and plans its motion within the B-UAVC. Probabilistic collision avoidance is ensured by constraining each robot's motion to be within its corresponding B-UAVC, such that the inter-robot and robot-obstacle collision probability is below a user-specified threshold.

An earlier version of this paper was published by \citet{Zhu2019MRS}. In this version, three main additional extensions are developed: a) we further consider static obstacles with uncertain locations in the environment; b) we extend the approach to double-integrator dynamics and differential-drive robots and c) we provide thorough simulation and experimental results and analyses. 

\subsection{Organization}
The remaining of this paper is organized as follows. In Section \ref{sec:preliminary} we present the problem statement and briefly summarize the concept of BVC. In Section \ref{sec:method_1} we formally introduce the buffered uncertainty-aware Voronoi cell (B-UAVC) and its construction method. We then describe how the B-UAVC is used for probabilistic multi-robot collision avoidance in Section \ref{sec:method_2}. Simulation and experimental results are presented in Section \ref{sec:sim_result} and Section \ref{sec:exp_result}, respectively. Finally, Section \ref{sec:conclsuion} concludes the paper. 

\section{Preliminaries}\label{sec:preliminary}

Throughout this paper vectors are denoted in bold lowercase letters, $\vx$, matrices in plain uppercase $M$, and sets in calligraphic uppercase, $\cS$. \rebuttal{$I$ indicates the identity matrix.} A superscript $ \vx^T $ denotes the transpose of $ \vx $. $\norm{\vx}$ denotes the Euclidean norm of $\vx$ and $\norm{\vx}_Q^{2} = \vx^TQ\vx$ denotes the weighted square norm. A hat $\hx$ denotes the mean of a random variable $\vx$. $\pr(\cdot)$ indicates the probability of an event and $p(\cdot)$ indicates the probability density function.

\subsection{Problem Statement}
Consider a group of $n$ robots operating in a $d$-dimensional space $\cW \subseteq \R^d$, where $d~\rebuttal{\in}~\{2, 3\}$, populated with $m$ static polygonal obstacles. 
For each robot $i \in \cI = \{1,\dots,n\}$, $\vp_i \in \R^d$ denotes its position, $\vv_i = \dot{\vp}_i$ its velocity and $\va_i = \dot{\vv}_i$ its acceleration. 
Let $\cG = \{\vg_1,\dots,\vg_n\}$ denote their goal locations.
\rebuttal{A safety radius $r_s$ is given for all robots.}
We consider that the position of each robot is obtained by a state estimator and is described as a Gaussian distribution with covariance $\Sigma_i$, i.e. $\vp_i \sim \cN(\hp_i, \Sigma_i)$. 
\rebuttal{We also consider static polytope obstacles with known shapes but uncertain locations. 
For each obstacle $o\in \cI_o = \{1,\dots,m\}$, denote by $\hat{\cO}_o \subset \R^d$ its occupied space when located at the expected (mean) position. $\hat{\cO}_o$ is given by a set of vertices. Hence, the space actually occupied by the obstacle can be written as $\cO_o = \{\vx+\vd_o~|~\vx\in\hat{\cO}_o,\vd_o \sim \Gau(0, \Sigma_o) \} \subset \R^d$, where $\vd_o $ is the uncertain translation of the obstacle's position, which has a zero mean and covariance $\Sigma_o$. 
}

A robot $i$ in the group is collision free with another robot $j$ if their distance is greater than the sum of their radii, i.e. \rebuttal{$\tn{dis}(\vp_i,\vp_j) \geq 2r_s$}
and with the obstacle $o$ if the minimum distance between the robot and the obstacles is larger than its radius, i.e. \rebuttal{$\tn{dis}(\vp_i, \cO_o) \geq r_s$}.
\rebuttal{
The distance function $\tn{dis}(\cdot)$ between a robot with another robot or an obstacle are defined as $\tn{dis}(\vp_i,\vp_j) = \norm{\vp_i - \vp_j}$, and $\tn{dis}(\vp_i, \cO_o) = \min_{\vp \in \cO_o}\norm{\vp_i - \vp}$, respectively. 
}
Note that the robots' and obstacles' positions are random variables following Gaussian distributions, which have an infinite support. Hence, the collision-free condition can only be satisfied in a probabilistic manner, which is defined as a chance constraint as follows.

\begin{definition}[Probabilistic Collision-Free]\label{def:pcollfree}
    A robot $i$ at position $\vp_i \sim \cN(\hp_i, \Sigma_i)$ is probabilistic collision-free with a robot $j$ at position $\vp_j \sim \cN(\hp_j, \Sigma_j)$ and an obstacle $o$ at position $\vp_o \sim \cN(\hp_o, \Sigma_o)$ if 
    \begin{align}
            \pr(\rebuttal{\tn{dis}(\vp_i,\vp_j) \geq 2r_s}) &\geq 1 - \delta, ~~\forall j\in\cI, j\neq i, \label{eq:chanceConRobot} \\
            \pr(\rebuttal{\tn{dis}(\vp_i, \cO_o) \geq r_s}) &\geq 1 - \delta, ~~\forall o \in \cI_o, \label{eq:chanceConObs}
    \end{align}
    where $\delta$ is the collision probability threshold for inter-robot and robot-obstacle collisions. 
\end{definition}

The objective of probabilistic collision avoidance is to compute a local motion plan, $\vu_i$, for each robot in the group, that respects its kinematic and dynamical constraints, makes progress towards its goal location, and is probabilistic collision free with other robots as well as obstacles in the environment. 
In this paper, we first consider single-integrator dynamics for the robots,
\begin{equation}
    \dot{\vp}_i = \vu_i,
\end{equation}
and then extend it to double-integrator systems, differential-drive robots and robots with general high-order dynamics.

\subsection{Buffered Voronoi Cell}\label{subsec:vc}
The key idea of our proposed method is to compute an uncertainty-aware collision-free region for each robot in the system, which is a major extension of the deterministic buffered Voronoi cell (BVC) method \citep{Zhou2017,pierson2020weighted}. In this section, we briefly describe the concept of BVC.

For a set of deterministic points $(\vp_1, \dots, \vp_n) \in \R^d$, the standard Voronoi cell (VC) of each point $i \in \cI$ is defined as \citep{Okabe2009}
\begin{equation}\label{eq:vcDef}
    \cV_i = \{\vp\in\R^d : \norm{\vp- \vp_i} \leq \norm{\vp - \vp_j}, \forall j\neq i \},
\end{equation}
which can also be written as
\begin{equation}\label{eq:vcPlane}
    \cV_i = \{ \vp\in\R^d: \vp_{ij}^T\vp \leq \vp_{ij}^T\frac{\vp_i+\vp_j}{2}, \forall j\neq i \},
\end{equation}
where $\vp_{ij} = \vp_j - \vp_i$.
It can be observed that $\cV_i$ is the intersection of a set of hyperplanes which separate point $i$ with any other point $j$ in the group, as shown in Fig. \ref{subfig:vc}. Hence, VC can be obtained by computing the separating hyperplanes between each pair of points. 

To consider the footprints of robots, a buffered Voronoi cell for each robot $i$ is defined as follows:
\begin{equation}\label{eq:bvcPlane}
    \cV_i^b = \{ \vp\in\R^d: \vp_{ij}^T\vp \leq \vp_{ij}^T\frac{\vp_i+\vp_j}{2} - r_s\norm{\vp_{ij}}, \forall j\neq i \},
\end{equation}
which is obtained by retracting the edges of the VC with a safety distance (buffer) \rebuttal{$r_s$}. 

\rebuttal{In deterministic scenarios, if the robots are mutually collision-free, then the BVC of each robot is a non-empty set \citep{Zhou2017}.}
It is also trivial to prove that the BVCs are disjoint and if the robots are within their corresponding BVCs individually, they are collision free with each other. Using the concept of BVC, \citet{Zhou2017} proposed a control policy for a group of single-integrator robots whose control inputs are velocities. Each robot can safely and continuously navigate in its BVC, given that other robots in the system also follow the same rule. However, the guarantee does not hold for double-integrator dynamics or non-holonomic robots such as differential-drive robots. 

\begin{figure*}[t]
    \centering
    \rebuttal{
    \subfloat[]{\label{subfig:vc}
       \includegraphics[width=.20\textwidth]{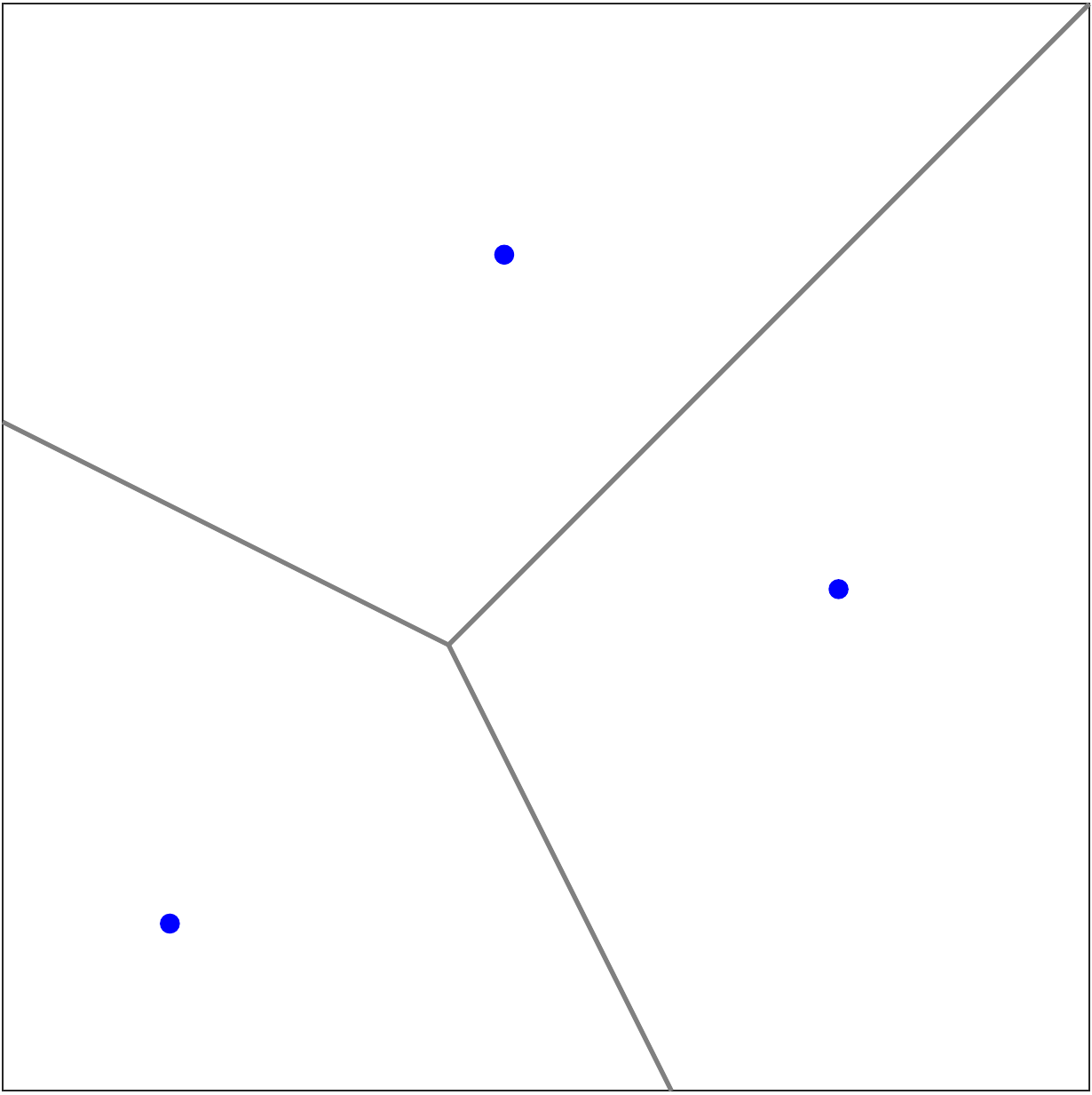}}}
    \quad
    \rebuttal{
    \subfloat[]{\label{subfig:uavc}
       \includegraphics[width=.20\textwidth]{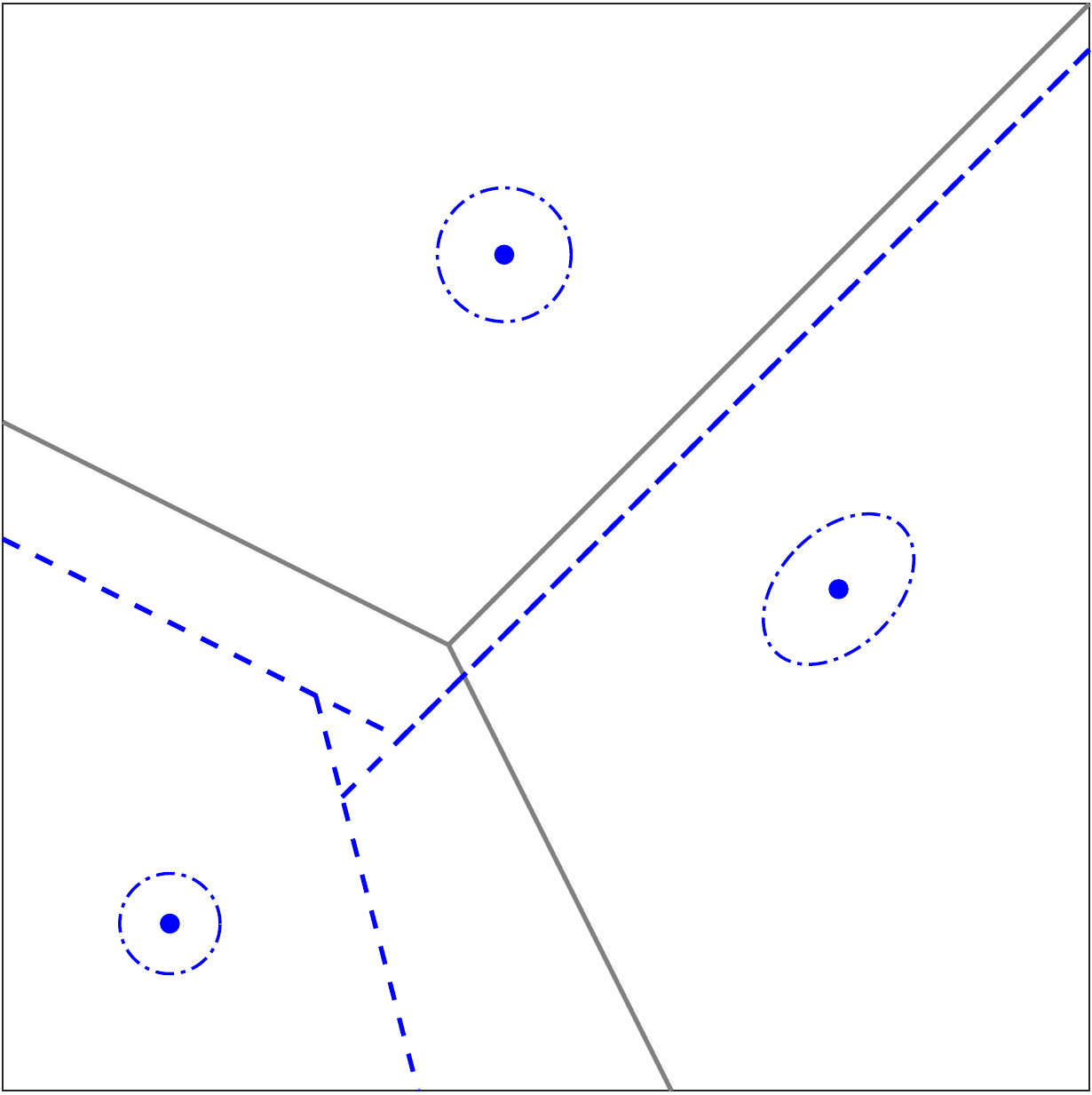}}}
    \quad
    \subfloat[]{\label{subfig:buavc1}
        \includegraphics[width=.20\textwidth]{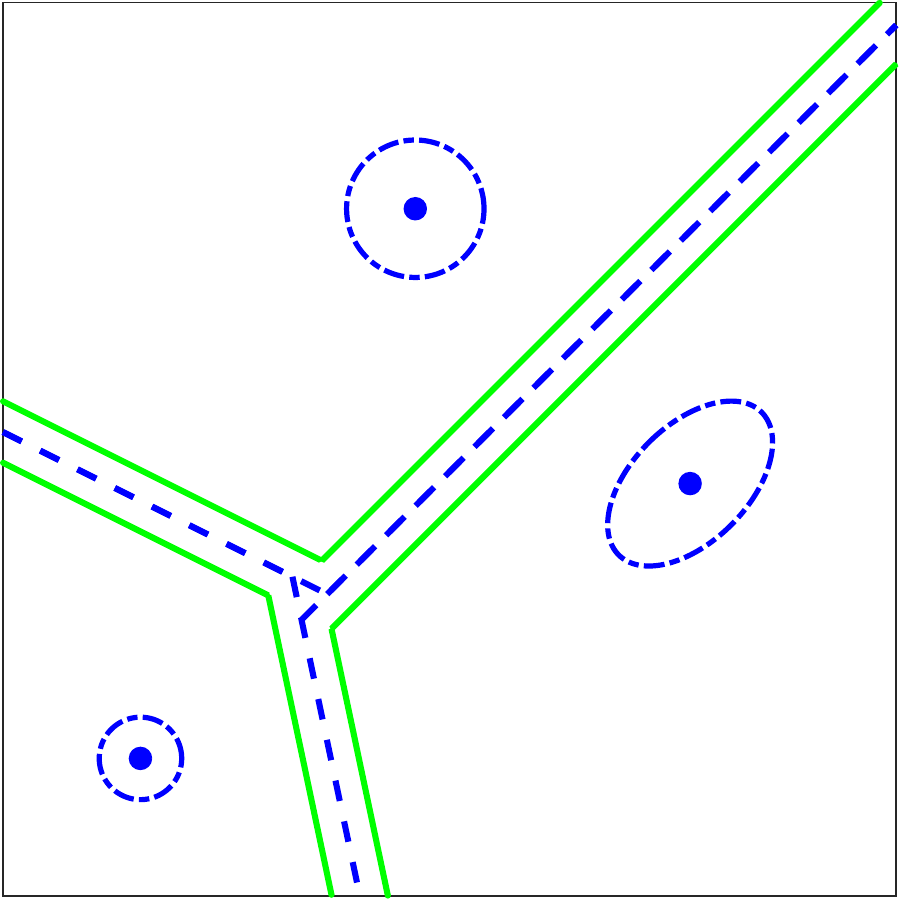}}
    \quad
    \subfloat[]{\label{subfig:buavc2}
        \includegraphics[width=.20\textwidth]{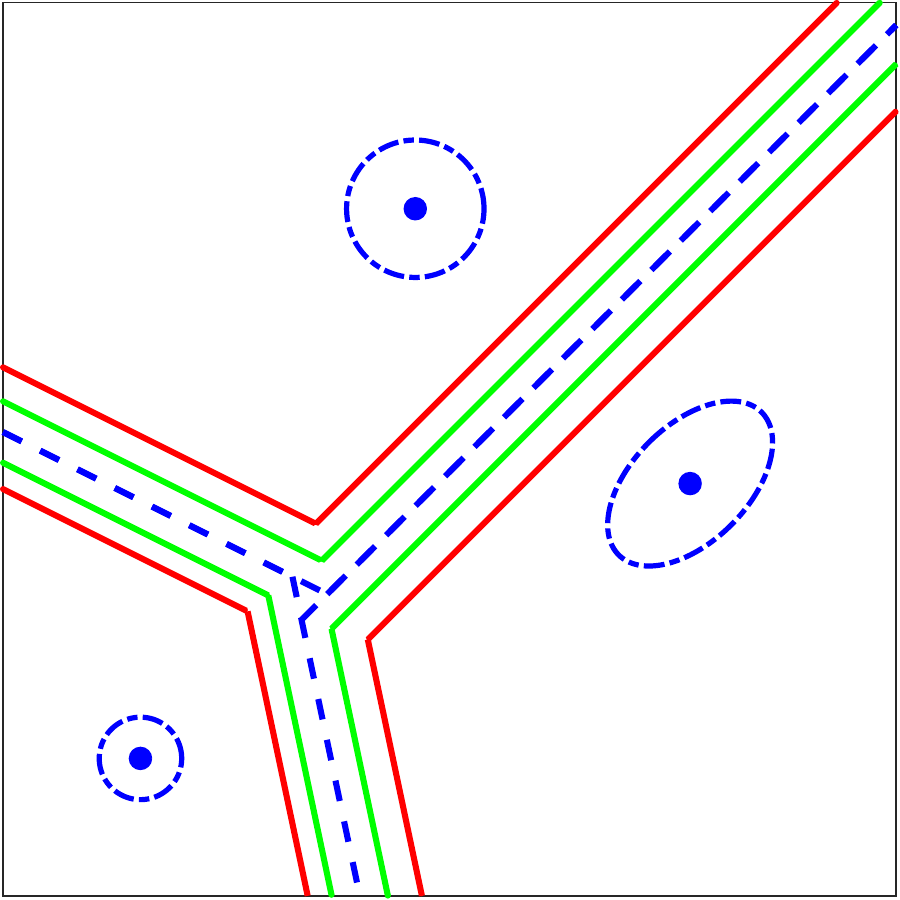}}

    \caption{Example of buffered uncertainty-aware Voronoi cells (B-UAVC). Blue dots are robots; blue dash-dot ellipses indicate the 3-$\sigma$ confidence ellipsoid of the position uncertainty. (a) Deterministic Voronoi cell (VC, the boundary in \rebuttal{gray solid} line). (b) Uncertainty-aware Voronoi cell based on the best linear separators (UAVC, the boundary in blue dashed line). (c) UAVC with robot raidus buffer (the boundary in green solid line). (d) Final B-UAVC with robot radius and collision probability buffer (the boundary in red solid line). }\label{fig:buavc}
\end{figure*}

\subsection{Shadows of Uncertain Obstacles}
To account for uncertain obstacles in the environment, we rely on the concept of obstacle shadows introduced by \citet{Axelrod2018}. The $\epsilon$-shadow is defined as follows:
\begin{definition}[$\epsilon$-Shadow]\label{def:shadow}
    A set $\cS_o \subseteq \R^d$ is an $\epsilon$-shadow of an uncertain obstacle $\cO_o$ if the probability $\pr(\cO_o \subseteq \cS_o) \geq 1 - \epsilon$.
\end{definition}
Geometrically, an $\epsilon$-shadow is a region that contains the uncertain obstacle with probability of at least $1 - \epsilon$, which can be non-unique. For example $\cS_o = \R^d$ is an $\epsilon$-shadow of any uncertain obstacle. To preclude this trivial case, the maximal $\epsilon$-shadow is defined:
\begin{definition}[Maximal $\epsilon$-Shadow]\label{def:max_shadow}
    A set $\cS_o \subseteq \R^d$ is a maximal $\epsilon$-shadow of an uncertain obstacle $\cO_o$ if the probability $\pr(\cO_o \subseteq \cS_o) = 1 - \epsilon$.
\end{definition}

The above definition ensures that if there exists a maximal $\epsilon$-shadow $\cS_o$ of the uncertain obstacle $\cO_o$ that does not intersect the robot, i.e. 
\rebuttal{$\tn{dis}(\vp_i, \cS_o) \geq r_s$}, 
then the collision probability between the robot and obstacle is below $\epsilon$, i.e. \rebuttal{$\pr({\tn{dis}}(\vp_i, \cO_o) \geq r_s) \geq 1-\epsilon$}. 
\rebuttal{Note that the maximal $\epsilon$-shadow may also be non-unique.
In this paper, we employ the method proposed by \cite{Dawson2020IROS} to construct such shadows. 
Recall that the uncertain obstacle $\cO_o$ is related to the nominal geometry $\hat{\cO}_o$ by $\cO_o = \{ \vx+\vd_o~|~ \vx\in\hat{\cO}_o, \vd_o \sim \Gau(0, \Sigma_o) \}$.}
To construct the maximal $\epsilon$-shadow, we first define the following ellipsoidal set
\begin{equation}\label{eq:shadowset}
    \cD_o = \{ \vd: \vd^T\Sigma_o^{-1}\vd \leq F^{-1}(1-\epsilon) \},
\end{equation}
where $F^{-1}(\cdot)$ is the inverse of the cumulative distribution function (CDF) of the chi-squared distribution with $d$ degrees of freedom. 
Next, Let 
\begin{equation}\label{eq:shadowobs}
    \cS_o = \hat{\cO}_o + \cD_o \rebuttal{=\{\vx + \vd~|~\vx\in\hat{\cO}_o,\vd\in\cD_o\}},
\end{equation}
be the Minkowski sum of \rebuttal{the nominal obstacle shape $\hat{\cO}_o$ and the ellipsoidal set $\cD_o$. 
Then, we have the following lemma \citep{Axelrod2018} and theorem \citep{Dawson2020IROS}:
\begin{lemma}\label{lemma:probability_Do}
    Let $\vd_o \sim \Gau(0, \Sigma_o) \in \R^d$ and $\cD_o = \{ \vd:\vd^T\Sigma_o^{-1}\vd \leq F^{-1}(1-\epsilon) \} \subset \R^d$, then $\pr(\vd_o \in \cD_o) = 1 - \epsilon$.
\end{lemma}
\begin{theorem}\label{theorem:maximal_shadow}
    $\cS_o$ is a maximal $\epsilon$-shadow of $\cO_o$.
\end{theorem}
}
\rebuttal{Proofs of the above lemma and theorem are given in Appendix \ref{appendix:proof_probability_Do} and \ref{appendix:proof_maximal_shadow}.}

\section{Buffered Uncertainty-Aware Voronoi Cells} \label{sec:method_1}
In this section, we formally introduce the concept of buffered uncertainty-aware Voronoi cells (B-UAVC) and give its construction method. 

\subsection{Definition of B-UAVC}\label{subsec:defBUAVC}
Our objective is to obtain a probabilistic safe region for each robot in the workspace given the robots and obstacles positions, and taking into account their uncertainties.

\begin{definition}[Buffered Uncertainty-Aware \\Voronoi Cell]\label{def:buavc}
    Given a team of robots $i \in \{1,\dots,n\}$ with positions mean $\hp_i \in \R^d$ and covariance $\Sigma_i \in \R^{d\times d}$, and a set of convex polytope obstacles $o \in \{1,\dots,m\}$ with known shapes and locations mean $\hp_o \in \R^d$ and covariance $\Sigma_o \in \R^{d\times d}$, the buffered uncertainty-aware Voronoi cell (B-UAVC) of each robot is defined as a convex polytope region:
    \begin{align}
        \cV_i^{u,b} = \{ \vp\in\R^d: \va_{ij}^T\vp &\leq b_{ij} - \beta_{ij}, \forall j\neq i, j\in\cI, \\ \tn{and~~} \va_{io}^T\vp &\leq b_{io} - \beta_{io}, \forall o\in\cI_o \},
    \end{align}
    such that the probabilistic collision free constraints in Definition \ref{def:pcollfree} are satisfied. 
\end{definition}

In the above B-UAVC definition, $\va_{ij}, \va_{io} \in \R^d$ and $b_{ij}, b_{io} \in \R$ are parameters of the hyperplanes that separate the robot from other robots and obstacles, which results in a decomposition of the workspace. $\beta_{ij}$ and $\beta_{io}$ are additional buffer terms added to retract the decomposed space for probabilistic collision avoidance. Accordingly, we further define 
\begin{align}\label{eq:uavc}
    \cV_i^{u} = \{ \vp\in\R^d: \va_{ij}^T\vp &\leq b_{ij}, \forall j\neq i, j\in\cI, \\ \tn{and~~} \va_{io}^T\vp &\leq b_{io}, \forall o\in\cI_o \},
\end{align}
that does not include buffer terms to be the uncertainty-aware Voronoi cell (UAVC) of robot $i$. 

It can be observed the UAVC and B-UAVC of robot $i$ are the intersection of the following:
\begin{enumerate}
    \item $n-1$ half-space hyperplanes separating robot $i$ from robot $j$ for all $j\neq i, j\in \cI$;
    \item $m$ half-space hyperplanes separating robot $i$ from obstacle $o$ for all $o \in \cI_o$.
\end{enumerate}
In the following, we will describe how to calculate the separating hyperplanes with parameters $(\va_{ij}, b_{ij})$ and $(\va_{io}, b_{io})$ that construct the UAVC and then the corresponding buffer terms $\beta_{ij}, \beta_{io}$ constructing the B-UAVC for probabilistic collision avoidance.

\subsection{Inter-Robot Separating Hyperplane}\label{subsec:robotPlane}
In contrast to only separating two deterministic points in Voronoi cells, we separate two uncertain robots with known positions mean and covariance. To achieve that, we rely on the concept of the best linear separator between two Gaussian distributions \citep{Anderson1962}.

Given $\vp_i \sim \cN(\hp_i,\Sigma_i)$ and $\vp_j \sim \cN(\hp_j,\Sigma_j)$, consider a linear separator $\va_{ij}^T\vp = b_{ij}$ where $\va_{ij} \in \R^d$ and $b_{ij} \in \R$. The separator classifies the points $\vp$ in the space into two clusters: $\va_{ij}^T\vp \leq b_{ij}$ to the first one while $\va_{ij}^T\vp > b_{ij}$ to the second. The separator parameters $\va_{ij}$ and $b_{ij}$ can be obtained by minimizing the maximal probability of misclassification.

The misclassification probability when $\vp$ is from the first distribution is 
\begin{equation*}
    \begin{aligned}
        \pr_i(\va_{ij}^T\vp > b_{ij}) &= \pr_i\left(\frac{\va_{ij}^T\vp - \va_{ij}^T\hp_i}{\sqrt{\va_{ij}^T\Sigma_i\va_{ij}}} > \frac{b_{ij} - \va_{ij}^T\hp_i}{\sqrt{\va_{ij}^T\Sigma_i\va_{ij}}}\right) \\ 
        &= 1 - \Phi((b_{ij} - \va_{ij}^T\hp_i)/\sqrt{\va_{ij}^T\Sigma_i\va_{ij}}),
    \end{aligned}
\end{equation*}
where $\Phi(\cdot)$ denotes the cumulative distribution function (CDF) of the standard normal distribution. 
Similarly, the misclassification probability when $\vp$ is from the second distribution is 
\begin{equation*}
    \begin{aligned}
        \pr_j(\va_{ij}^T\vp \leq b_{ij}) &= \pr_j\left(\frac{\va_{ij}^T\vp - \va_{ij}^T\hp_j}{\sqrt{\va_{ij}^T\Sigma_j\va_{ij}}} \leq \frac{b_{ij} - \va_{ij}^T\hp_j}{\sqrt{\va_{ij}^T\Sigma_j\va_{ij}}}\right) \\ 
        &= 1 - \Phi((\va_{ij}^T\hp_j - b_{ij})/\sqrt{\va_{ij}^T\Sigma_j\va_{ij}}).
    \end{aligned}
\end{equation*}

The objective is to minimize the maximal value of $\pr_i$ and $\pr_j$, i.e. 
\begin{equation}\label{eq:best_linear_separator}
    (\va_{ij}, b_{ij}) = \arg\underset{\va_{ij}\in\R^d,b_{ij}\in\R}{\min\max}(\pr_i, \pr_j),
\end{equation}
which can be solved using a fast minimax procedure. 
\rebuttal{In this paper, we employ the procedure developed by \cite{Anderson1962} to compute the best linear separator parameters $\va_{ij}$ and $b_{ij}$. A brief summary of the procedure is presented in Appendix \ref{appendix:best_linear_separator}.}

\begin{remark}\label{rmk:sepPlane}
    The best linear separator coincides with the separating hyperplane of Eq. (\ref{eq:vcPlane}) when $\Sigma_i = \Sigma_j = \sigma^2 I$. In this case, $\va_{ij} = \frac{2}{\sigma^2}(\hp_j - \hp_i)$ and $b_{ij} = \frac{1}{\sigma^2}(\hp_j-\hp_i)^T(\hp_i+\hp_j)$.
\end{remark}

\begin{remark}\label{rmk:sepMutual}
    $\forall i\neq j \in \cI, \va_{ji} = -\va_{ij}, b_{ji} = -b_{ij}$. This can be obtained according to the definition of the best linear separator.
\end{remark}

\begin{remark}\label{rmk:uavcNotFull}
    In contrast to deterministic Voronoi cells, the UAVCs constructed from the best linear separators generally do not constitute a full tessellation of the workspace, i.e. $\bigcup_1^n\cV_i^u \subseteq \cW$, as shown in Fig. \ref{subfig:uavc}.
\end{remark}

\subsection{Robot-Obstacle Separating Hyperplane}\label{subsec:obsPlane}

\begin{figure}[t]
    \centering
    \includegraphics[width=0.45\textwidth]{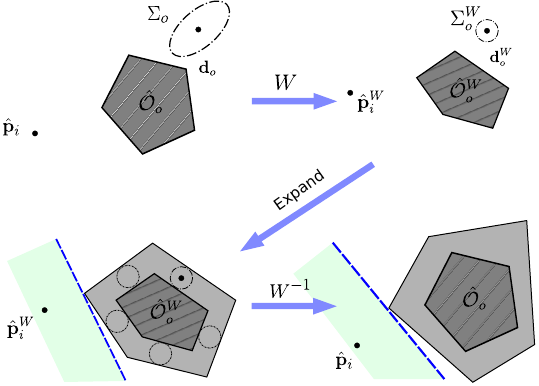}
    \caption{Depiction of uncertainty-aware separating hyperplane calculation between a point and an arbitrary polytope obstacle with uncertain location. (Top left) A point and a polytope obstacle with uncertain location. (Top right) Effects of the transformation $W$ to normalize the error covariance. (Bottom left) $\epsilon$-shadow of the transformed obstacle and the max-margin separating hyperplane in the transformation space. (Bottom right) Inverse transformation to obtain the uncertainty-aware separating hyperplane. }%
    \label{fig:poly_plane}%
\end{figure}

Our method to calculate the uncertainty-aware separating hyperplane between a robot and a convex polytope obstacle with uncertain location is illustrated in Fig. \ref{fig:poly_plane}. Given the mean position of the robot $\hp_i$ and the \rebuttal{uncertain obstacle $\cO_o = \{\vx + \vd_o~|~\vx\in\hat{\cO}_o, \vd_o \sim \Gau(0, \Sigma_o)$\}}, we first perform a linear coordinate transformation: 
\begin{equation}\label{eq:coordinate_transformation}
    W = (\sqrt{\Sigma_o})^{-1},
\end{equation}
Under the transformation, the robot mean position and obstacle information become
\begin{align}
    {\hp}_i^W &= W\hp_i, \\
    \hat{\cO}_o^W &= W\hat{\cO}_o, \\
    {\vd}_o^W &= W\vd_o, \\
    {\Sigma}_o^W &= W\Sigma_o W^T = I^{d\times d}.
\end{align}
\rebuttal{The transformed uncertain obstacle is then $\cO_o^W = \{\vx^W + \vd_o^W~|~\vx^W \in \hat{\cO}_o^W, \vd_o^W\sim\Gau(0,I)\}$.} Here we use the super-script ${\cdot}^W$ to indicate variables in the transformed space. Note that the obstacle position uncertainty covariance is normalized to an identity matrix under the transformation, as shown in Fig. \ref{fig:poly_plane} (Top right). 
\rebuttal{This coordinate transformation technique to normalize the uncertainty covariance has also been applied to other motion planning under uncertainty works \citep{Hardy2013}.}

Then given the collision probability threshold $\delta$, we compute a $\epsilon$-shadow of the transformed uncertain obstacle $\cO_o^W$ based on Eqs. (\ref{eq:shadowset})-(\ref{eq:shadowobs}):
\begin{align}
    {\cD}_o^W &= \{{\vd}^W: {{\vd}^{W}}^T{\vd}^W \leq F^{-1}(1-\epsilon) \}, \label{eq:transD} \\ 
    {\cS}_o^W &= \hat{\cO}_o^W + {\cD}_o^W,  \label{eq:transShadow}
\end{align}
where $\epsilon = 1-\sqrt{1-\delta}$, making that $\pr({\cO}_o^W \subseteq {\cS}_o^W) = \sqrt{1-\delta}$.

\rebuttal{
Note that we assume $\hat{\cO}_o$ is a convex polytope. Hence, the transformed $\hat{\cO}_o^W$ is also a polytope. In addition, it can be observed the set ${\cD}_o^W$ defined in Eq. (\ref{eq:transD}) is a circular (sphere in 3D) set with radius $\sqrt{F^{-1}(1-\epsilon)}$. Hence, we can compute the $\epsilon$-shadow in Eq. (\ref{eq:transShadow}) of the transformed uncertain obstacle by dilating its nominal shape by the diameter of the set ${\cD}_o^W$, which results in an inflated convex polytope. Note that the resulted convex polytope is slightly larger than the exact Minkowski sum $\cS_o^W$ which has smaller round corners. This introduces some conservativeness. For simplicity, we use the same notation $\cS_o^W$ for the resulted inflated convex polytope and thus there is $\pr({\cO}_o^W \subseteq {\cS}_o^W) > \sqrt{1-\delta}$.
}

Next, we separate ${\hp}_i^W$ from ${\cS}_o^W$ by finding a max-margin separating hyperplane between them. Note that ${\cS}_o^W$ is a bounded convex polytope that can be described by a list of vertices $({\psi}_1^W, \dots, {\psi}_{p_o}^W)$. Hence, finding a max-margin hyperplane between ${\hp}_i^W$ and ${\cS}_o^W$ can be formulated as a support vector machine (SVM) problem \citep{Honig2018}, which can be efficiently solved using a quadratic program:
\begin{equation}\label{eq:quadprog}
    \begin{aligned}
        \min \quad & {\va_{io}^{W}}^T{\va}_{io}^W \\
        \tn{s.t.} \quad & {{\va}_{io}^W}^T{\hp}_{io}^W - {b}_{io}^W \leq 1, \\
        & {{\va}_{io}^W}^T{\psi}_k^W - {b}_{io}^W \geq 1, ~~\forall~k \in 1,\dots, p_o.
    \end{aligned}
\end{equation}

The solution of the above quadratic program (\ref{eq:quadprog}) formulates a max-margin separating hyperplane with parameters $({\va}_{io}^W, {b}_{io}^W)$. We then shift it along its normal vector towards the obstacle shadow, resulting in a separating hyperplane exactly touching the shadow, as shown in Fig. \ref{fig:poly_plane} (Bottom left). Finally we perform an inverse coordinate transformation $W^{-1}$ and obtain the uncertainty-aware separating hyperplane between the robot and obstacle in the original workspace:
\begin{equation}\label{eq:an_obs}
    \begin{aligned}
        \va_{io} &= W^{T}{\va}_{io}^W, \\
        b_{io} &= {b}_{io}^W,
    \end{aligned}
\end{equation}
as shown in Fig. \ref{fig:poly_plane} (Bottom right), \rebuttal{in which the $\epsilon$-shadow in the transformed space $\cS_o^W$ becomes $\cS_o$ in the original space.}

\begin{remark}\label{rmk:shadow}  
    The linear coordinate transformation $W$ and its inverse $W^{-1}$ preserves relative geometries of $\cO_o$. That is, $\pr(\cO_o \subseteq \cS_o) = \pr({\cO}_o^W \subseteq {\cS}_o^W) > \sqrt{1-\delta}$.
\end{remark}

\subsection{Collision Avoidance Buffer and B-UAVC}\label{subsec:buffer}
In Section \ref{subsec:robotPlane} and \ref{subsec:obsPlane} we have described the method to compute the hyperplanes that construct the UAVC. Now we introduce two buffer terms to the UAVC, to account for the robot physical safety radius and the collision probability threshold.

Recall Eq. (\ref{eq:uavc}) that the UAVC of robot $i$ can be written as the intersection of a set of separating hyperplanes
\begin{equation*}
    \begin{aligned}
        \cV_i^{u} = \{ \vp\in\R^d: \va_{ij}^T\vp &\leq b_{ij}, \forall j\neq i, j\in\cI, \\ \tn{and~~} \va_{io}^T\vp &\leq b_{io}, \forall o\in\cI_o \},
    \end{aligned}
\end{equation*}
Let \rebuttal{$l\in\cI_l = \{1,\cdots,n,n+1,\cdots,n+m\}, l \neq i$} denote any other robot or obstacle, we can write the UAVC in the following form 
\begin{equation}
    \cV_i^u = \{ \vp\in\R^d: \va_{il}^T\vp \leq b_{il}, \forall l\in\cI_l, l\neq i \}.
\end{equation}
which combines the notations for inter-robot and robot-obstacle separating hyperplanes. Next, we will describe the computation method of probabilistic collision avoidance buffer to extend the UAVC to B-UAVC. 

\subsubsection{Robot safety radius buffer}
We compute the robot safety radius buffer by shifting the boundary of the UAVC towards the robot by a distance equal to the robot's radius. Hence the corresponding buffer for the hyperplane $(\va_{il}, b_{il})$ is 
\begin{equation}\label{eq:radius_buffer}
    \beta_i^r = \rebuttal{r_s}\norm{\va_{il}}.
\end{equation}
Figure \ref{subfig:buavc1} shows the buffered UAVC of each robot after taking into account their safety radius.

\subsubsection{Collision probability buffer}
\rebuttal{To achieve probabilistic collision avoidance,}
we further compute a buffer term $\beta_i^\delta$, which is defined as
\begin{equation}\label{eq:delta_buffer}
    \beta_i^\delta = \sqrt{2\va_{il}^T\Sigma_i\va_{il}}\cdot\tn{erf}^{-1}(2\sqrt{1-\delta}-1),
\end{equation}
where $\tn{erf}(\cdot)$ 
is the Gauss error function \citep{Andrews1997} defined as $\tn{erf}(x) = \frac{2}{\sqrt{\pi}}\int_0^xe^{-t^2}dt$ and $\tn{erf}^{-1}(\cdot)$ is its inverse. In this paper, we assume the threshold satisfies $0 < \delta < 0.75 $, which is reasonable in practice. Hence, $\tn{erf}^{-1}(2\sqrt{1-\delta}-1) > 0, \beta_i^\delta>0$. 
This buffer can be obtained by following the proof of forthcoming Theorem \ref{thm:buavc_robot} and Theorem \ref{thm:buavc_obs}.

Finally, the buffered uncertainty-aware Voronoi cell (B-UAVC) is obtained by combining the two buffers  
\begin{equation}\label{eq:BUAVC}
    \begin{aligned}
        \cV_{i}^{u, b} 
        = \{ \vp \in \R^d: 
        \va_{il}^{T}\vp \leq b_{il} - \beta_i^r - \beta_i^\delta, 
        \forall l\in\cI_l, l\neq i \}.
    \end{aligned}
\end{equation}
Figure \ref{subfig:buavc2} shows the final B-UAVC of each robot in the team.

\subsection{Properties of B-UAVC}\label{subsec:prop_buavc}
In this subsection, we justify the design of $\epsilon$ in Eq. (\ref{eq:transD}) when computing the shadow of uncertain obstacles, and computation of the collision probability buffer $\beta_i^\delta$ in Eq. (\ref{eq:delta_buffer}) by presenting the following two theorems. 

\begin{theorem}[Inter-Robot Probabilistic Collision Free]\label{thm:buavc_robot}
    $\forall \vp_i \sim \cN(\hp_i, \Sigma_i)$ and $\vp_j \sim \cN(\hp_j, \Sigma_j)$, where $\hp_i \in \cV_{i}^{u,b}$ and $\hp_j \in \cV_{j}^{u,b}, i \neq j \in \cI$, we have 
    \begin{equation*}
        \pr(\tn{dis}(\vp_i,\vp_j) \geq 2r_s) \geq 1-\delta,
    \end{equation*}
    i.e. the probability of collision between robots $i$ and $j$ is below the threshold $\delta$.
\end{theorem}

\begin{proof}
    We first introduce the following lemma:
    \begin{lemma}[Linear Chance Constraint \citep{Blackmore2011}]\label{lem:linChance}
        A multivariate random variable $\vx \sim \cN(\hx, \Sigma)$ satisfies
        \begin{equation}\label{eq:linChance}
            \pr(\va^T\vx \leq b) = \half + \half\tn{erf}\left( \frac{b-\va^T\hx}{\sqrt{2\va^T\Sigma\va}} \right).
        \end{equation}
    \end{lemma}

    According to Eq. (\ref{eq:BUAVC}), if $\hp_i\in\cV_i^{u,b}$, there is
    \begin{equation}\label{eq:iBUAVC}
        \va_{ij}^T\hp_i \leq b_{ij} - r_s\norm{\va_{ij}} - \sqrt{2\va_{ij}^T\Sigma_i\va_{ij}}\cdot\tn{erf}^{-1}(2\sqrt{1-\delta}-1).
    \end{equation}
    Applying Lemma \ref{lem:linChance} and substituting the above equation, we have
    \begin{equation}\label{eq:iBUAVC_r}
        \begin{aligned}
            \pr(\va_{ij}\vp_i &\leq b_{ij} - r_s\norm{\va_{ij}}) \\
            &= \half + \half\tn{erf}\left( \frac{b_{ij} - r_s\norm{\va_{ij}})-\va_{ij}^{T}\hp_{i}}{\sqrt{2\va_{ij}^{T}\va_{ij}}} \right) \\ 
            &\geq \half + \half\tn{erf}\left( \tn{erf}^{-1}(2\sqrt{1-\delta}-1) \right) \\ 
            &= \half + \half(2\sqrt{1-\delta}-1) \\ 
            &= \sqrt{1-\delta}.
        \end{aligned}
    \end{equation}
    Similarly for robot $j$, there is
    \begin{equation}
        \pr(\va_{ji}\vp_j \leq b_{ji} - r_s\norm{\va_{ji}}) \geq \sqrt{1-\delta}.
    \end{equation}
    Note that $\va_{ij} = -\va_{ji}$, $b_{ij} = -b_{ji}$ (Remark \ref{rmk:sepMutual}). It is trivial to prove that
    \begin{equation}
        \left.\begin{aligned}
                \va_{ij}\vp_i \leq b_{ij} - r_s\norm{\va_{ij}}\\
                \va_{ji}\vp_j \leq b_{ji} - r_s\norm{\va_{ji}}
              \end{aligned}
        \right\}
        \implies \norm{\vp_i-\vp_j}\geq 2r_s.
    \end{equation}
    Hence, we have 
    \begin{equation}
        \begin{aligned}
            &\pr(\tn{dis}(\vp_i,\vp_j) \geq 2r_s) = \pr(\norm{\vp_i-\vp_j}\geq 2r_s) \\ 
            &\geq \pr(\va_{ij}\vp_i \leq b_{ij} - r_s\norm{\va_{ij}})\cdot\pr(\va_{ji}\vp_j \leq b_{ji} - r_s\norm{\va_{ji}}) \\ 
            &\geq \sqrt{1-\delta}\cdot\sqrt{1-\delta} \\ 
            &= 1-\delta.
        \end{aligned}
    \end{equation}
    This completes the proof.
    \qed
\end{proof}

\begin{theorem}[Robot-Obstacle Probabilistic Collision Free]\label{thm:buavc_obs}
    $\forall \vp_i \sim \cN(\hp_i,\Sigma_i)$, where $\hp_i \in \cV_{i}^{u,b}$, we have $\pr(\tn{dis}(\vp_i, \cO_o)\geq r_s) \geq 1-\delta$, i.e. the probability of collision between robot $i$ and obstacle $o$ is below the threshold $\delta$.
\end{theorem}

\begin{proof}
    Similar to Eq. (\ref{eq:iBUAVC_r}), we have
    \begin{equation}
        \pr(\va_{io}\vp_i \leq b_{io} - r_s\norm{\va_{io}}) \geq \sqrt{1-\delta}.
    \end{equation}
    Based on the computation of $\va_{io}$ and $b_{io}$ in Eq. (\ref{eq:quadprog})-(\ref{eq:an_obs}), it is straightforward to prove that
    \begin{equation}
        \va_{io}\vp_i \leq b_{io} - r_s\norm{\va_{io}} \implies \tn{dis}(\vp_i, \cS_o) \geq r_s.
    \end{equation}
    Thus, 
    \begin{equation}
        \pr(\tn{dis}(\vp_i, \cS_o) \geq r_s) \geq \sqrt{1-\delta}.
    \end{equation}
    If $\cO_o \subseteq \cS_o$ and $\tn{dis}(\vp_i, \cS_o) \geq r_s$, there is $\tn{dis}(\vp_i, \cO_o) \geq r_s$. Hence, by combining with Remark \ref{rmk:shadow}, we have
    \begin{equation}
        \begin{aligned}
            \pr(\tn{dis}(\vp_i, \cO_o) \geq r_s) &\geq \pr(\cO_o \subseteq \cS_o) \cdot \pr(\tn{dis}(\vp_i, \cS_o) \geq r_s) \\ 
            &> \sqrt{1-\delta}\cdot\sqrt{1-\delta} \\ 
            &= 1-\delta,
        \end{aligned}
    \end{equation}
    which completes the proof.
    \qed
\end{proof}

\section{Collision Avoidance Using B-UAVC} \label{sec:method_2}
In this section, we present our decentralized collision avoidance method using the B-UAVC. We start by describing a reactive feedback controller for single-integrator robots, followed by its extensions to double-integrator and non-holonomic differential-drive robots. A receding horizon planning formulation is further presented for general high-order dynamical systems. We also provide a discussion on our proposed method. 

\subsection{Reactive Feedback Control}

\subsubsection{Single integrator dynamics}
Consider robots with single-integrator dynamics $\dot{\vp}_i = \vu_i$, where $\vu_i = \vv_i$ is the control input. Similar to \citet{Zhou2017}, a fast reactive feedback one-step controller can be designed to make each robot move towards its goal location $\vg_i$, as follows:
\begin{equation}\label{eq:reactive_controller}
     \vu_i = v_{i,\max}\cdot \frac{\vg_i^* - \hp_i}{\norm{\vg_i^* - \hp_i}},
\end{equation}
where $v_{i,\max}$ is the robot maximal speed and
\begin{equation}\label{eq:goal_projection}
    \vg_i^* := \argmin_{\vp\in\cV_i^{u,b}}\norm{\vp - \vg_i},
\end{equation}
is the closest point in the robot's B-UAVC to its goal location. 

The strategy used in the controller, Eq. (\ref{eq:reactive_controller}), is also called the ``move-to-projected-goal'' strategy \citep{Arslan2019}. At each time step, each robot in the system first constructs its B-UAVC $\cV_i^{u,b}$, then computes the closest point in $\cV_i^{u,b}$ to its goal, i.e. the ``projected goal'', and generates a control input according to Eq. (\ref{eq:reactive_controller}). Note that the constructed B-UAVC is a convex polytope represented by the intersection of a set of half-spaces hyperplanes. Hence, finding the closest point, Eq. (\ref{eq:goal_projection}), can be recast as a linearly constrained least-square problem, which can be solved efficiently using quadratic programming in polynomial time \citep{kozlov1980polynomial}. 

\subsubsection{Double integrator dynamics}
For single-integrator robots, the reactive controller Eq. (\ref{eq:reactive_controller}) guarantees the robot to be always within its corresponding B-UAVC and thus probabilistic collision free with other robots and obstacles. However, the controller may drive the robot towards to the boundary of its B-UAVC. Consider the double-integrator robot which has a limited acceleration, $\ddot{\vp}_i = \vu_i$, where $\vu_i = \va\vc\vc_i$ is the control input. It might not be able to continue to stay within its B-UAVC when moving close to the boundary of the B-UAVC. Hence, to \rebuttal{enhance} safety, as illustrated in Fig. \ref{fig:stopping} we introduce an additional safety stopping buffer, which is defined as
\begin{equation}\label{eq:double_int_extra_buffer}
    \beta_i^s = 
    \begin{cases}
        \frac{\norm{\va_{il}^T\vv_i}^2}{2acc_{i,\max}}, &\tn{if~~} \va_{il}^T\vv_i > 0; \\
        0, &\tn{otherwise},
    \end{cases}
\end{equation}
where $acc_{i,\max}$ is the maximal acceleration of the robot. 
\rebuttal{This additional stopping buffer heuristically leaves more space for the robot to decelerate in advance before touching the boundaries of the original B-UAVC.}
Hence, the updated B-UAVC in Eq. (\ref{eq:BUAVC}) with an additional safety stopping buffer now becomes
\begin{equation}\label{eq:BUAVC_S}
    \begin{aligned}
        \cV_{i}^{u, b} 
        = \{ \vp \in \R^d: 
        \va_{il}^{T}\vp \leq b_{il} - \beta_i^r - \beta_i^\delta - \beta_i^s, \\ 
        \forall l\in\cI_l, l\neq i \}.
    \end{aligned}
\end{equation}

Accordingly, the reactive feedback one-step controller for double-integrator robots is as follows,
\begin{equation}\label{eq:double_int_controller}
    \vu_i = acc_{i,\max}\cdot \frac{\vg_i^* - \hp_i}{\norm{\vg_i^* - \hp_i}}.
\end{equation}

\begin{figure}[t]
	\centering	
	\includegraphics[width=0.7\columnwidth]{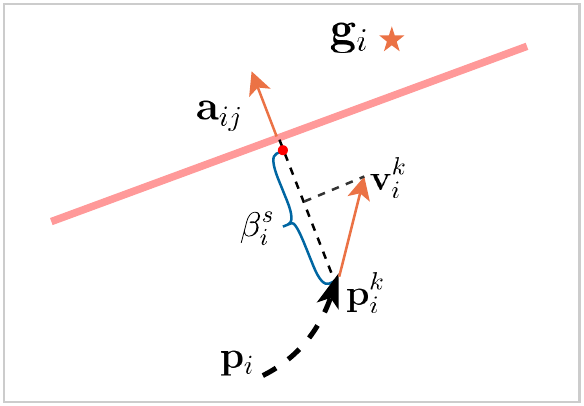}
  \caption{Additional buffer is added to allow robots with double-integrator dynamics to have enough space to decelerate.}
  	\label{fig:stopping}
\end{figure}

\subsubsection{Differential-drive robots}\label{subsubsec:control_diff}
Consider differential-drive robots moving on a two dimensional space $\cW \subseteq \R^2$, whose motions are described by
\begin{equation}
    \begin{aligned}
        &\dot{\hp}_i = v_i \mat \cos\theta_i \\ \sin\theta_i \mate, \\
        &\dot{\theta}_i = \omega_i,
    \end{aligned}
\end{equation}
where $\theta_i \in [-\pi, \pi)$ is the orientation of the robot, and $\vu_i = (v_i, \omega_i)^T \in \R^2$ is the vector of robot control inputs in which $v_i$ and $\omega_i$ are the linear and angular velocity, respectively. 
We adopt the control strategy developed by \citet{Arslan2019} and \citet{Astolfi1999} and briefly describe it in the following. 

As shown in Fig. \ref{fig:diff_control}, firstly, two line segments 
\begin{align}\label{eq:differ_drive_LUW}
    L_v &= \cV_i^{u,b} \cap H_N, \\ 
    L_\omega &= \cV_i^{u,b} \cap H_G,
\end{align}
are determined, in which $H_N$ is the straight line from the robot position towards its current orientation and $H_G$ is the straight line towards its goal location, respectively. Then the closest point in the robot's B-UAVC, $\vg_i^*$, and in the two lines segments $\vg_{i,v}^*$, $\vg_{i,\omega}^*$ is computed. Finally the control inputs of the robot are given by
\begin{equation}\label{eq:diff_control}
    \begin{aligned}
        v_i &= -k\cdot[\cos(\theta)~\sin\theta](\hp_i - \vg_{i,v}^*), \\ 
        \omega_i &= k\cdot\tn{atan}\left(\frac{[-\sin(\theta)~\cos\theta](\hp_i - (\vg_i^*+\vg_{i,\omega}^*)/2)}{[\cos(\theta)~\sin\theta](\hp_i - (\vg_i^*+\vg_{i,\omega}^*)/2)}\right),
    \end{aligned}
\end{equation} 
where $k > 0$ is the fixed control gain. It is proved by \citet{Arslan2019} that if the local safe region is convex, then the robot will stay within the convex safe region under the control law of Eq. (\ref{eq:diff_control}).

\begin{figure}[t]
	\centering	
	\includegraphics[width=0.7\columnwidth]{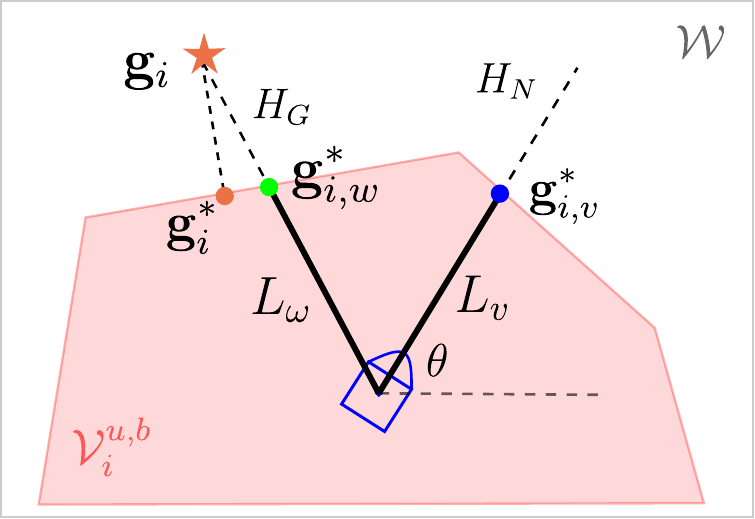}
    \caption{Reactive feedback control for differential-drive robots. }
  	\label{fig:diff_control}
\end{figure}

\subsection{Receding Horizon Planning}\label{subsec:control_mpc}
Consider general high-order dynamical systems with, potentially nonlinear, dynamics $\vx_i^{k} = \vf_i(\vx_i^{k-1}, \vu_i^{k-1})$, where $\vx_i^k \in \R^{n_x}$ denotes the robot state at time step $k$ which typically includes the robot position $\vp_i^k$ and velocity $\vv_i^k$, and $\vu_i^k \in \R^{n_u}$ the robot control input. To plan a local trajectory that respects the robot kinodynamic constraints, we formulate a constrained optimization problem with $N$ time steps and a planning horizon $\tau = N\Delta t$, where $\Delta t$ is the time step, as follows, 
\begin{problem}[Receding Horizon Trajectory Planning]\label{prob:rhc}
    \begin{subequations}
        \begin{align}
        \min\limits_{\hx_i^{1:N}, \vu_i^{0:N-1}}  ~~        
                                & \sum_{k=0}^{N-1}\vu_i^kR\vu_i^k + (\hp_i^N - \vg_{i}^N)^TQ_N(\hp_i^N - \vg_{i}^N) \nonumber \\
        \text{s.t.}	~~	        & \vx_i^0 = \hx_i, \\
                                & \hx_i^{k} = \vf_i(\hx_i^{k-1}, \vu_i^{k-1}), \\
                                & \hp_i^k \in \cV_i^{u,b}, \label{eq:buavcCons} \\ 
                                & \vu_i^{k-1} \in \cU_i,\\
                                &\forall i\in\cI, \,\forall k\in \{1,\dots,N\}.
        \end{align}
    \end{subequations}
\end{problem}
In Problem \ref{prob:rhc}, $\cU_i \in \R^{n_u}$ is the admissible control space; $R \in \R^{n_u \times n_u}$, $Q_N \in \R^{d\times d}$ are positive semi-definite symmetric matrices. The constraint (\ref{eq:buavcCons}) restrains the planned trajectory to be within the robot's B-UAVC $\cV_i^{u,b}$. According to the definition of $\cV_i^{u,b}$ in Eq. (\ref{eq:BUAVC_S}), the constraint can be formulated as a set of linear inequality constraints:
\begin{equation}\label{eq:buavcLinCons}
    \va_{il}^{T}\hp_i^k \leq b_{il} - \beta_i^r - \beta_i^\delta -\beta_i^s, ~\forall l\in\cI_l, l\neq i.
\end{equation}

At each time step, the robot first constructs its corresponding B-UAVC $\cV_i^{u,b}$ represented by a set of linear inequalities and then solves the above receding horizon planning problem. 
\rebuttal{The problem is in general a nonlinear and non-convex optimization problem due to the robot's nonlinear dynamics formulated as equality constraints $\hx_i^{k} = \vf_i(\hx_i^{k-1}, \vu_i^{k-1})$. 
}
While a solution of the problem including the planned trajectory and control inputs is obtained, the robot only executes the first control input $\vu_i^0$. Then with time going on and at the next time step, the robot updates its B-UAVC and solves the optimization problem again. The process is performed until the robot reaches its goal location. 

\rebuttal{

}

\rebuttal{
\begin{remark}[Probability of collision for the planned trajectory]
    From Theorem {\ref{thm:buavc_robot}} and {\ref{thm:buavc_obs}}, constraint ({\ref{eq:buavcCons}}) guarantees that at each stage within the planning horizon, the collision probability of robot $i$ with any other robot or obstacle is below the specified threshold $\delta$. Hence, the probability of collision for the entire planning trajectory of robot $i$ with respect to each other robot and obstacle can be bounded by $\pr(\cup_{k=1}^{N}\hp_i^k \notin \cV_i^{u,b}) \leq \sum_{k=1}^{N}\pr(\hp_i^k \notin \cV_i^{u,b}) = N\delta$. Nevertheless, this bound is over conservative in practice. The real collision probability of the planned trajectory is much smaller than $N\delta$ \mbox{\citep{Schmerling2017}}. Hence, we impose the collision probability threshold $\delta$ for each individual stage in the context of receding horizon planning, thanks to the fast re-planning and relatively small displacement between stages \citep{luo2020multi}. 
\end{remark}
}

\rebuttal{
    Algorithm \ref{alg:ca_buavc} summarizes our proposed method for decentralized probabilistic multi-robot collision avoidance, in which each robot in the system first constructs its B-UAVC, and then compute control input accordingly to restrain its motion to be within the B-UAVC. 

\begin{algorithm}[t]
    \caption{Collision Avoidance Using B-UAVC for Each Robot $i\in\cI$ in a Multi-robot Team}
    \label{alg:ca_buavc}
    \begin{algorithmic}[1]
            \Statex ------------------ Construction of B-UAVC ------------------
            \State Obtain $\vp_i \sim \Gau(\hp_i, \Sigma_i)$ via state estimation
            \For {Each other robot $j \in \cI, j \neq i$}
                \State Estimate $\vp_j \sim \Gau(\hp_j, \Sigma_j)$
                \State Compute the best linear separator parameters $(\va_{ij}, b_{ij})$ via Eq. (\ref{eq:best_linear_separator})
            \EndFor
            \For {Each static obstacle $o\in\cI_o$}
                \State Estimate $\vd_o \sim \Gau(0, \Sigma_o)$ with known $\hat{\cO}_o$
                \State Compute the separating hyperplane parameters $(\va_{io}, b_{io})$ via Eqs. (\ref{eq:coordinate_transformation})-(\ref{eq:an_obs})
            \EndFor
            \For {Each separating hyperplane $l \in \cI_l, l\neq i$}
                \State Compute the safety radius buffer via Eq. (\ref{eq:radius_buffer}): $\beta_i^r = {r_s}\norm{\va_{il}}$
                \State Compute the collision probability buffer via Eq. (\ref{eq:delta_buffer}): $\beta_i^\delta = \sqrt{2\va_{il}^T\Sigma_i\va_{il}}\cdot\tn{erf}^{-1}(2\sqrt{1-\delta}-1)$
                \State Construct the B-UAVC via Eq. (\ref{eq:BUAVC})
            \EndFor
            \Statex ------------------ Collision Avoidance Action ------------------
            \If {$i$ is single-integrator}
                \State Compute control input via Eqs. (\ref{eq:reactive_controller})-(\ref{eq:goal_projection})
            \ElsIf {$i$ is double-integrator}
                \State Compute control input via Eqs. (\ref{eq:double_int_extra_buffer})-(\ref{eq:double_int_controller})
            \ElsIf {$i$ is differential-drive}
                \State Compute control input via Eqs. (\ref{eq:differ_drive_LUW})-(\ref{eq:diff_control})
            \Else 
                \State Compute control input by solving Problem 1
            \EndIf
    \end{algorithmic}
\end{algorithm}

}

\rebuttal{
\subsection{Discussion}

\subsubsection{Uncertainty estimation}\label{subsubsec:estimation_discussion}
For each robot $i$ in the system, to construct its B-UAVC, the robot needs a) its own position estimation mean $\hp_i$ and uncertainty covariance $\Sigma_i$ from onboard measurements via a filter, e.g. a Kalman filter, and b) to know each other robot $j$'s position mean $\hp_i$ and uncertainty covariance $\Sigma_j$. In case communication is available, such position estimation information can be communicated among robots. However, in a fully decentralized system where there is no communication, each robot $i$ will need to estimate other robot $j$'s position mean and covariance, denoted by $\tilde{\vp}_j$ and $\tilde{\Sigma}_j$, via its own onboard sensor measurements. In this case, we assume that robot $i$'s estimation of robot $j$'s position mean is the same as robot $j$'s own estimation, i.e. $\tilde{\vp}_j = \hp_j$; while robot $i$'s estimation of the uncertainty covariance of robot $j$ is larger than its own localization uncertainty covariance, i.e. $|\tilde{\Sigma}_j| \geq |\Sigma_i|$. This assumption is reasonable in practice since the robot generally has more accurate measurements of its own position than other robots in the environment. Then robot $i$ computes its B-UAVC using $\hp_i, \Sigma_i, \tilde{\vp}_j$, and $\tilde{\Sigma}_j$. According to the properties of the best linear separator, this assumption leads that each robot $i$ always partitions a smaller space when computing the separating hyperplane with another robot $j$, which results in a more conservative B-UAVC to ensure safety for robot $i$ itself.  

\subsubsection{Empty B-UAVCs}
Taking into account uncertainty, the robots being probabilistic collision-free (Definition 1), i.e., $\pr(\norm{\vp_i-\vp_j}\geq 2r_s) \geq$ $1 - \delta, \forall i,j \in \{1,\dots,n\}, i\neq j$, does not guarantee that the defined B-UAVC $\cV_i^{u,b}$ is non-empty. Nevertheless, the case $\cV_i^{u,b}$ being empty is rarely observed in our simulations and experiments. We handle this situation by decelerating the robot if its B-UAVC is empty.

}

\section{Simulation Results}\label{sec:sim_result}
We now present simulation results comparing our proposed B-UAVC method with state-of-the-art baselines as well as a performance analysis of the proposed method in a variety of scenarios.

\subsection{Comparison to the BVC Method}\label{subsec:com_bvc}
We first compare our proposed B-UAVC method with the BVC approach \citep{Zhou2017} that we extend in two-dimensional obstacle-free environments with single-integrator robots. Both the B-UAVC and BVC methods only need robot position information to achieve collision avoidance, in contrast to the well-known reciprocal velocity obstacle (RVO) method \citep{VanDenBerg2011} which also requires robot velocity information to be communicated or sensed. Comparison between BVC and RVO has been demonstrated by \citet{Zhou2017} in 2D scenarios, hence in this paper we focus on comparing the proposed B-UAVC with BVC. 

We deploy the B-UAVC and BVC in a $10 \times 10$m environment with 2, 4, 8, 16 and 32 robots forming an \emph{antipodal circle swapping} scenario (\citet{VanDenBerg2011}). In this scenario, the robots are initially placed on a circle (equally spaced) and their goals are located at the antipodal points of the circle. We use a circle with a radius of 4.0 m in simulation. 
Each robot has a radius of 0.2 m, a local sensing range of 2.0 m and a maximum allowed speed of 0.4 m/s. 
The goal is assumed to be reached for each robot when the distance between its center and goal location is smaller than 0.1 m. 
To simulate collision avoidance under uncertainty, two different levels of noise, $\Sigma_1 = \tn{diag}(0.04~\tn{m}, 0.04~\tn{m})^2$ and $\Sigma_2 = \tn{diag}(0.06~\tn{m}, 0.06~\tn{m})^2$, are added to the robot position measurements. 
\rebuttal{Particularly, each robot's localization uncertainty covariance is $\Sigma_1$ and its estimation of other robots' position uncertainty covariance is $\Sigma_2$}. 
The time step used in simulation is $\Delta t = 0.1$ s.

In the basic BVC implementation, an extra $10\%$ or $100\%$ radius buffer is added to the robot's real physical radius to account for measurement uncertainty for comparison \citep{Wang2019}. In the B-UAVC implementation, the collision probability threshold is set as $\delta = 0.05$. Any robot will stop moving when it arrives at its goal or is involved in a collision. Both the B-UAVC and BVC methods use the same deadlock resolution techniques proposed in this paper (\rebuttal{Appendix \ref{appendix:deadlock}}). We set a maximum simulation step $K = 800$ and the collision-free robots that do not reach their goals within $K$ steps are regarded to be in deadlocks/livelocks.

For each case (number of robots $n$) and each method, we run the simulation 10 times. In each single run, we evaluate the following performance metrics: 
(a) collision rate, 
(b) minimum distance among robots,
(c) average travelled distance of robots, 
and (d) time to complete a single run.
The collision rate is defined to be the ratio of robots colliding over the total number of robots.  
Time to complete a single run is defined to be the time when the last robot reaches its goal. Note that the metrics (2)(3)(4) are calculated for robots that successfully reach their goal locations. 
Finally, statistics of 10 instances under each case are presented.

The simulation results are presented in Fig. \ref{fig:compare_bvc_si}. In all runs, no deadlocks are observed. In terms of collision avoidance, both the B-UAVC approach and BVC with additional 100\% robot radius achieve zero collision in all runs. The BVC with only 10\% robot radius leads to collisions when the total number of robots gets larger. In particular, when there are 32 robots an average of 28\% robots collide, as shown in Fig. \ref{subfig:com_bvc_si_col_rate}. While the BVC with 100\% additional robot radius can also achieve zero collision rate as our proposed B-UAVC, it is more conservative and less efficient. In average, the B-UAVC saves 10.1\% robot travelled distance (Fig. \ref{subfig:com_bvc_si_tra_dis}) and 14.4\% time for completing a single run (Fig. \ref{sub@subfig:com_bvc_si_tra_time}) comparing to the BVC with additional 100\% robot radius. 

\rebuttal{
\begin{remark}
   The ``BVC + $X\%$'' is a heuristic way to handle uncertainty. The above simulation results show that if $X$ is too small, then it cannot ensure safety; while if $X$ is too large, the results will be very conservative and less efficient. So generally reasoning about individual uncertainties using the proposed B-UAVC method will perform better than determining an extra $X\%$ buffer. 
\end{remark}
}

\rebuttal{
\begin{remark}
   In some cases we can design such an $X$ that it will have the same results as the B-UAVC method. 
   Consider the case where $\Sigma_i = \Sigma_j = \sigma^2I$. According to Remark \ref{rmk:sepPlane}, the best linear separator coincides with the separating hyperplane computed by the BVC method, whose parameters are denoted by $\va_{ij}$ and $b_{ij}$. The hyperplane parameters can be further normalized to make $\norm{\va_{ij}}=1$. In this case, our B-UAVC and the BVC have the same safety radius buffer $\beta_i^r = r_s$. Given a collision probability threshold $\delta$, our B-UAVC further introduces another buffer to handle uncertainty
   \begin{equation*}
      \begin{aligned}
         \beta_i^\delta &= \sqrt{2\va_{il}^T\Sigma_i\va_{il}}\cdot\tn{erf}^{-1}(2\sqrt{1-\delta}-1) \\ 
         &= \sigma\sqrt{2}\cdot\tn{erf}^{-1}(2\sqrt{1-\delta}-1).
      \end{aligned}
   \end{equation*}
   If we choose an extra safety buffer $X\%$ such that 
   \begin{equation*}
     X\%\cdot r_s = \sigma\sqrt{2}\cdot\tn{erf}^{-1}(2\sqrt{1-\delta}-1),
   \end{equation*}
   then the results of the ``BVC + $X\%$'' method are the same as our B-UAVC method.  
   However, our B-UAVC method can handle general cases where it is hard to design an $X\%$ to always achieve the same level of performance.
\end{remark}
}

\begin{figure}[t]
   \centering
   \subfloat[]{\label{subfig:com_bvc_si_col_rate}
      \includegraphics[width=.23\textwidth]{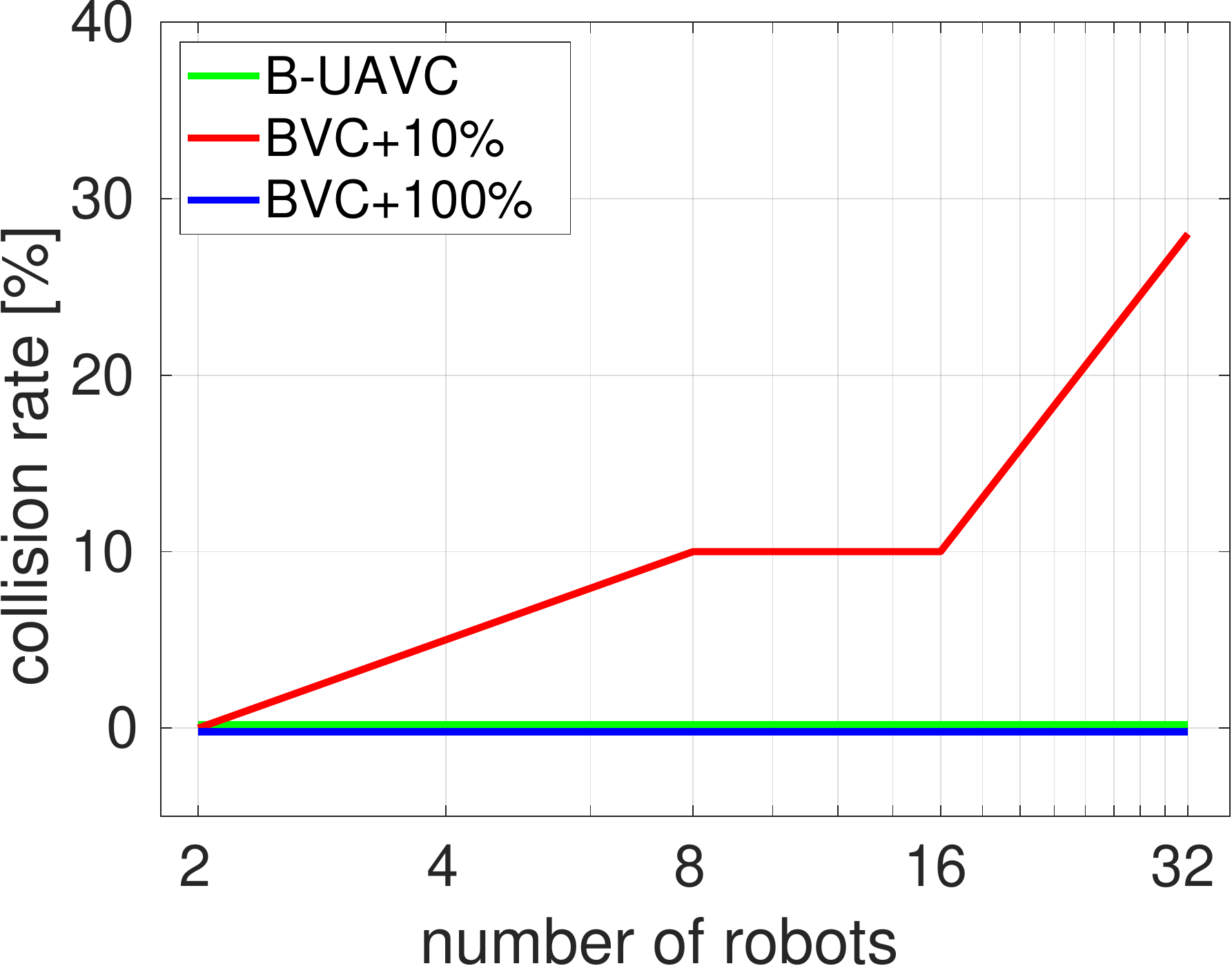}}
   \subfloat[]{\label{subfig:com_bvc_si_min_dis}
      \includegraphics[width=.23\textwidth]{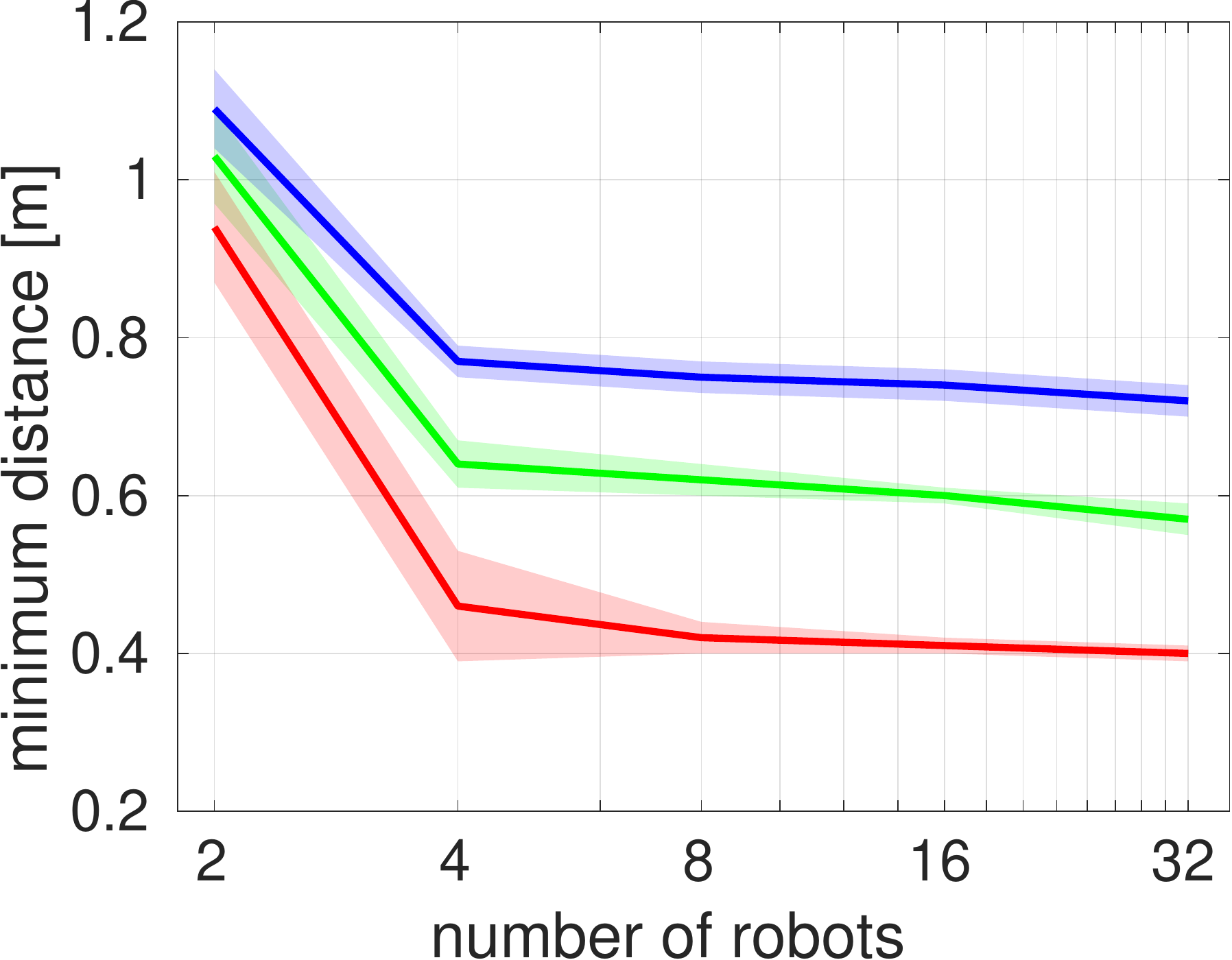}}
   \\ 
   \subfloat[]{\label{subfig:com_bvc_si_tra_dis}
      \includegraphics[width=.23\textwidth]{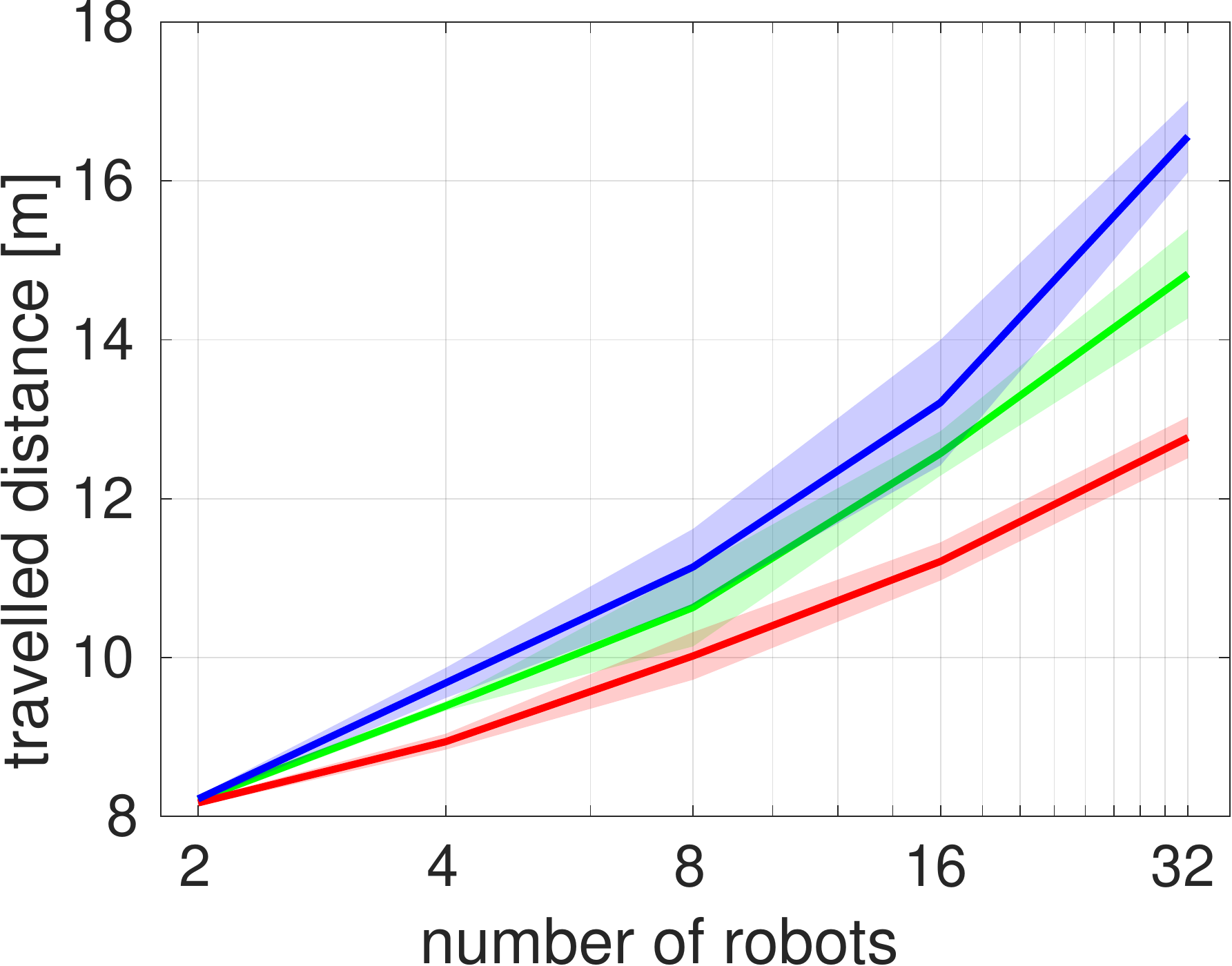}}
   \subfloat[]{\label{subfig:com_bvc_si_tra_time}
      \includegraphics[width=.23\textwidth]{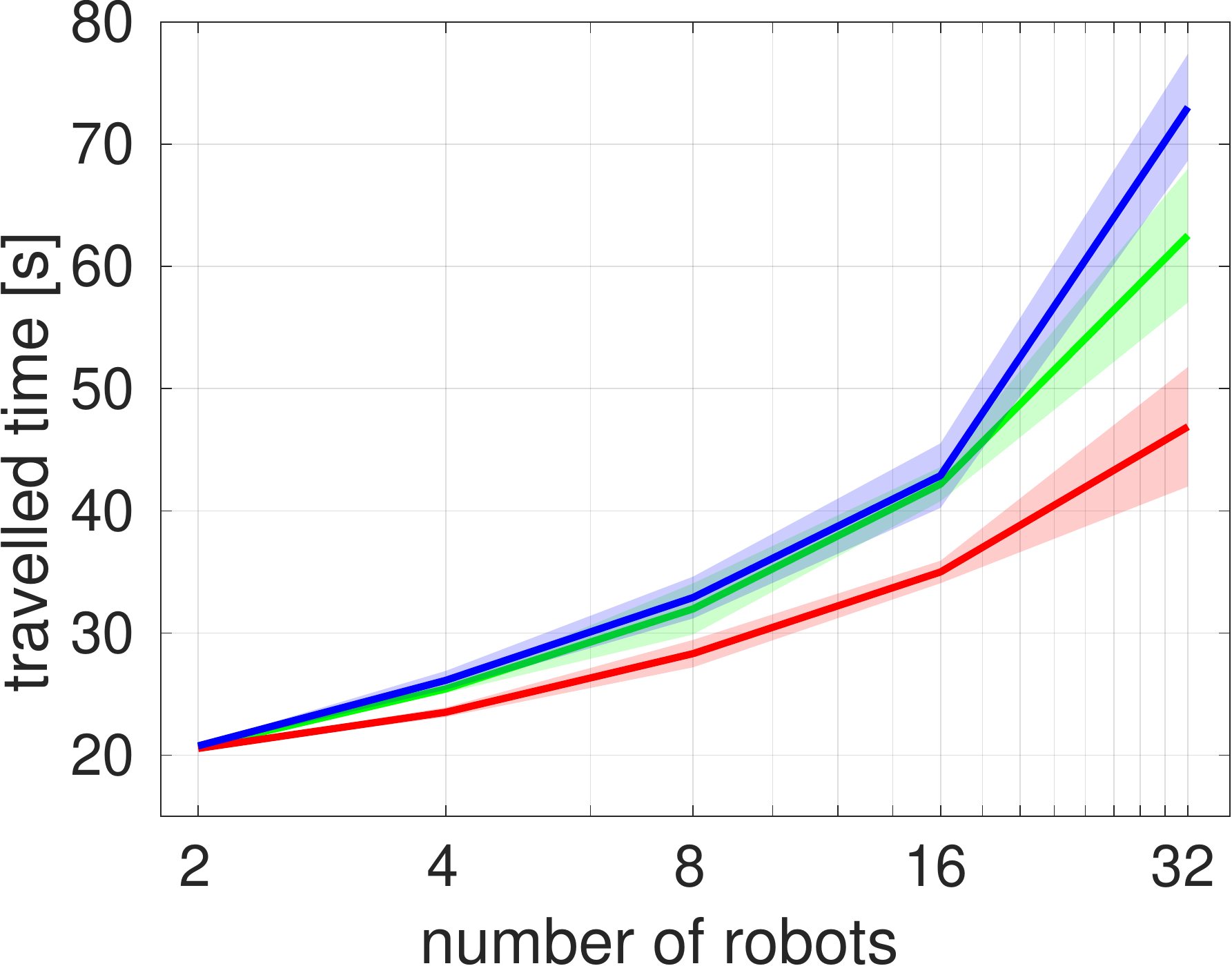}}

   \caption{Evaluation of the antipodal circle scenario with varying numbers of single-integrator robots. The (a) collision rate, (b) minimum distance, (c) travelled distance and (d) complete time are shown. Lines denote mean values and shaded areas around the lines denote standard deviations over 10 repetitions for each scenario.}
   \label{fig:compare_bvc_si}
\end{figure}

\begin{figure*}[h]
   \centering
   \subfloat[$t = 1$ s.]{\label{}
      \includegraphics[width=.185\textwidth]{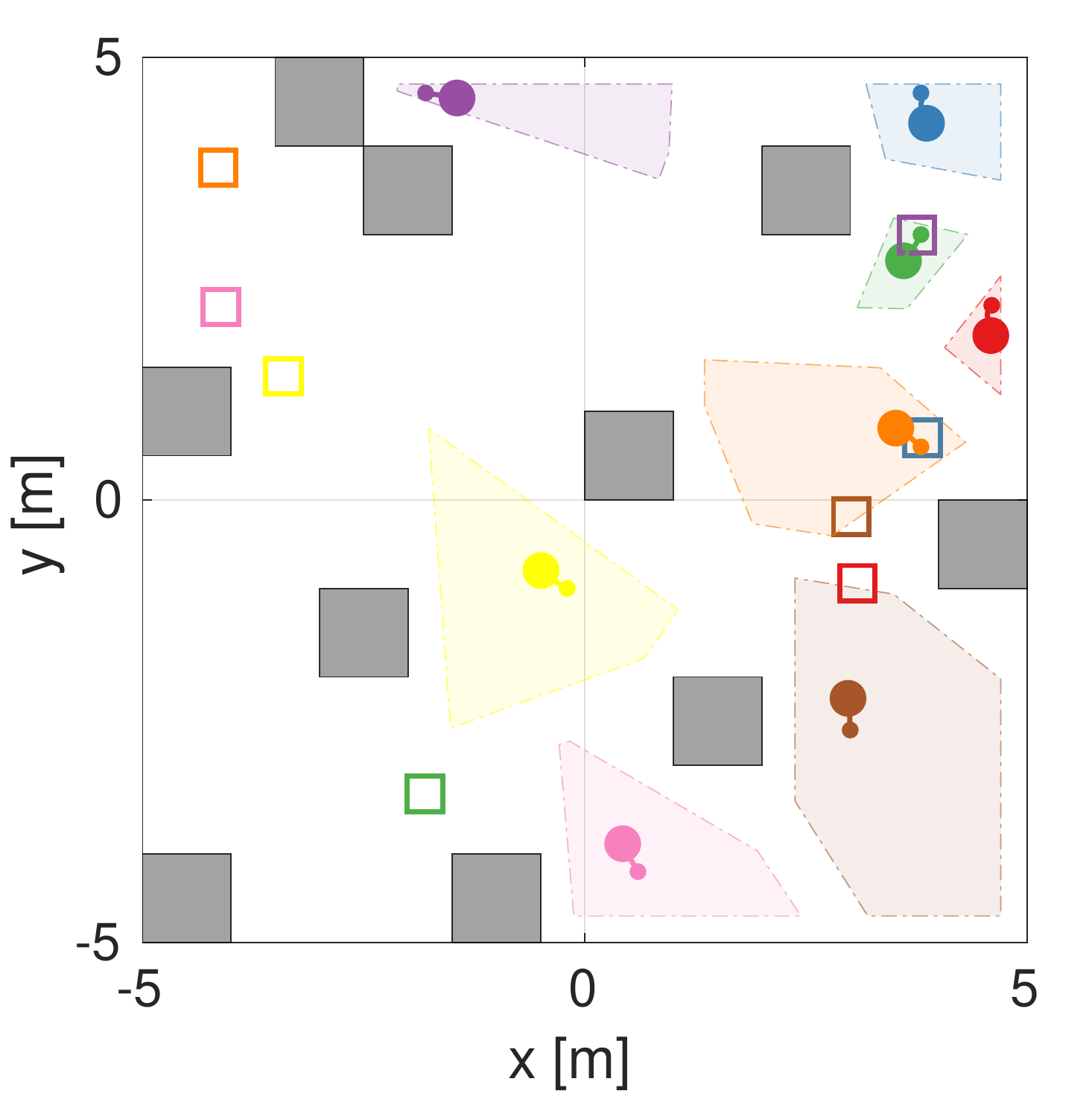}}
   \subfloat[$t = 6$ s.]{\label{}
      \includegraphics[width=.185\textwidth]{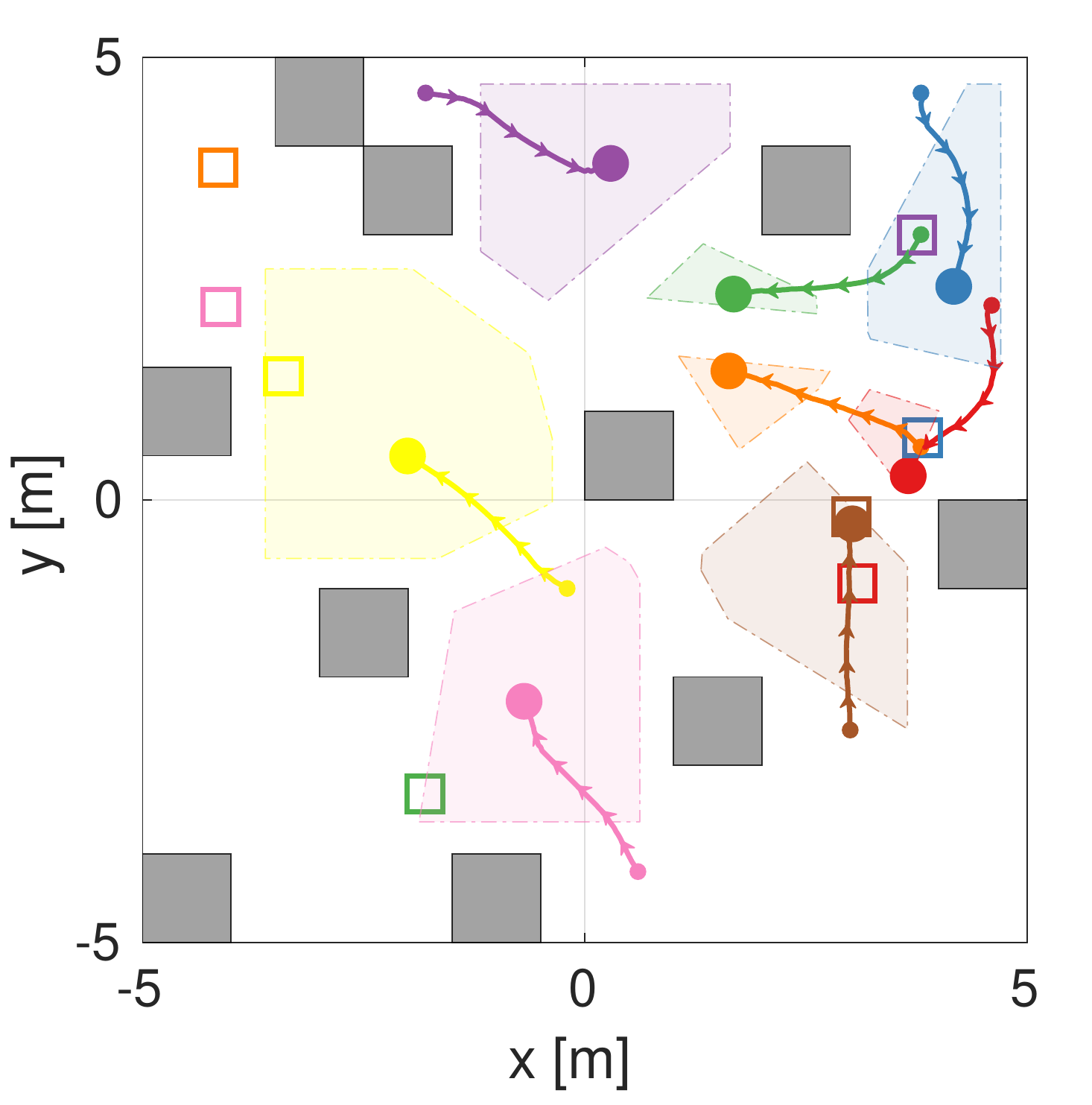}}
   \subfloat[$t = 12$ s.]{\label{}
      \includegraphics[width=.185\textwidth]{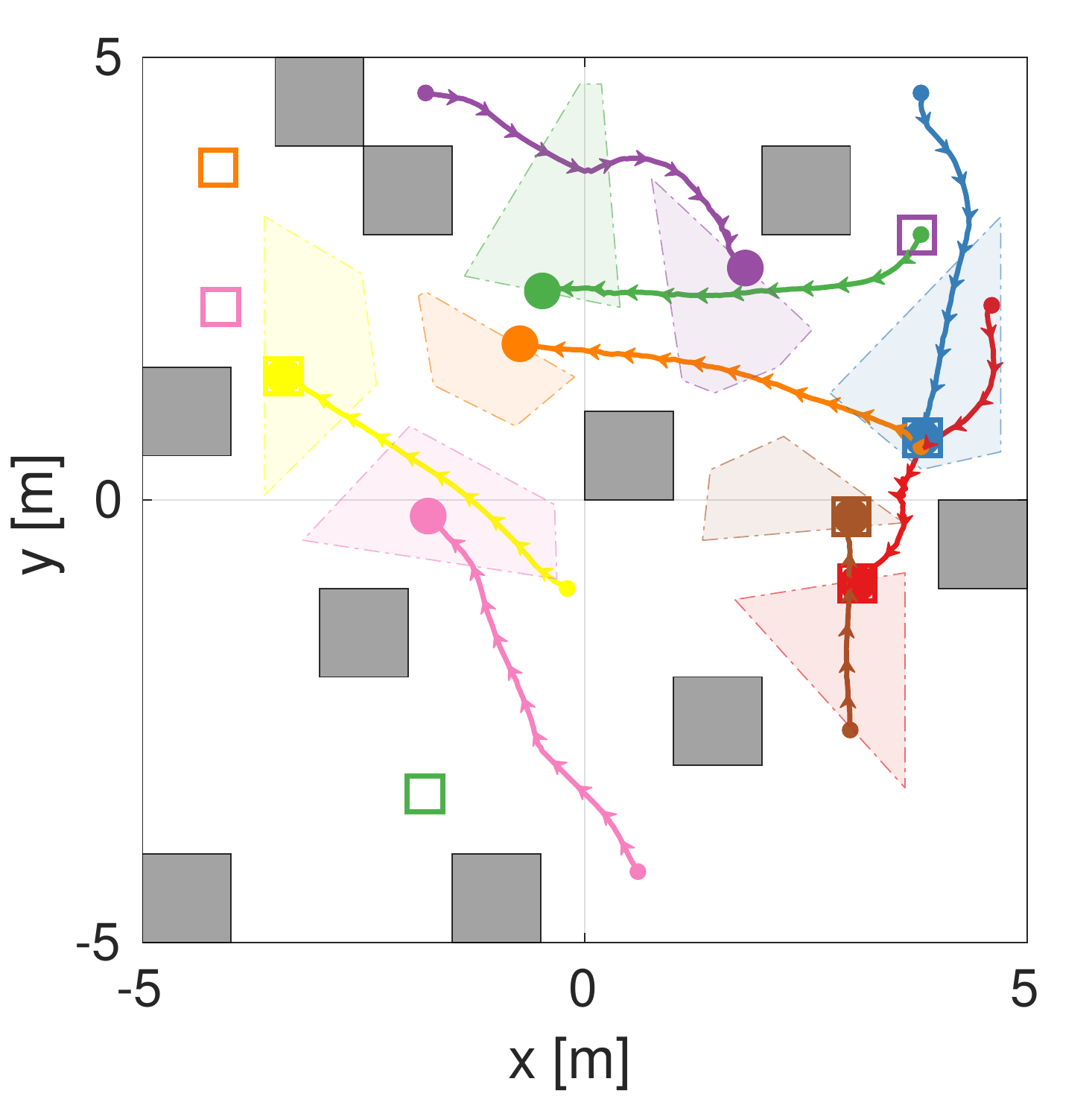}}
   \subfloat[$t = 18$ s.]{\label{}
      \includegraphics[width=.185\textwidth]{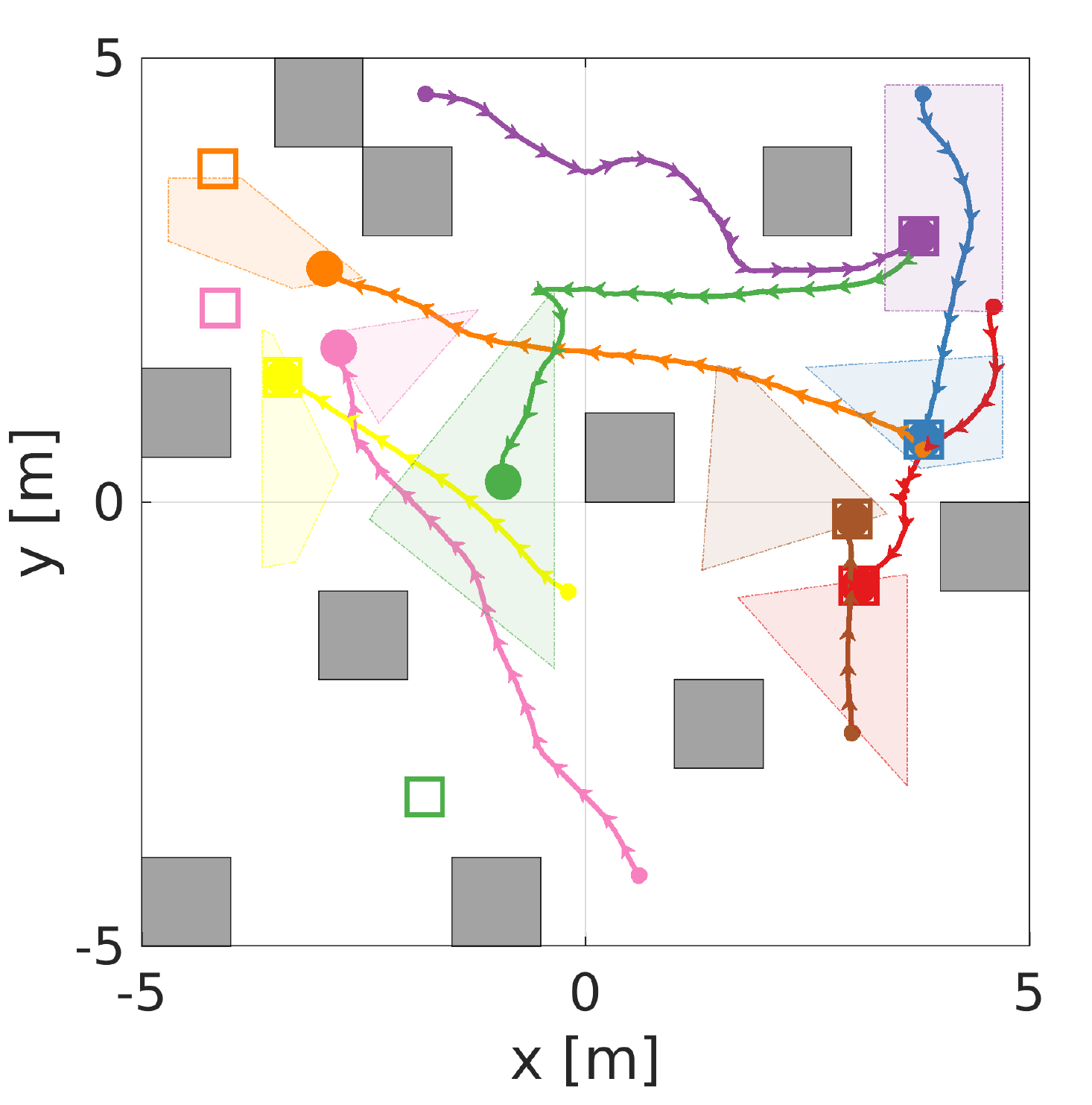}}
   \subfloat[$t = 27$ s.]{\label{}
      \includegraphics[width=.185\textwidth]{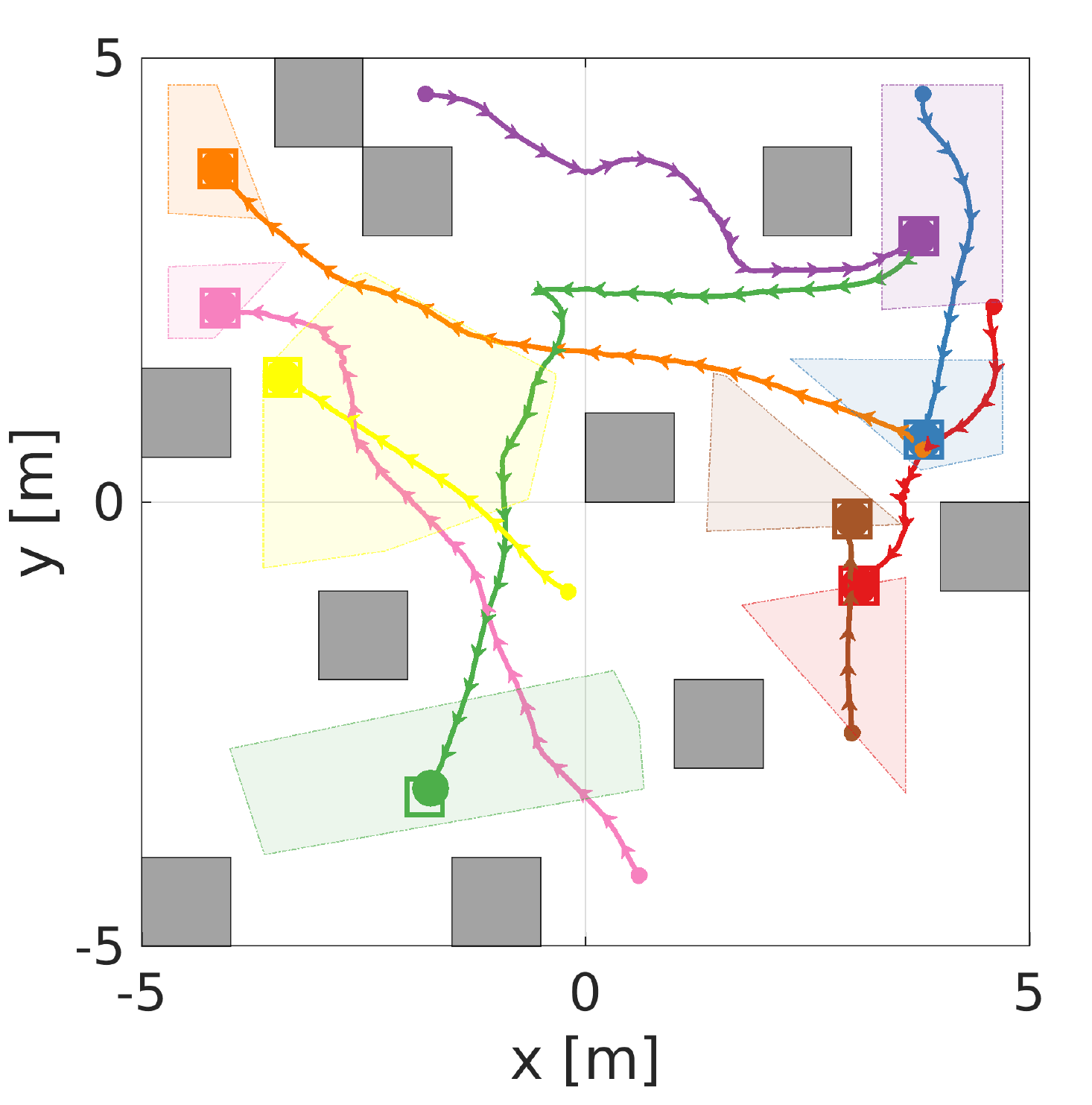}}

   \caption{A sample simulation run of the random moving scenario with 8 robots and 10\% obstacle density. The robot initial and goal locations are marked in circle disks and solid squares. Grey boxes are static obstacles. The B-UAVCs are shown in shaded patches with dashed boundaries. }
   \label{fig:r8o10}
\end{figure*}

\subsection{Performance Analysis}
We then study the effect of collision probability threshold on the performance of the proposed B-UAVC method. Similarly, we deploy the B-UAVC in a $10 \times 10$m environment with 2, 4, 8, 16 and 32 robots in obstacle-free and cluttered environments with 10\% obstacle density. In the obstacle-free case for each number of robots $n$, 10 scenarios are randomly generated to form a challenging \emph{asymmetric swapping} scenario \citep{Serra2020}, indicating that the environment is split into $n$ sections around the center and each robot is initially randomly placed in one of them while required to navigate to its opposite section around the center. In the obstacle-cluttered case, 10 \emph{random moving} scenarios are simulated for each different number of robots in which robot initial positions and goal locations are randomly generated. 
Fig. \ref{fig:r8o10} shows a sample run of the scenario with 8 robots and 10 obstacles. 
We then run each generated scenario 5 times given a parameter setting (collision probability threshold). The robots have the same radius and maximal speed as in Section \ref{subsec:com_bvc}. Localization noise with zero mean and covariance $\Sigma = \tn{diag}(0.06~\tn{m}, 0.06~\tn{m})^2$ is added. For evaluation of performance, we focus on the robot collision rate, the robot deadlock rate, 
and the minimum distance among successful robots. 

We evaluate the performance of B-UAVC with different levels of collision probability threshold: $\delta =$ 0.05, 0.10, 0.20 and 0.30. The simulation results are presented in Fig. \ref{fig:compare_thresh}. 
In the top row of the figure, we consider the collision rate among robots. The result shows that with a roughly small collision probability threshold $\delta = 0.05, 0.10 ,0.20$, no collisions are observed in both obstacle-free asymmetric swapping and obstacle-cluttered random moving scenarios, indicating that the B-UAVC method maintains a high level of safety. However, when $\delta$ is set to 0.3, the collision rate among robots increase dramatically, in particular when the number of robots is large. For example, in the asymmetric swapping scenario with 32 robots, there are 68.75\% robots involve in collisions in average. 
In the bottom row of the figure, the minimum distance among robots are compared. The result shows that with smaller threshold, the minimum distance will be a little bit larger. The reason is that robots with a smaller threshold will have more conservative behavior and have smaller B-UAVCs during navigation. 

\begin{figure}[t]
    \centering
    \subfloat[]{\label{subfig:thresh_col_rate_asy}
       \includegraphics[width=.23\textwidth]{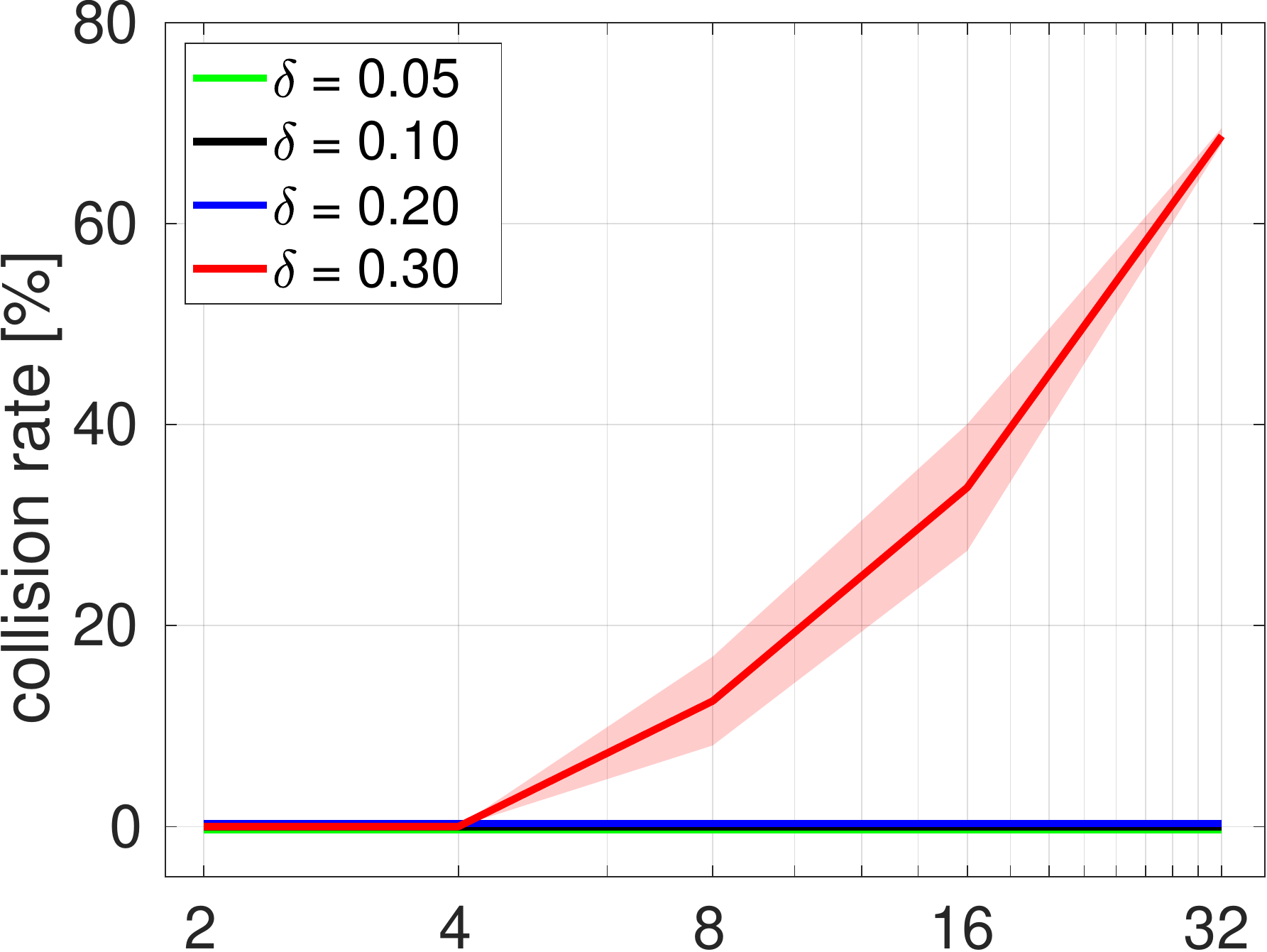}}
    \subfloat[]{\label{subfig:thresh_col_rate_rand}
       \includegraphics[width=.23\textwidth]{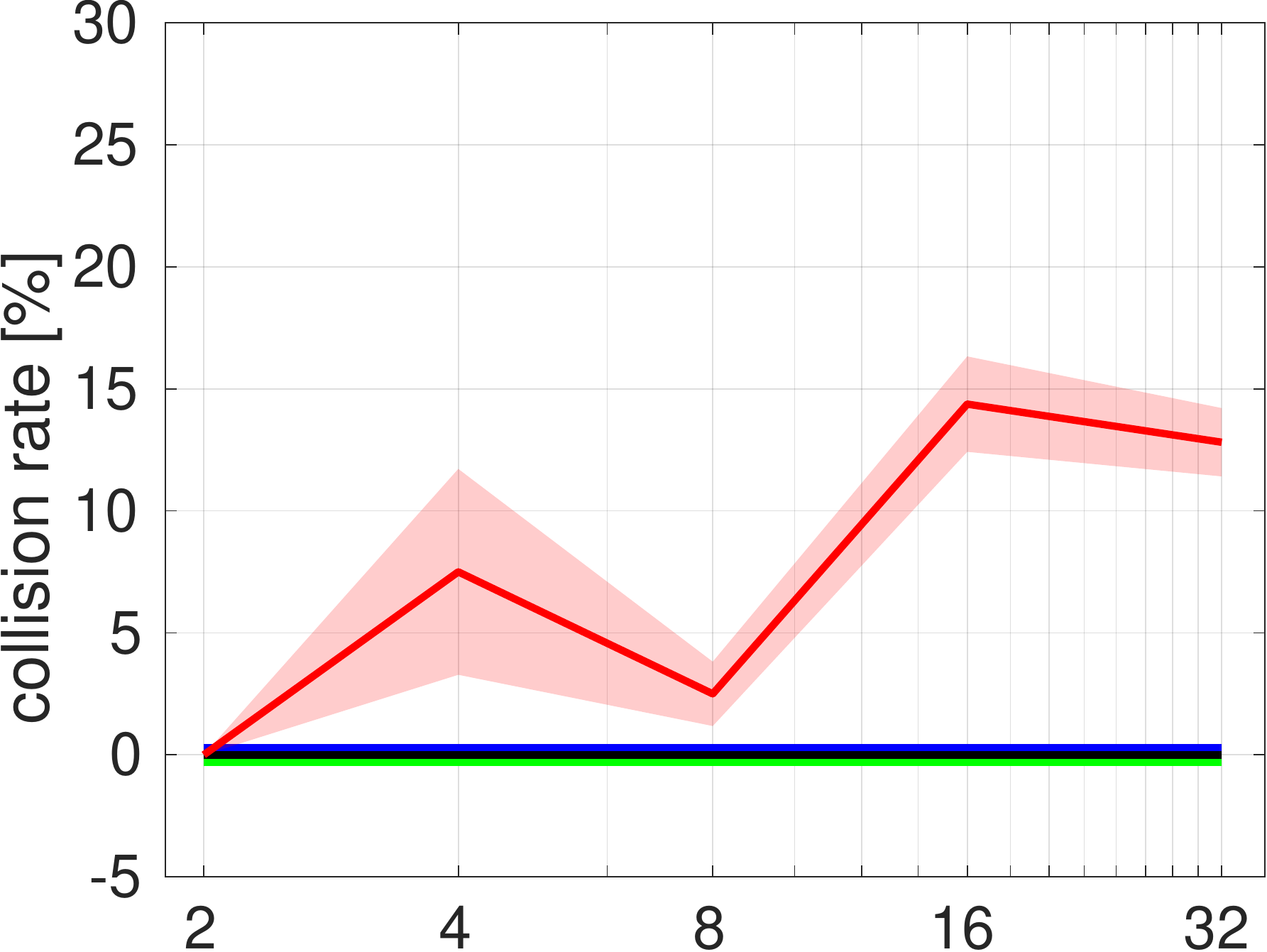}}
    \\
    \subfloat[]{\label{subfig:thresh_min_dis_asy}
       \includegraphics[width=.23\textwidth]{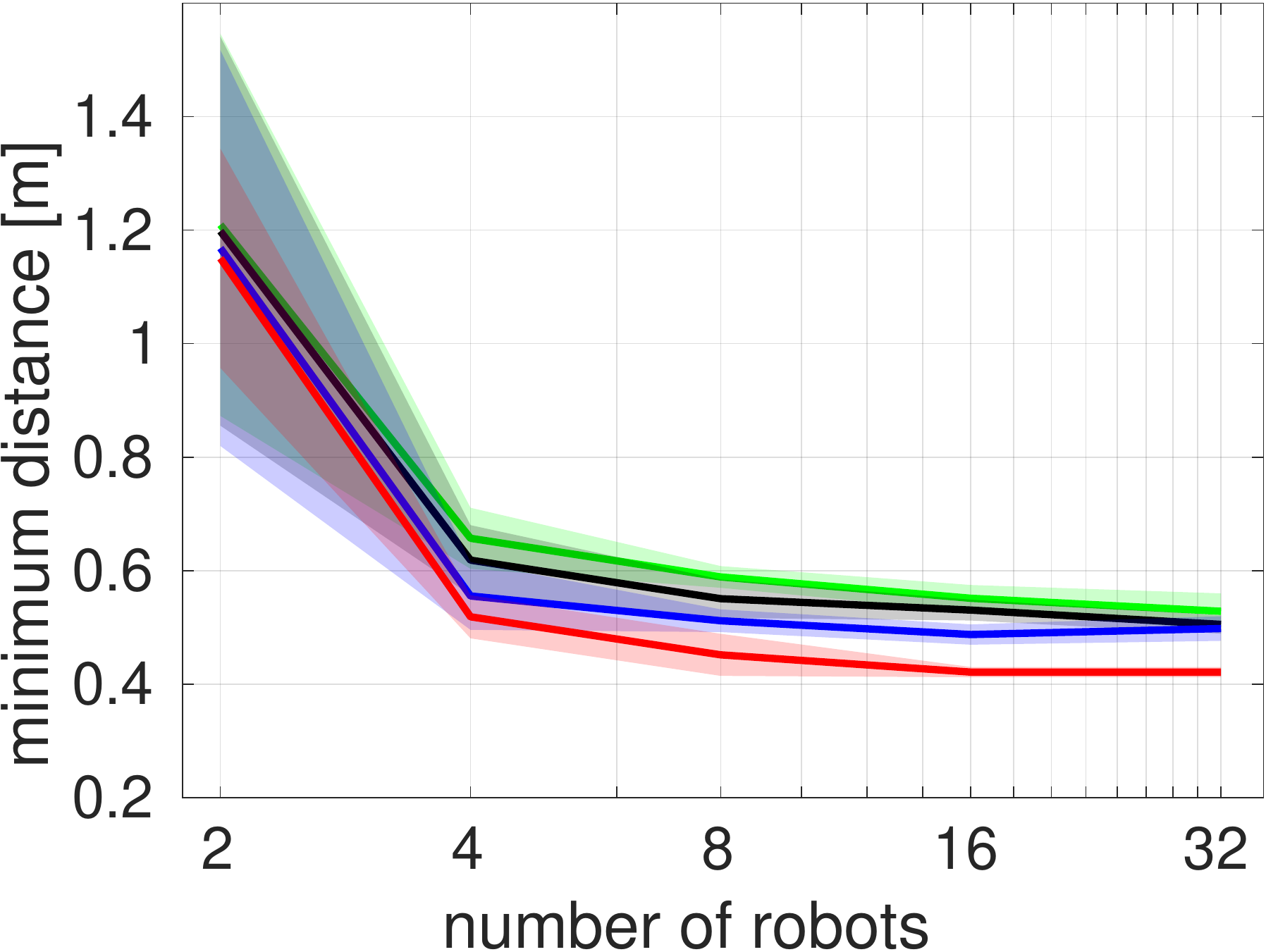}}
    \subfloat[]{\label{subfig:thresh_min_dis_rand}
       \includegraphics[width=.23\textwidth]{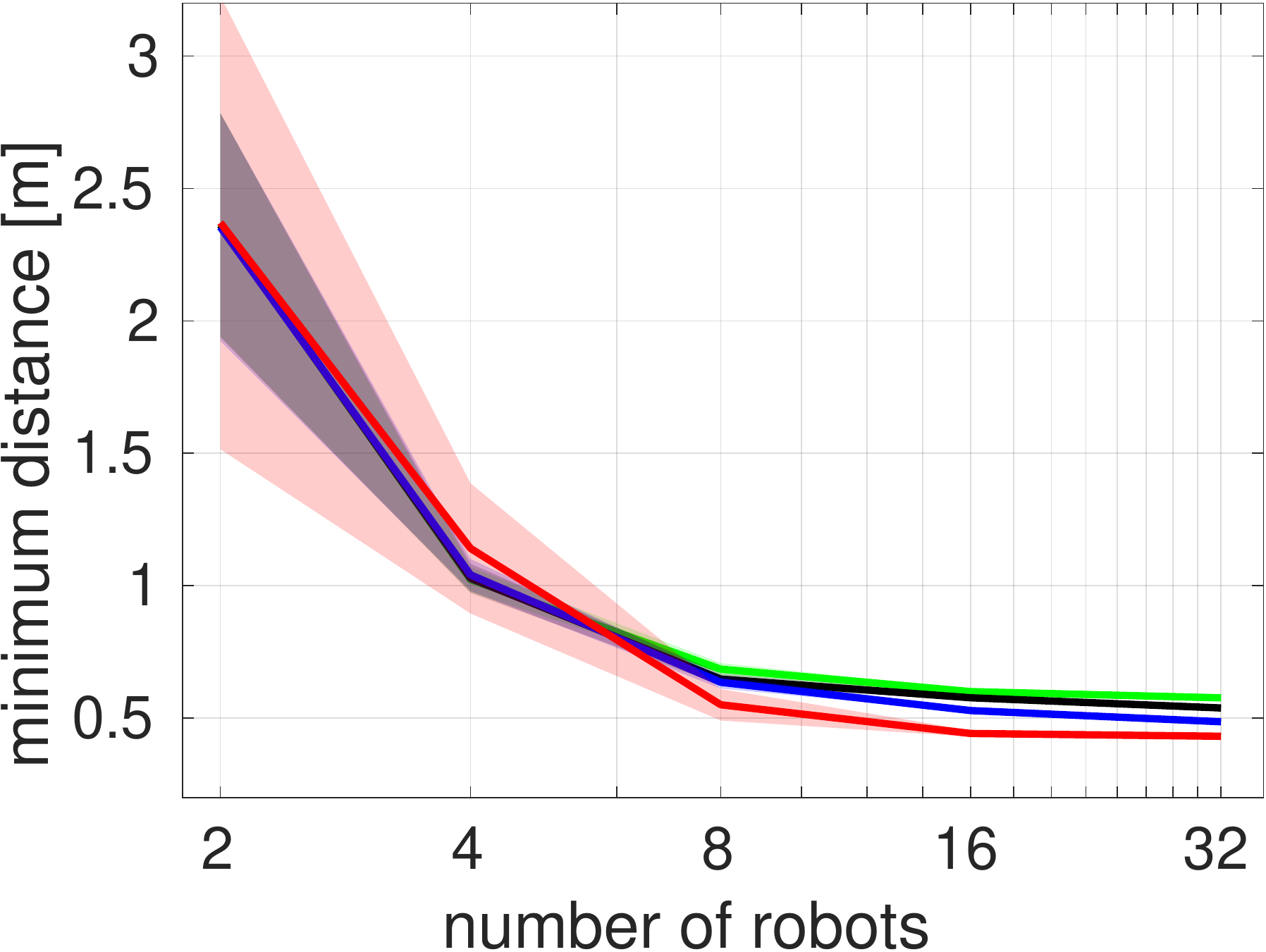}}
    \caption{Effect of the collision probability threshold on the method performance. The (a)-(b) collision rate, and (c)-(d) minimum distance among robots are shown. The evaluation has 2, 4, 8, 16, and 32 robot cases with 10 instances each. The left column shows results of the asymmetric swapping scenario and the right column shows results of the random moving scenario with 10\% obstacle density. Lines denote mean values and shaded areas around the lines denote standard deviations over 50 runs. }
    \label{fig:compare_thresh}
\end{figure}

\subsection{Simulations with Quadrotors in 3D Space}

\begin{figure}[t]
   \centering
   \subfloat{\label{}
      \includegraphics[width=.20\textwidth]{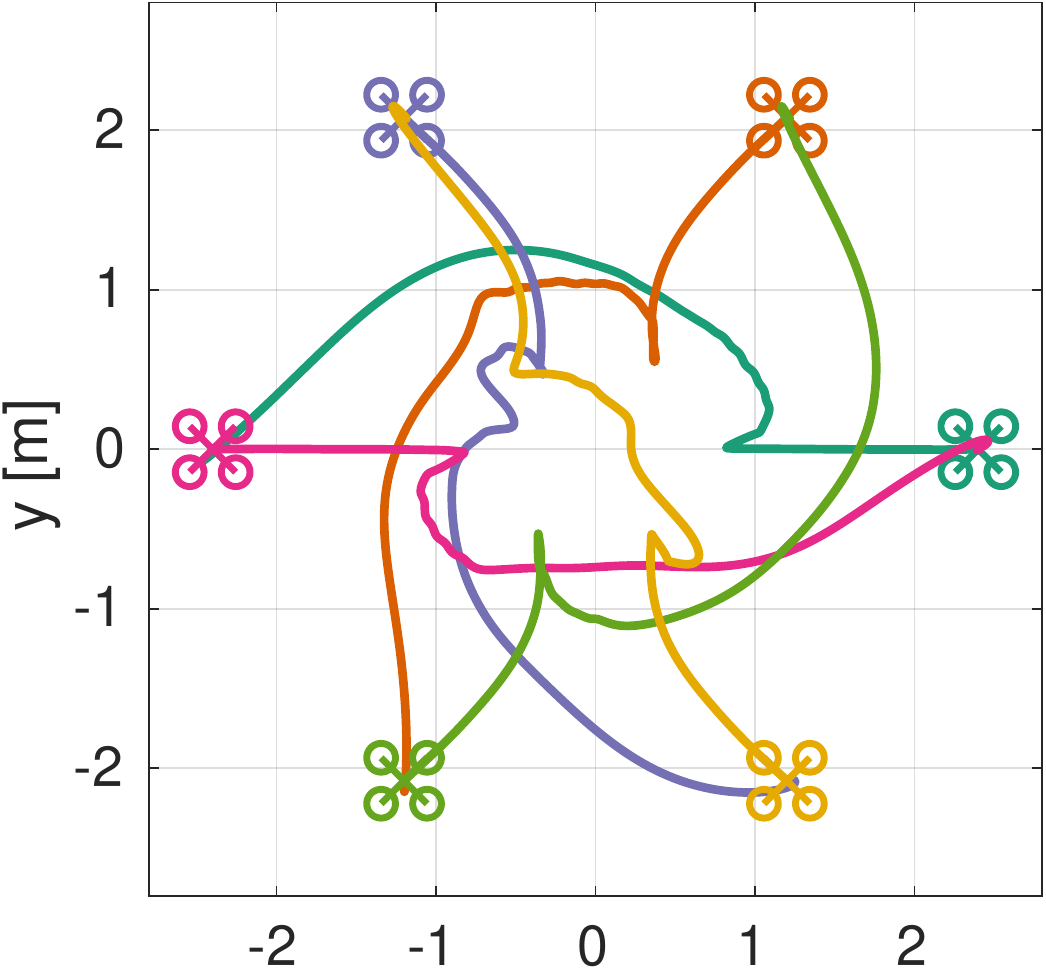}}
   \subfloat{\label{}
      \includegraphics[width=.20\textwidth]{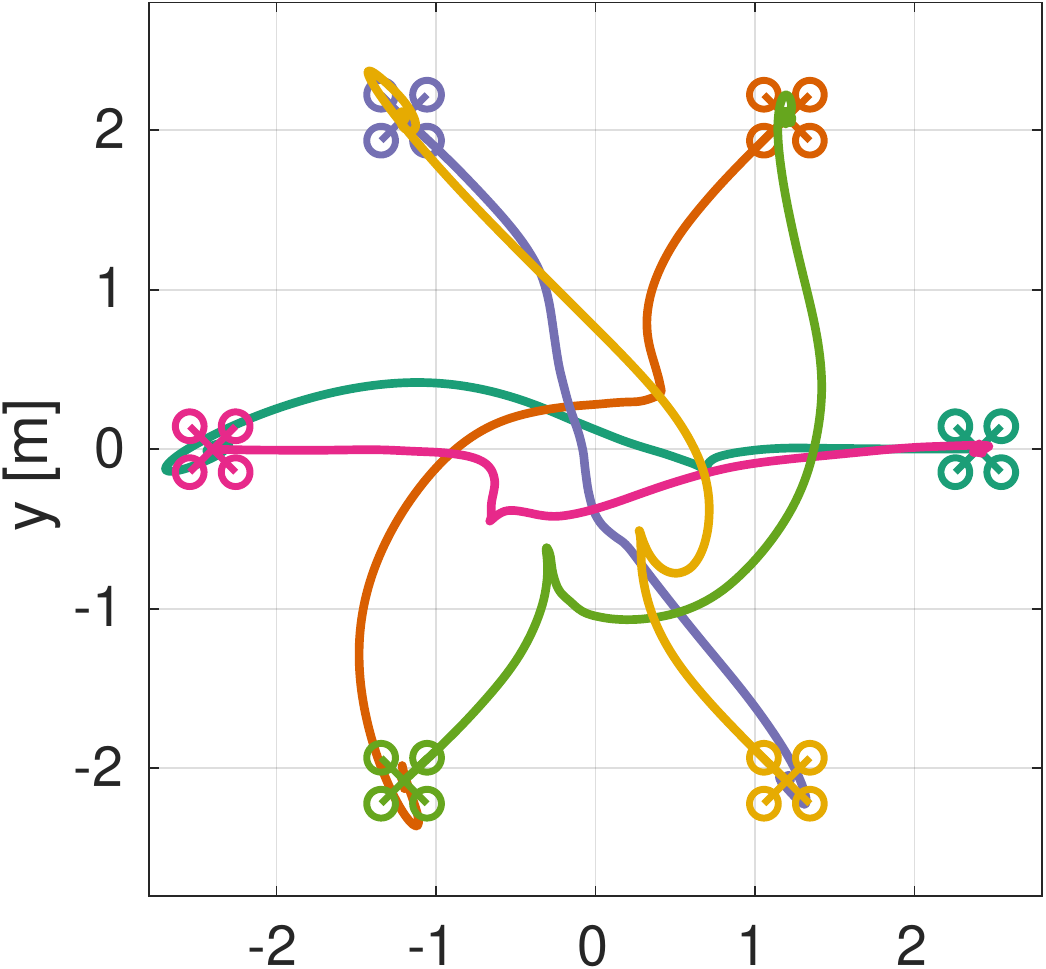}}
   \\
   ~
   \setcounter{subfigure}{0}
   \subfloat[]{\label{subfig:buavc_4}
      \includegraphics[width=.20\textwidth]{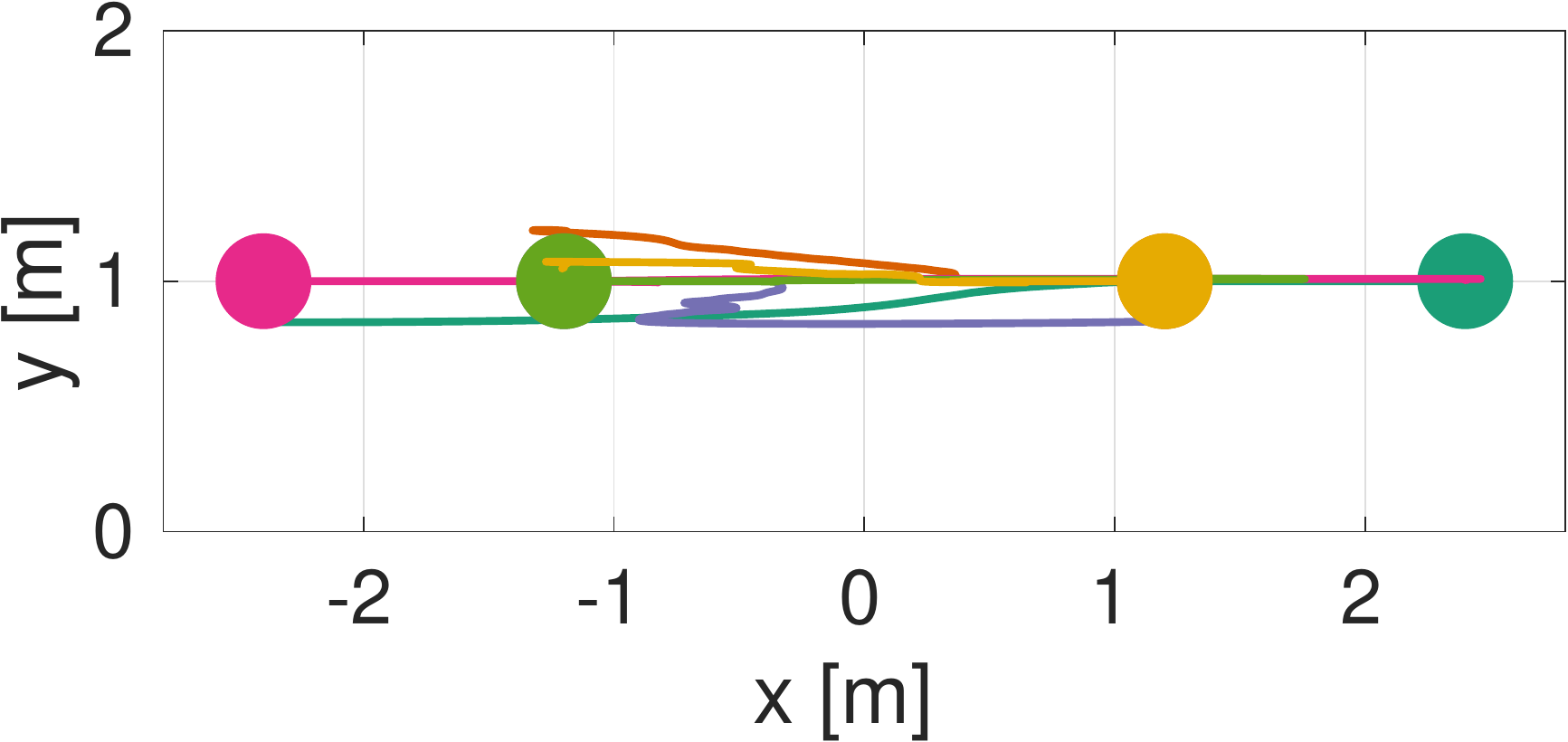}}
   \subfloat[]{\label{subfig:ccmpc_4}
      \includegraphics[width=.20\textwidth]{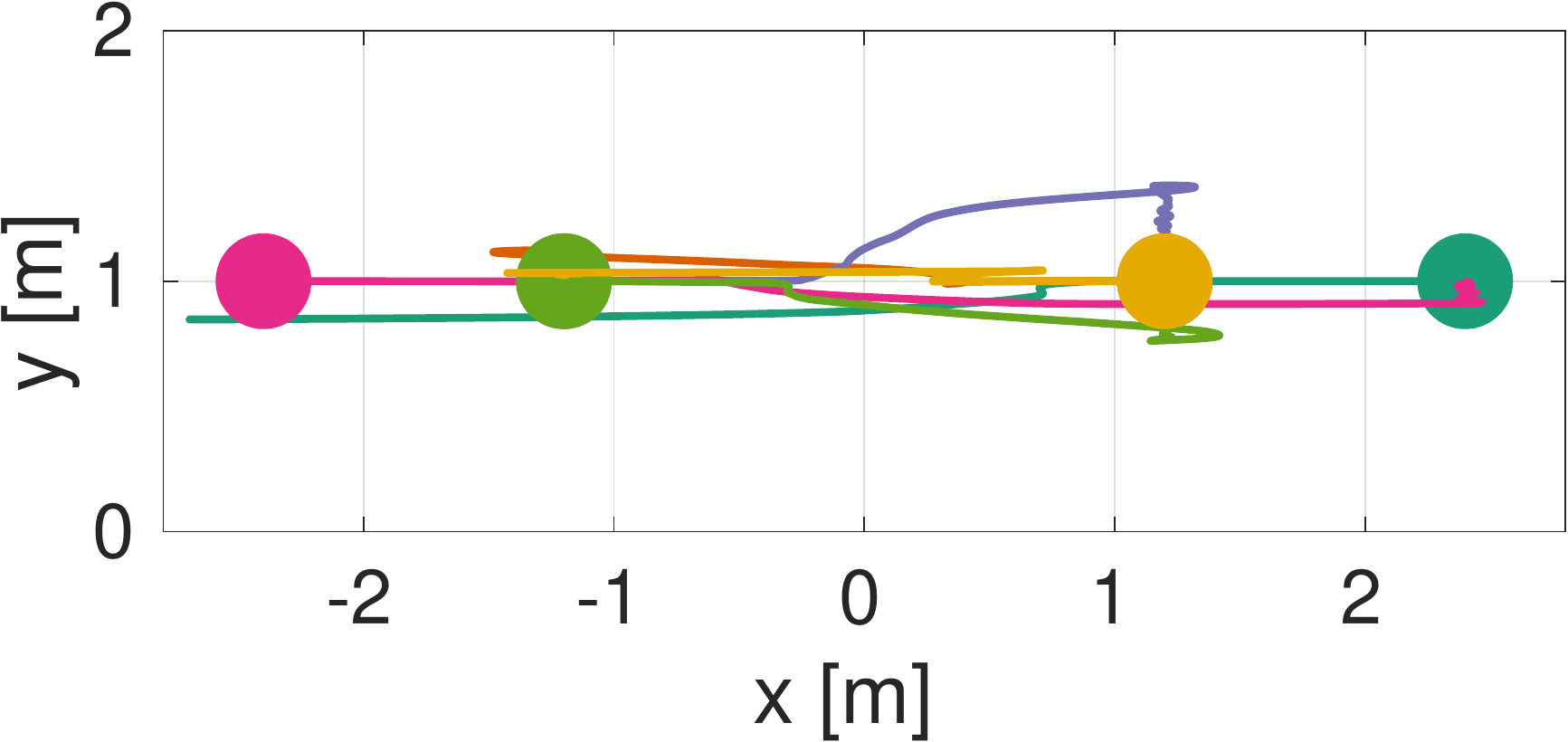}} 
   ~ 
   \caption{Simulation with six quadrotors exchanging positions in 3D space. Solid lines represent executed trajectories of the robots. (a) Results of our B-UAVC method. (b) Results of the CCNMPC method \citep{Zhu2019RAL}.}%
   \label{fig:simMPCResults}%
\end{figure}

We evaluate our receding horizon planning algorithm with quadrotors in 3D space and compare our method with one of the state-of-the-art quadrotor collision avoidance methods: the chance constrained nonlinear MPC (CCNMPC) with sequential planning \citep{Zhu2019RAL}, which requires communication of future planned trajectories among robots. For both methods, we adopt the same quadrotor dynamics model for planning. The quadrotor radius is set as $r = 0.3$ m and the collision probability threshold is set to $\delta = 0.03$. The time step is $\Delta t = 0.05$ s and the total number of steps is $N = 20$ resulting in a planing horizon of one second. 

As shown in Fig. \ref{fig:simMPCResults}, we simulate with six quadrotors exchanging their initial positions in an obstacle-free 3D space. Each quadrotor is under localization uncertainty $\Sigma = \tn{diag}(0.04~\tn{m},~0.04~\tn{m},~0.04~\tn{m})^2$. For each method, we run the simulation 10 times and calculate the minimum distance among robots. Both our B-UAVC method and the CCNMPC method successfully navigates all robots without collision. An average minimum distance of 0.72 m is observed in our B-UAVC method, while the one of CCNMPC is 0.62 m, which indicates our method is more conservative than the CCNMPC. However, the CCNMPC is centralized and requires robots to communicate their future planned trajectories with each other, while the B-UAVC method only needs robot positions to be shared or sensed.

\section{Experimental Validation}\label{sec:exp_result}
In this section we describe the experimental results with a team of real robots. A video demonstrating the results accompanies this paper. 

\begin{figure*}[t]
    \centering
    \subfloat{\label{}
       \includegraphics[width=.24\textwidth]{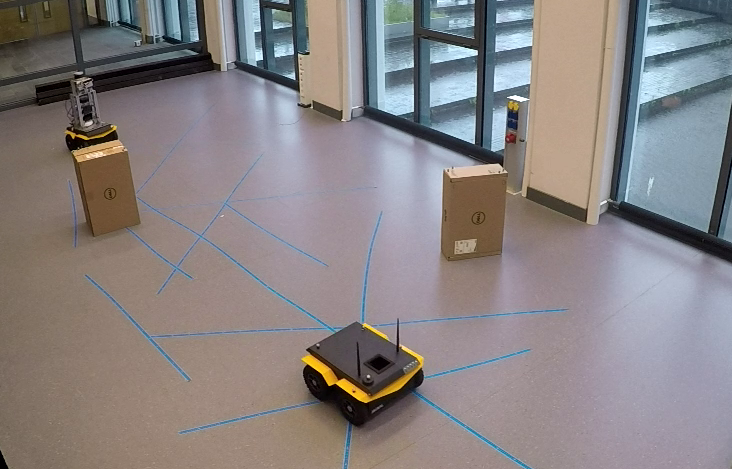}}
    \subfloat{\label{}
       \includegraphics[width=.24\textwidth]{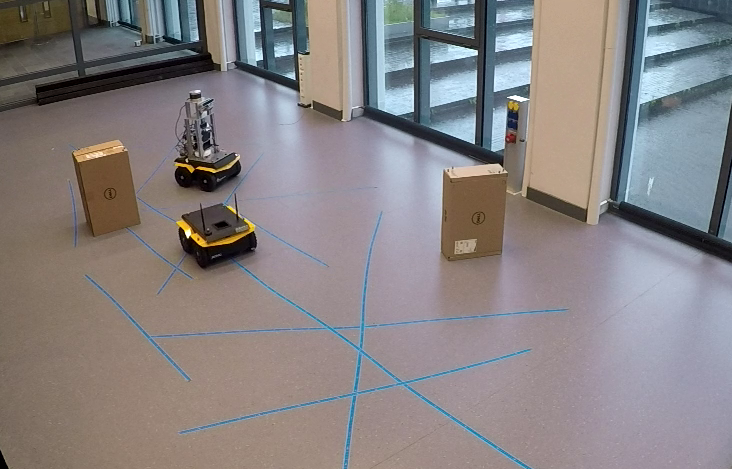}}
    \subfloat{\label{}
       \includegraphics[width=.24\textwidth]{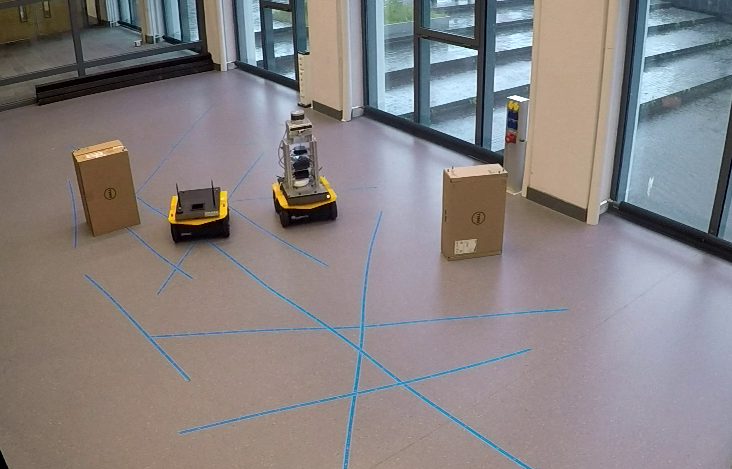}}
    \subfloat{\label{}
       \includegraphics[width=.24\textwidth]{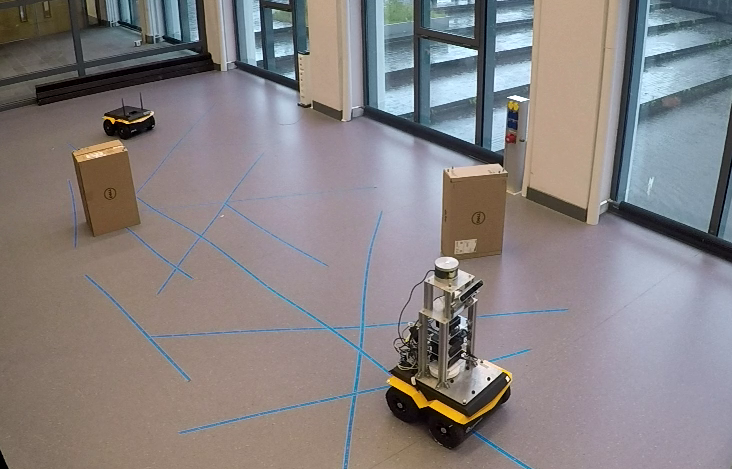}}
    \\
    \setcounter{subfigure}{0}
    \subfloat[$t$ = 2 s.]{\label{}
       \includegraphics[width=.24\textwidth]{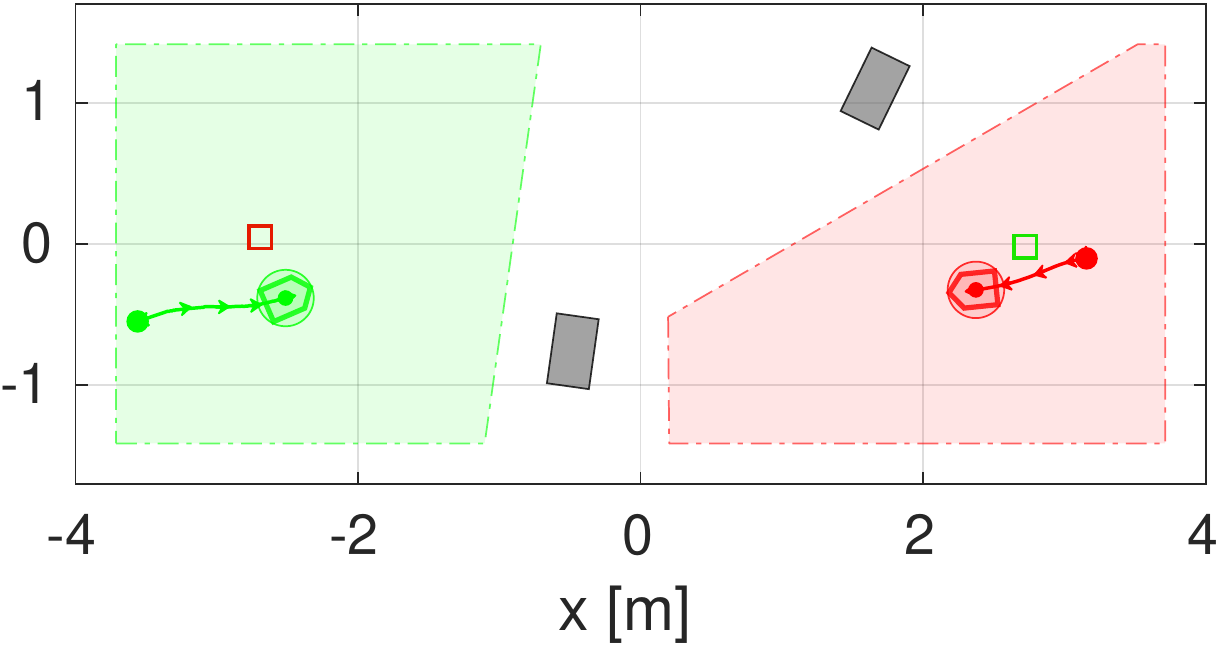}}
    \subfloat[$t$ = 6 s.]{\label{}
       \includegraphics[width=.24\textwidth]{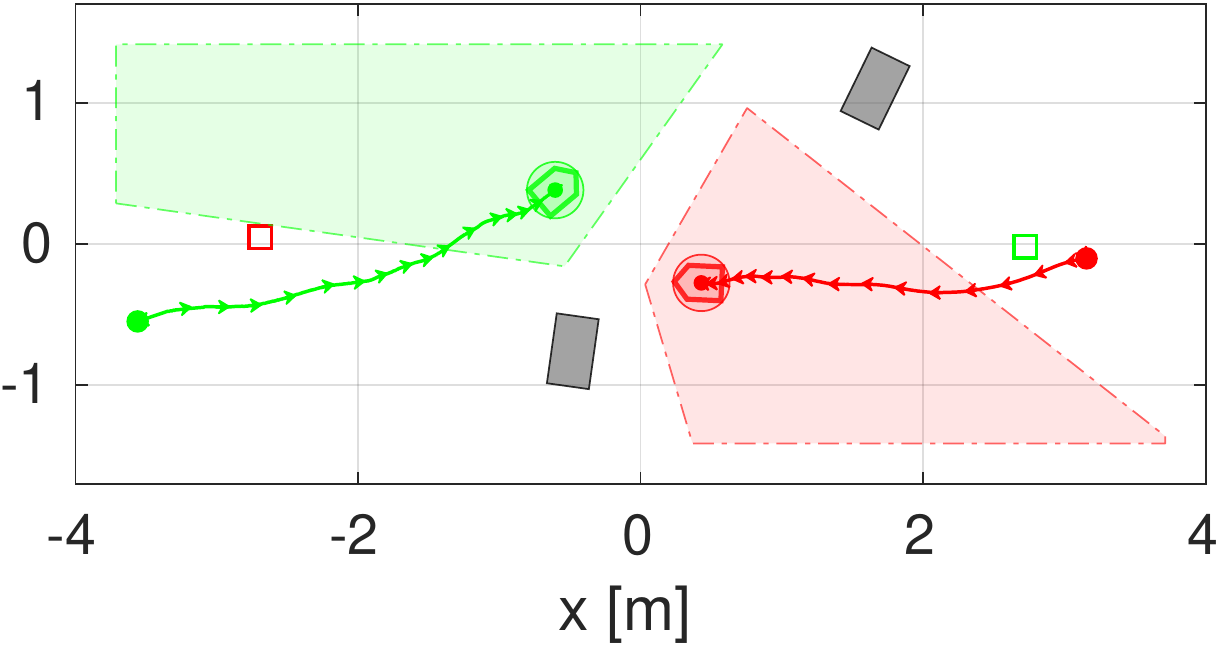}}
    \subfloat[$t$ = 8 s.]{\label{}
       \includegraphics[width=.24\textwidth]{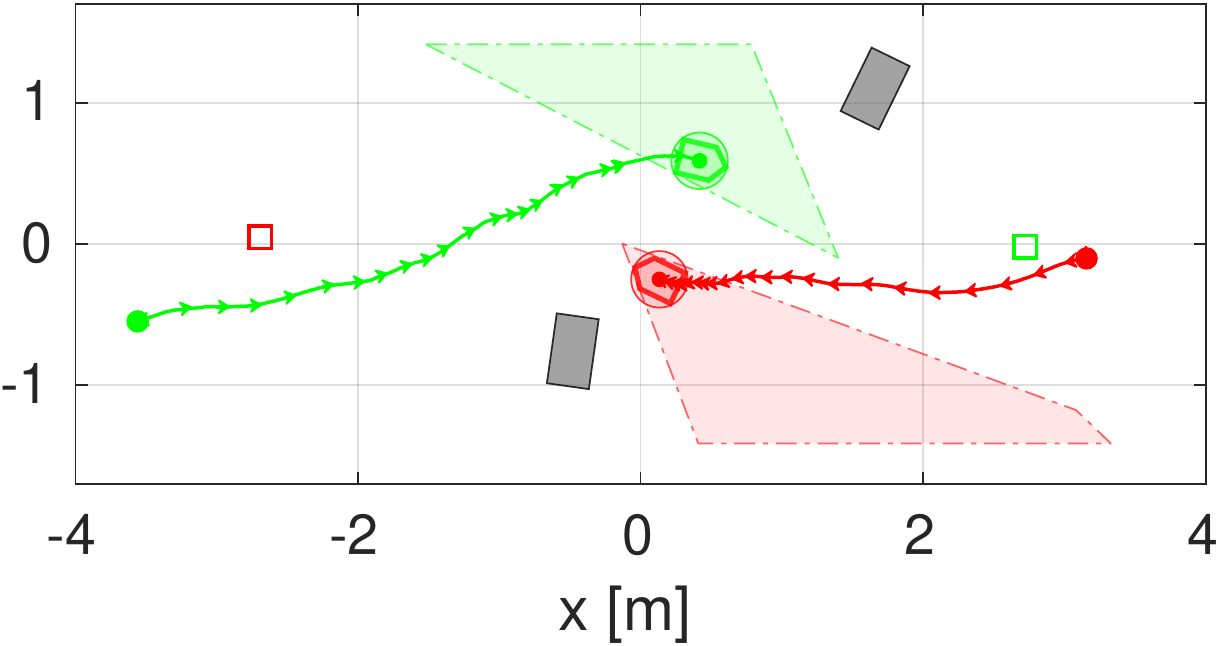}}
    \subfloat[$t$ = 15 s.]{\label{}
       \includegraphics[width=.24\textwidth]{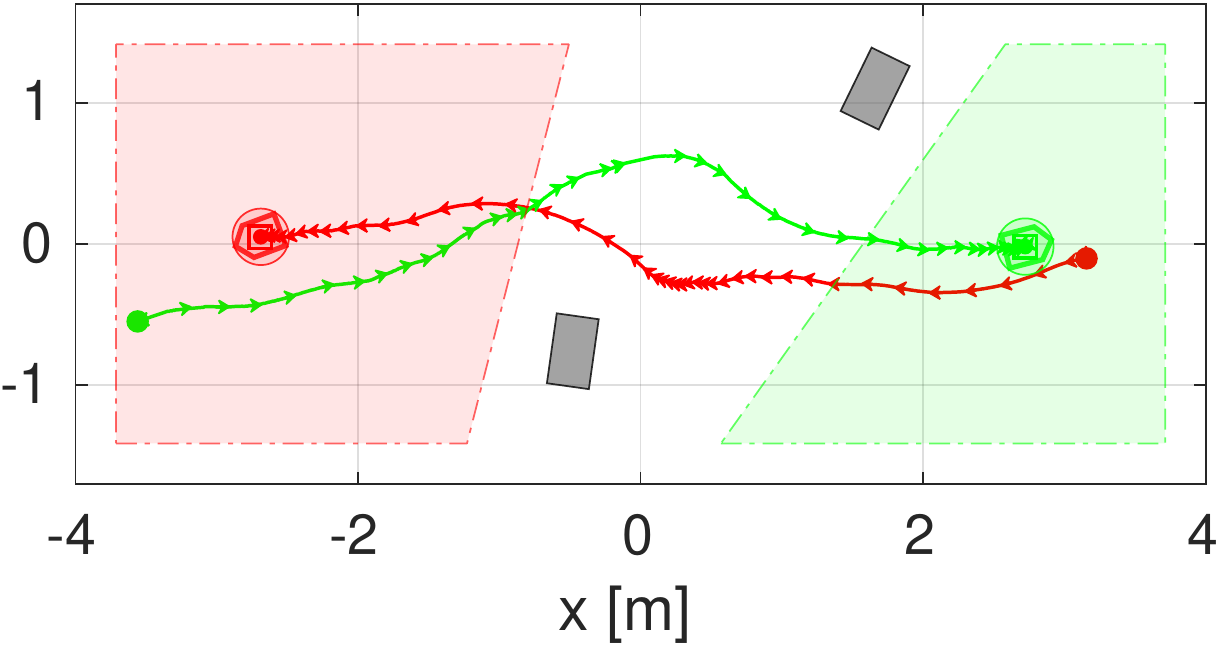}}
 
    \caption{Collision avoidance with two differential-drive robots and two static obstacles. The two robots are required to swap their positions. Top row: Snapshots of the experiment. Bottom row: Trajectories of the robots. Robot initial and goal positions are marked in circle disks and solid squares, respectively. Grey boxes are static obstacles. The B-UAVCs are shown in shaded patches with dashed boundaries. }%
    \label{fig:exp_jackal}
 \end{figure*}
 
 \begin{figure}[h]
    \centering
    \subfloat[]{\label{subfig:jackal_dis}
       \includegraphics[width=.23\textwidth]{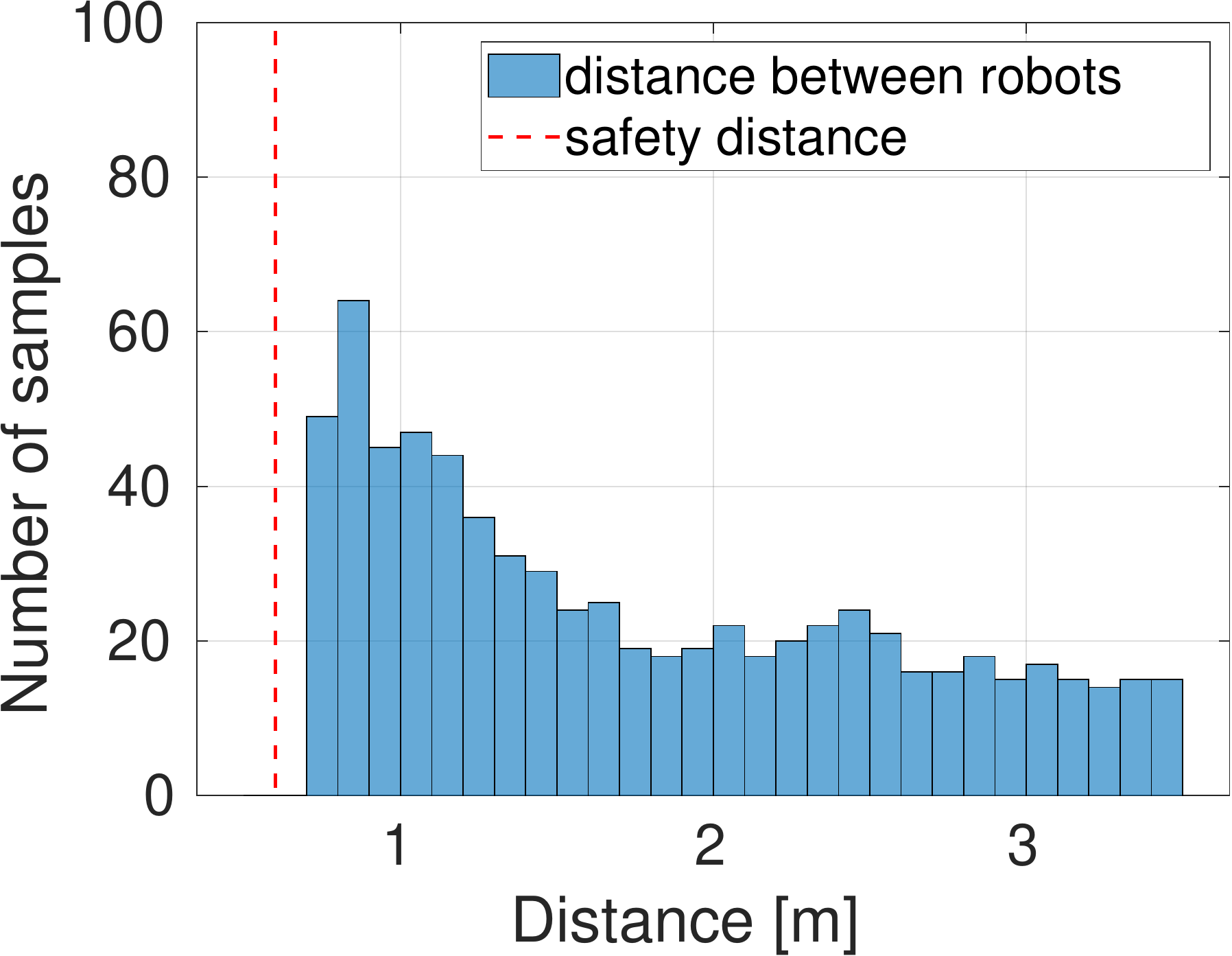}}
    \subfloat[]{\label{subfig:jackal_dis_obs}
       \includegraphics[width=.23\textwidth]{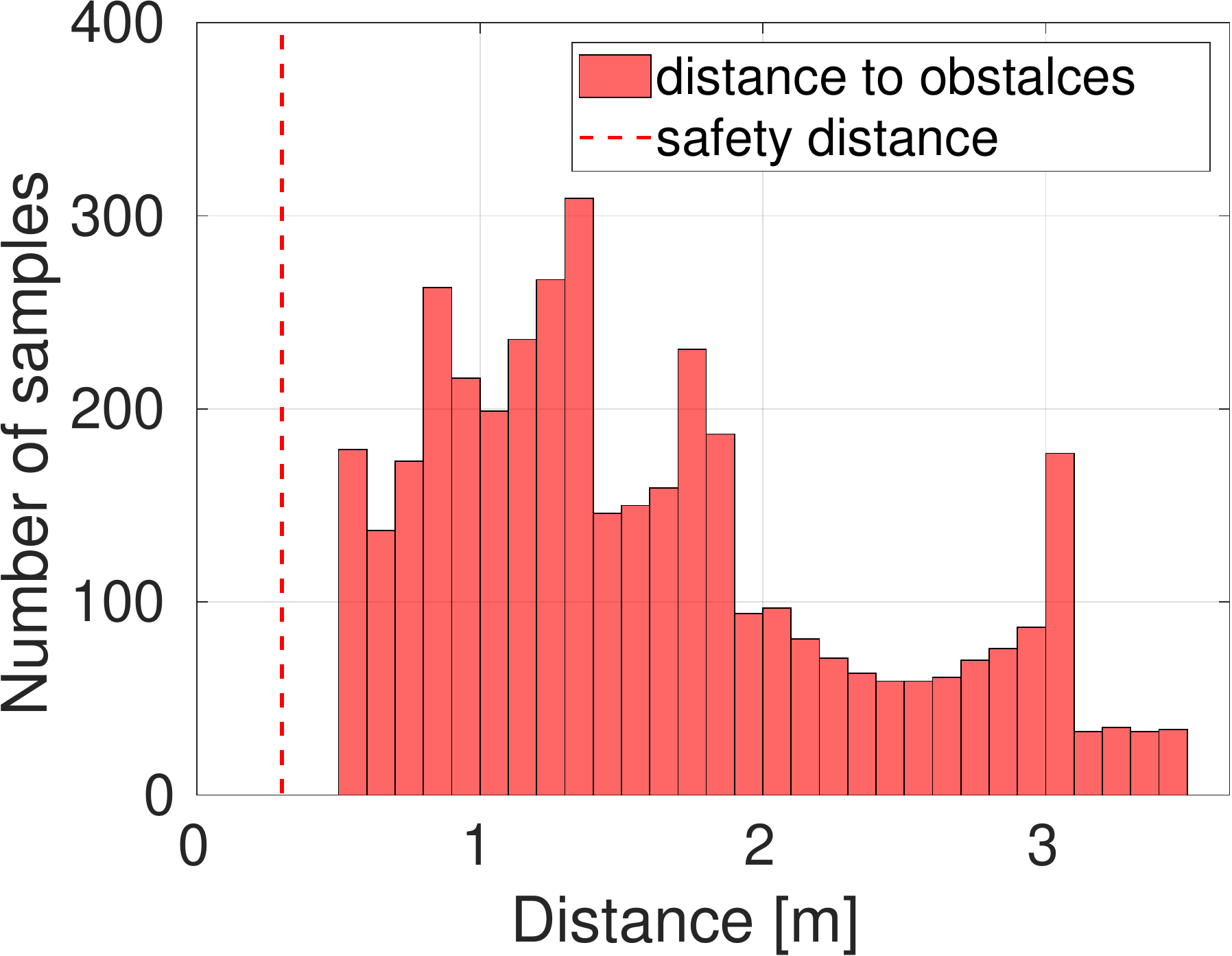}}
 
    \caption{Experimental results with two differential-drive robots. (a) Histogram of inter-robot distance. (b) Histogram of distance between robots and obstacles. }%
    \label{fig:exp_jackal_sta}
 \end{figure}

\subsection{Experimental Setup}
We test our proposed approach on both ground vehicles and aerial vehicles in an indoor environment of 8m (L) $\times$ 3.4m (W) $\times$ 2.5m (H). Our ground vehicle platform is the Clearpath Jackal robot and our aerial vehicle platform is the Parrot Bebop 2 quadrotor. For ground vehicles, we apply the controller designed for differential-drive robots as shown in Section \ref{subsubsec:control_diff}. For quadrotors, the receding horizon trajectory planner presented in Section \ref{subsec:control_mpc} is employed.  
\rebuttal{The quadrotor dynamics model $\vf$ in Problem 1 is given in Appendix \ref{appendix:quad_model}. For solving Problem 1 which is a nonlinear programming problem, we rely on the solver Forces Pro \citep{Zanelli2020} to generate fast C code to solve it.}
Both types of robots allow executing control commands sent via ROS. The experiments are conducted in a standard laptop (Quadcore Intel i7 CPU@2.6 GHz) which connects with the robots via WiFi. 

\rebuttal{An external motion capture system (OptiTrack) is used to track the pose (position and orientation) of each robot and obstacle in the environment running in real time at 120 Hz, which is regarded as the \emph{real} (ground-truth) pose. To validate collision avoidance under uncertainty, we then manually add Gaussian noise to the real pose data to generate noisy measurements. Taking the noisy measurements as inputs, a standard Kalman filter running at 120 Hz is employed to estimate the states of the robots and obstacles.} In all experiments, the added position measurements noise to the robots is zero mean with covariance $\Sigma_i^{\prime} = \tn{diag}(0.06~\tn{m}, 0.06~\tn{m}, 0.06~\tn{m})^2$, which results in an average estimated position uncertainty covariance $\Sigma_i = \tn{diag}(0.04~\tn{m}, 0.04~\tn{m}, 0.04~\tn{m})^2$. The added noise to the obstacles is zero mean with covariance $\Sigma_o^{\prime} = \tn{diag}(0.03~\tn{m}, 0.03~\tn{m}, 0.03~\tn{m})^2$ and the resulted estimated position uncertainty covariance is $\Sigma_o = \tn{diag}(0.02~\tn{m}, 0.02~\tn{m}, 0.02~\tn{m})^2$. The collision probability threshold is set as $\delta = 0.03$ as in previous works \citep{Zhu2019MRS,Zhu2019RAL}.

\begin{figure*}[h]
   \centering
   \subfloat{\label{}
      \includegraphics[width=.3\textwidth]{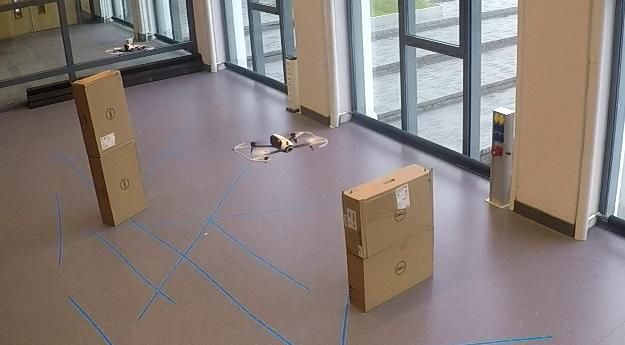}}
   \subfloat{\label{}
      \includegraphics[width=.3\textwidth]{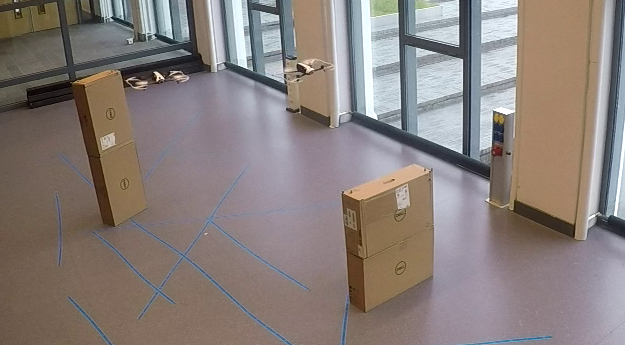}}
   \subfloat{\label{}
      \includegraphics[width=.3\textwidth]{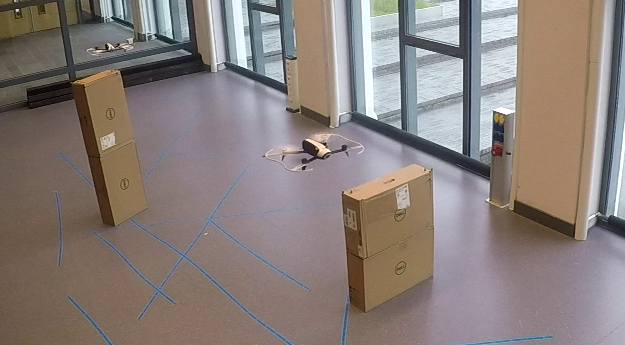}}
   \\
   \setcounter{subfigure}{0}
   \subfloat[$t$ = 0.1 s.]{\label{}
      \includegraphics[width=.33\textwidth]{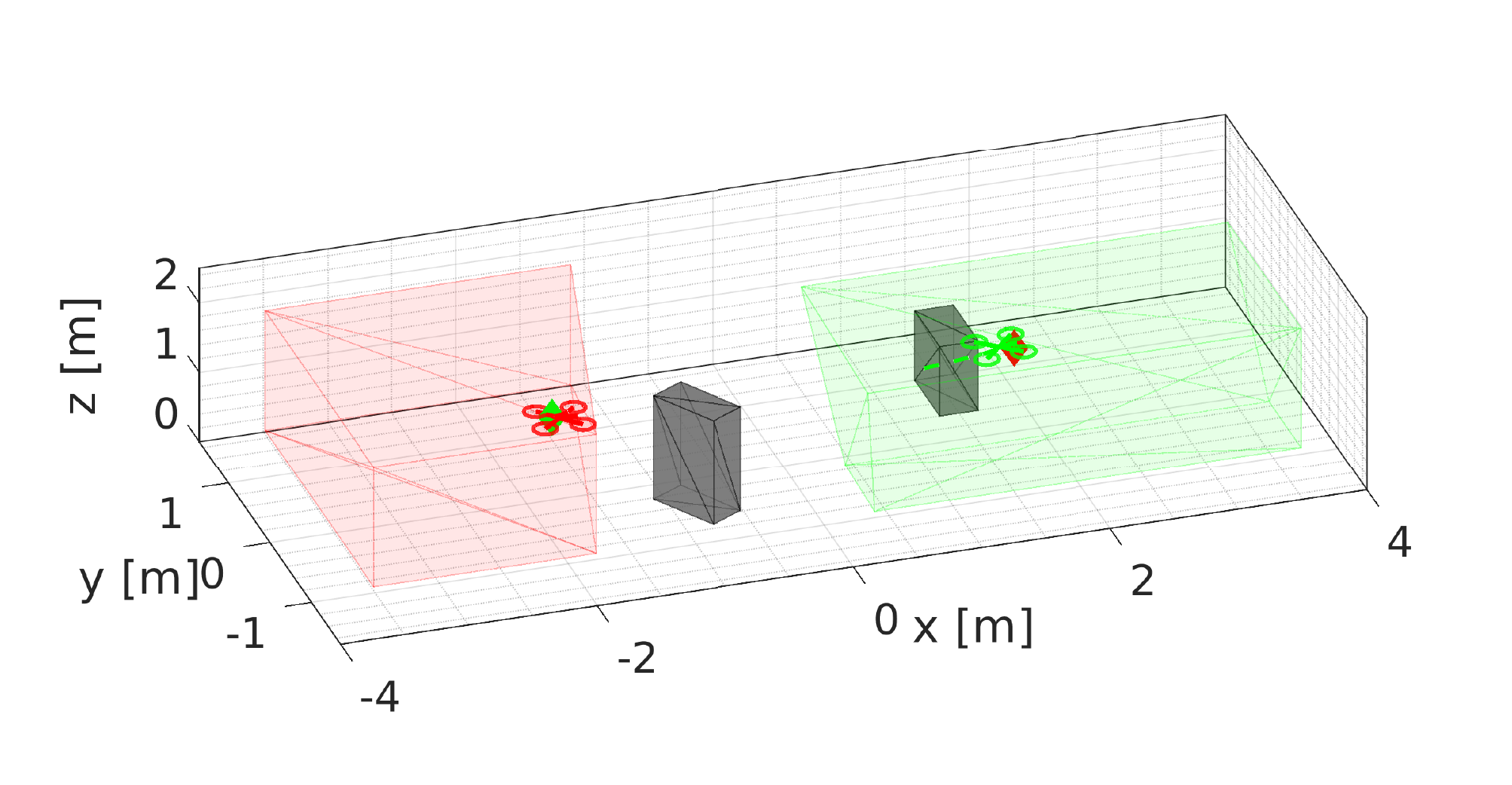}}
   \subfloat[$t$ = 5 s.]{\label{}
      \includegraphics[width=.33\textwidth]{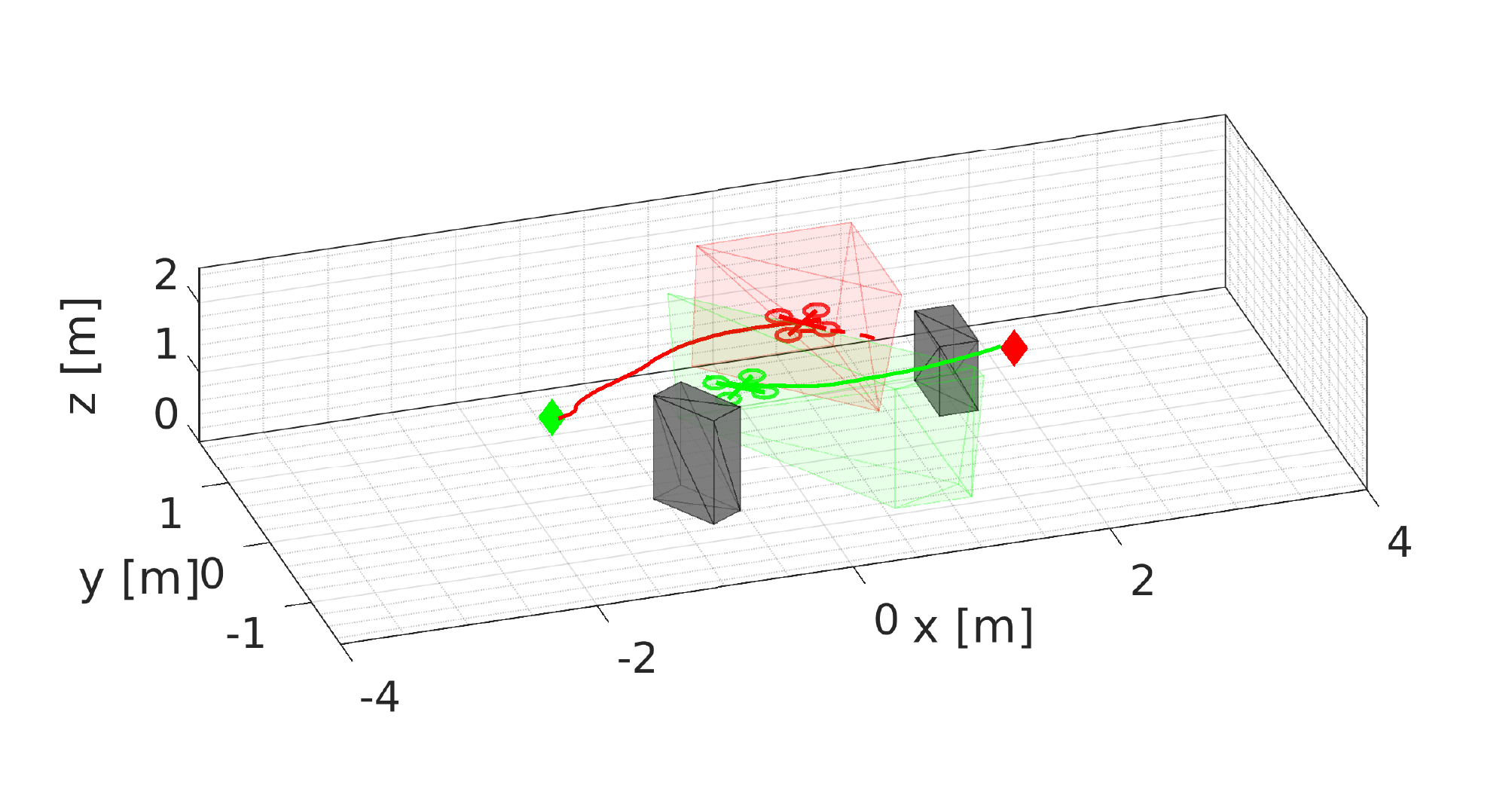}}
   \subfloat[$t$ = 10.5 s.]{\label{}
      \includegraphics[width=.33\textwidth]{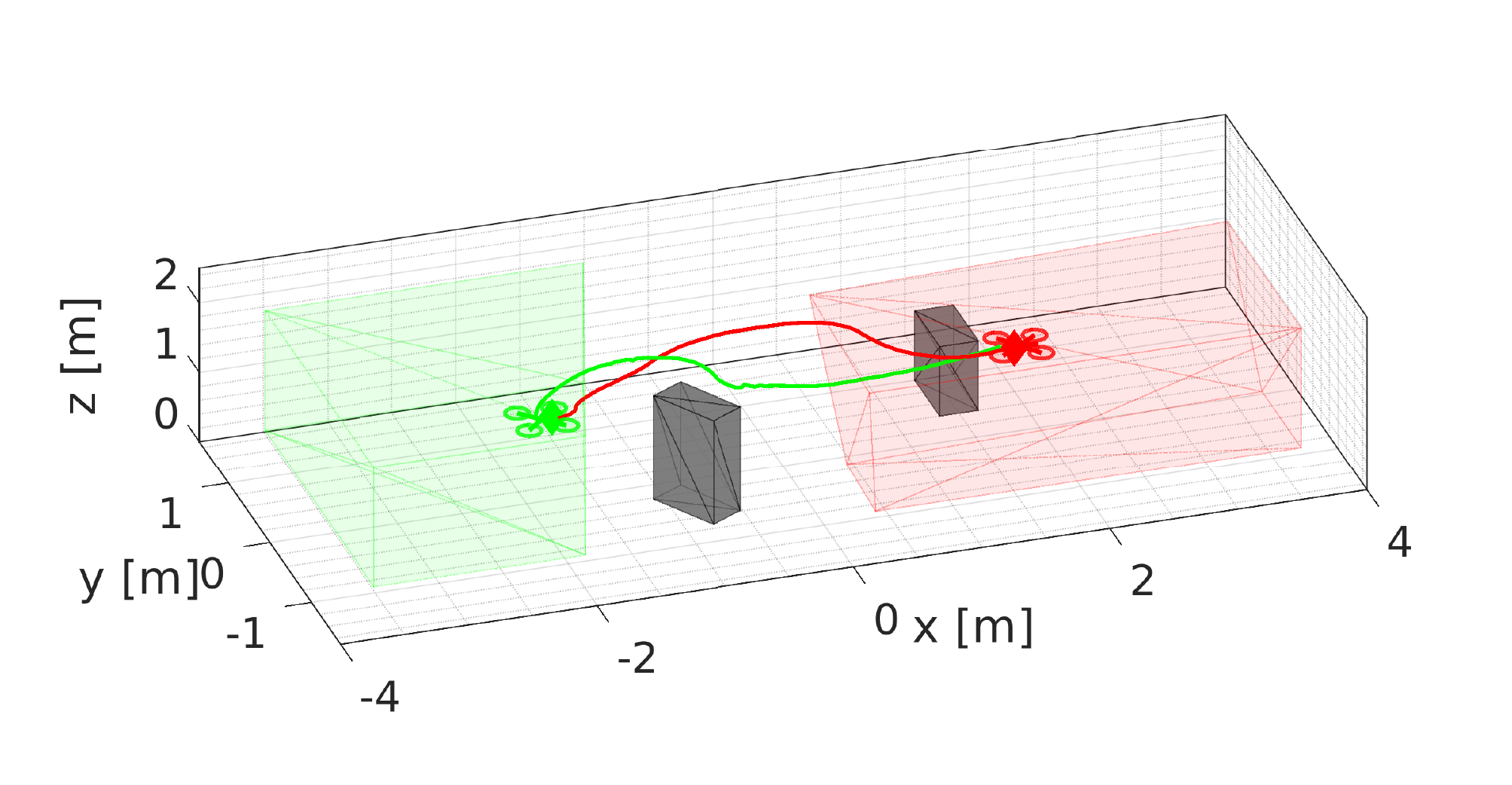}}

   \caption{Collision avoidance with two quadrotors and two static obstacles. The two quadrotors are required to swap their positions. Top row: Snapshots of the experiment. Bottom row: Trajectories of the robots. Quadrotor initial and goal positions are marked in circles and diamonds. Solid lines represent travelled trajectories and dashed lines represent planned trajectories.}%
   \label{fig:exp_quad_2}
\end{figure*}

\begin{figure*}[h]
   \centering
   \subfloat{\label{}
      \includegraphics[width=.3\textwidth]{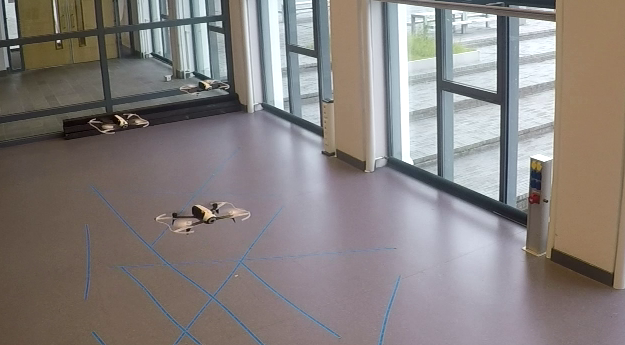}}
   \subfloat{\label{}
      \includegraphics[width=.3\textwidth]{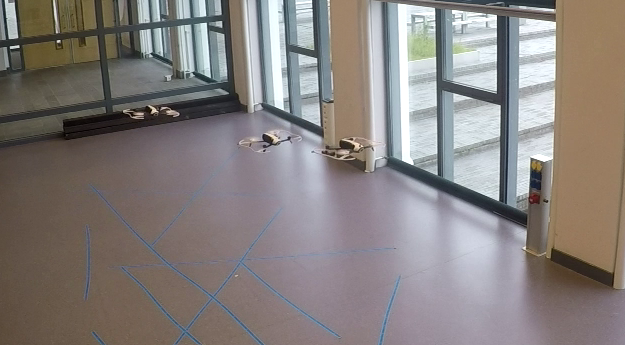}}
   \subfloat{\label{}
      \includegraphics[width=.3\textwidth]{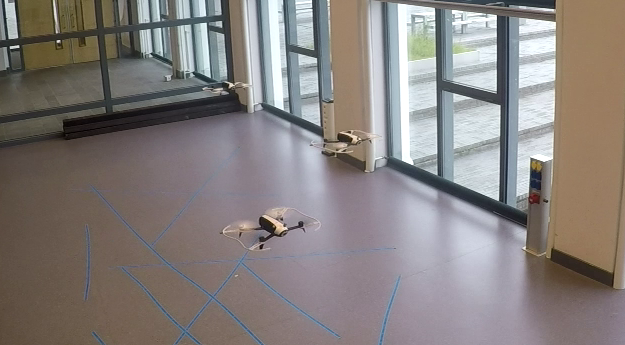}}
   \\
   \setcounter{subfigure}{0}
   \subfloat[$t$ = 0.1 s.]{\label{}
      \includegraphics[width=.33\textwidth]{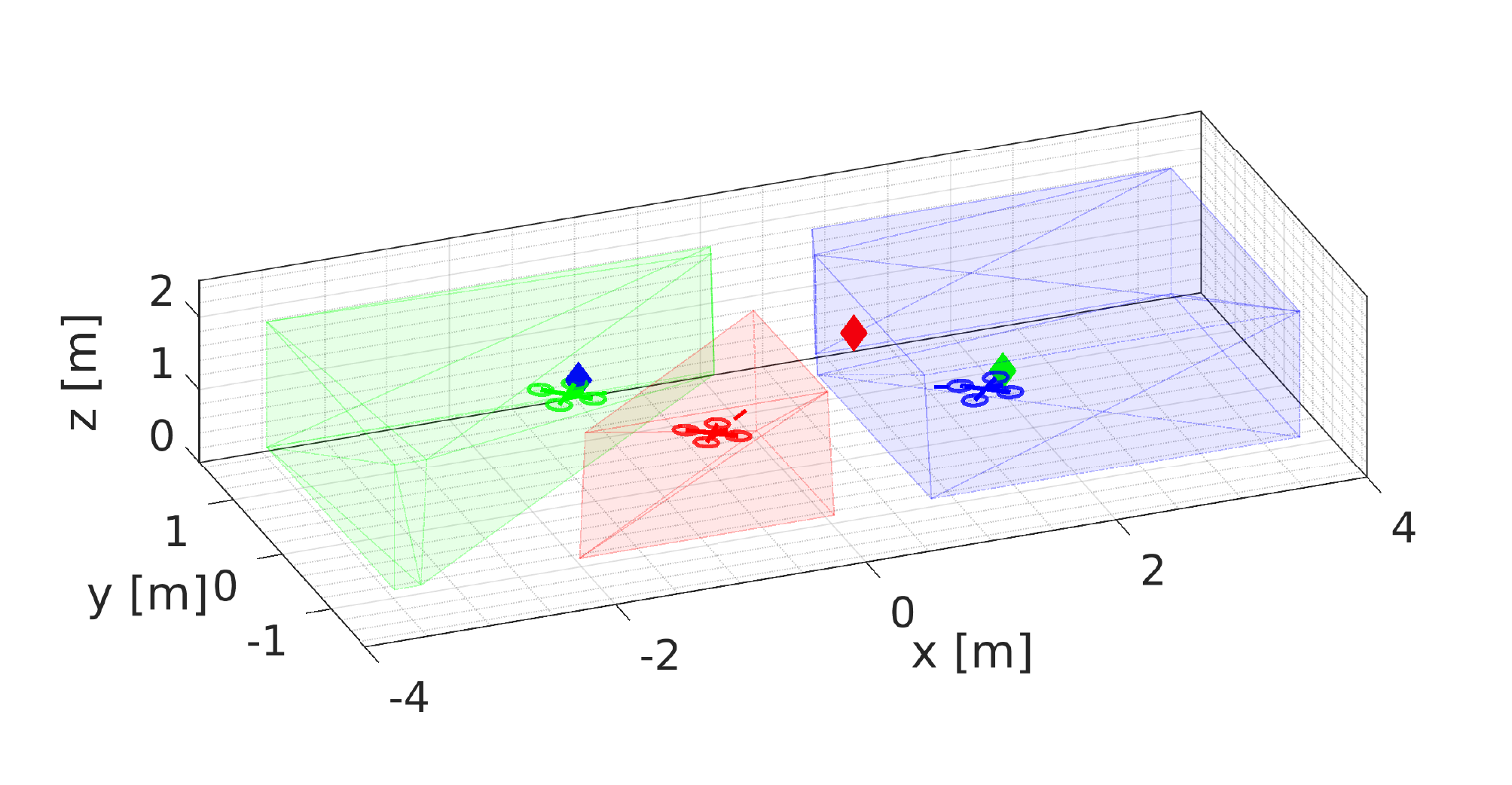}}
   \subfloat[$t$ = 6 s.]{\label{}
      \includegraphics[width=.33\textwidth]{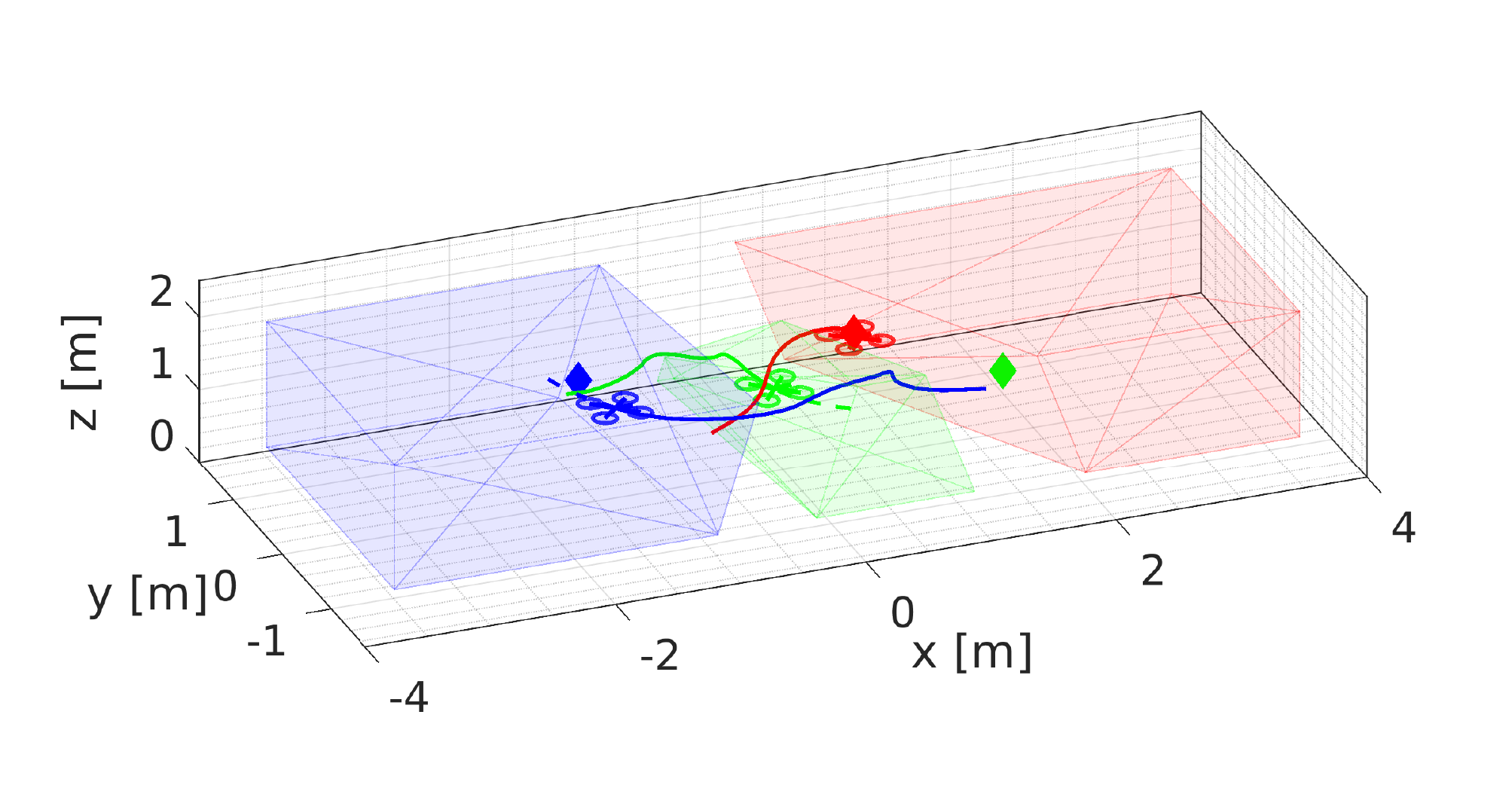}}
   \subfloat[$t$ = 10 s.]{\label{}
      \includegraphics[width=.33\textwidth]{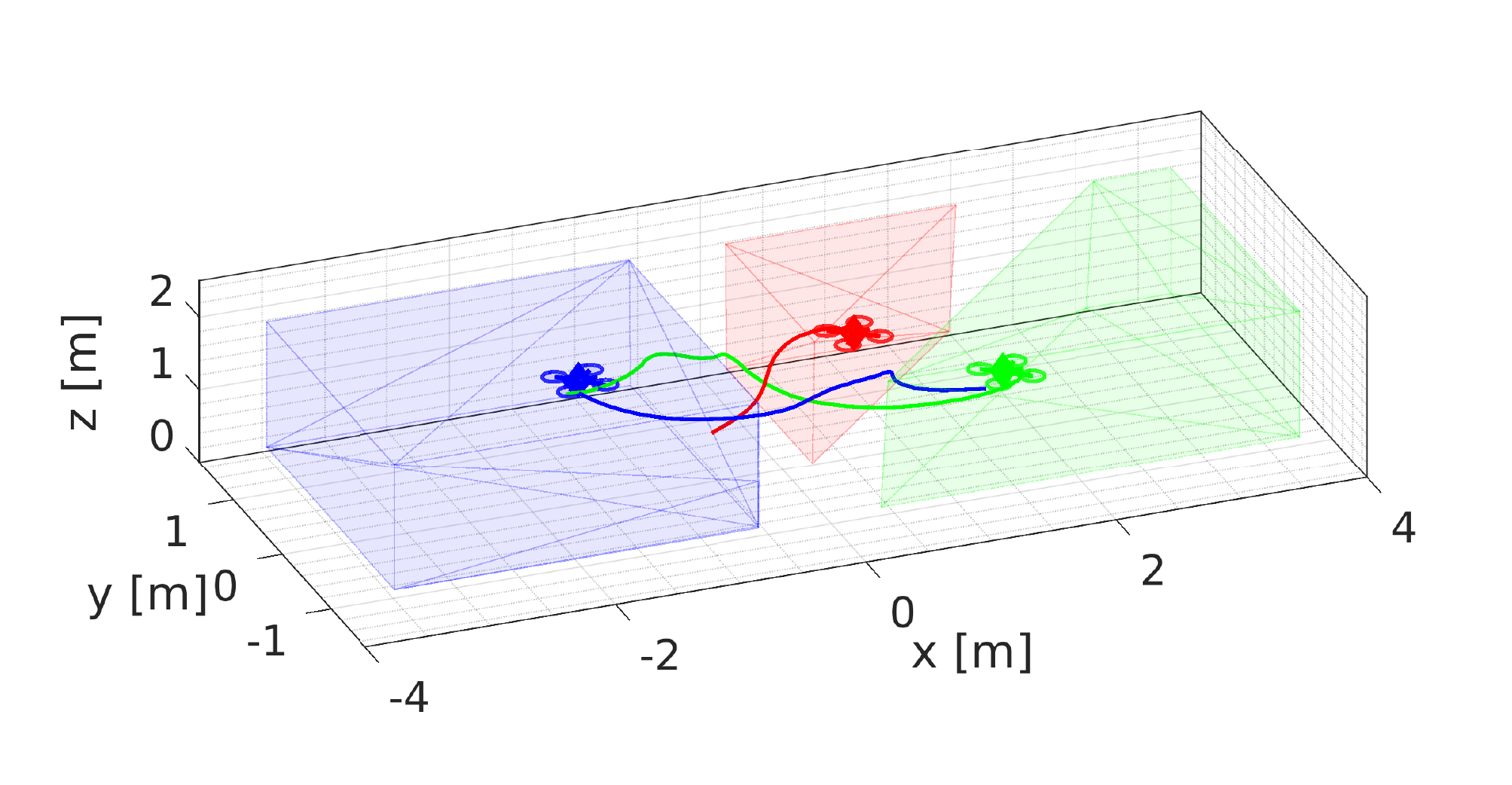}}

   \caption{Collision avoidance with three quadrotors in a shared workspace. Top row: Snapshots of the experiment. Bottom row: Trajectories of the robots.}%
   \label{fig:exp_quad_3}
\end{figure*}

\begin{figure}[t]
   \centering
   \subfloat[]{\label{subfig:quad_dis_his}
      \includegraphics[width=.23\textwidth]{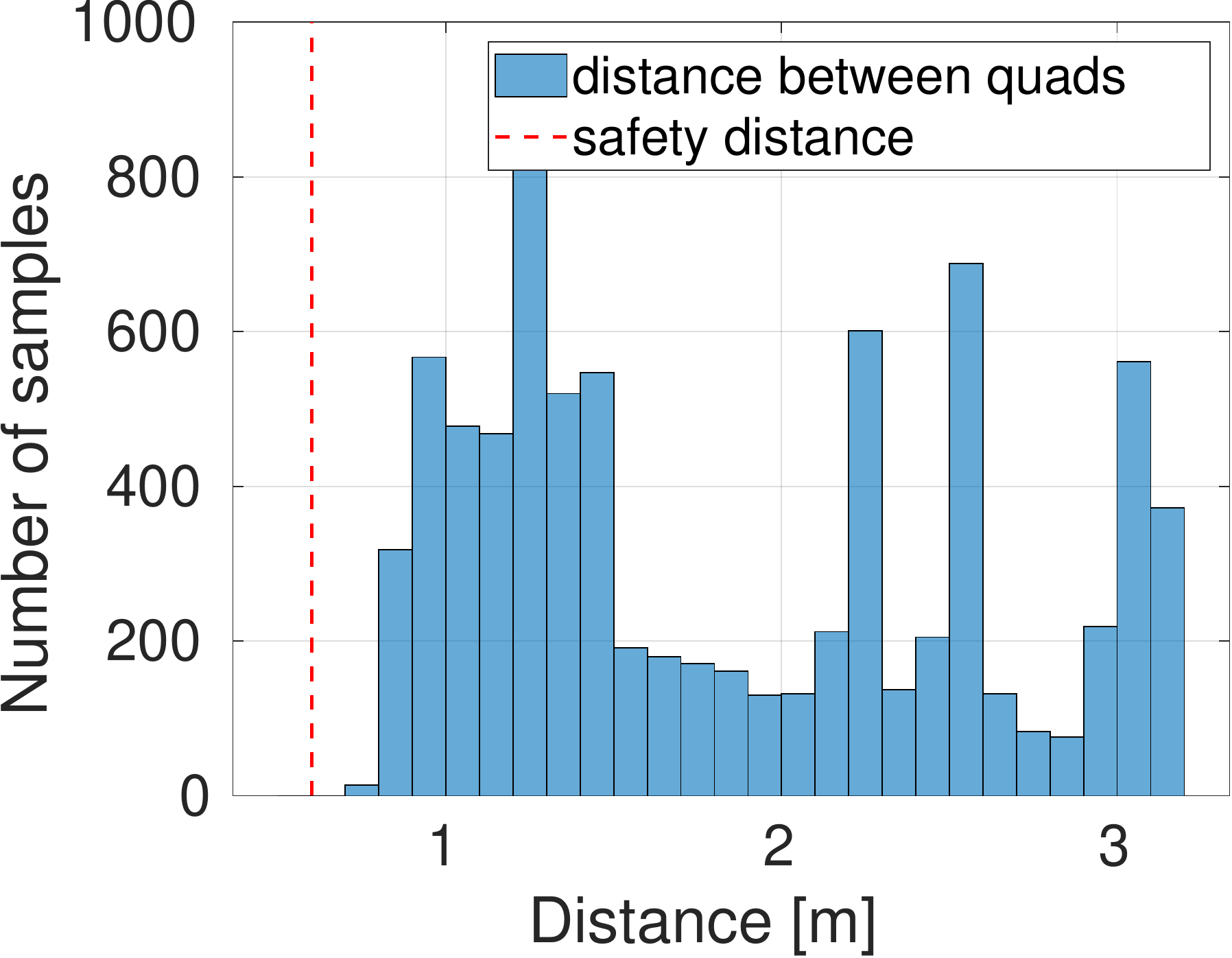}}
   \subfloat[]{\label{subfig:quad_obs_dis_his}
      \includegraphics[width=.23\textwidth]{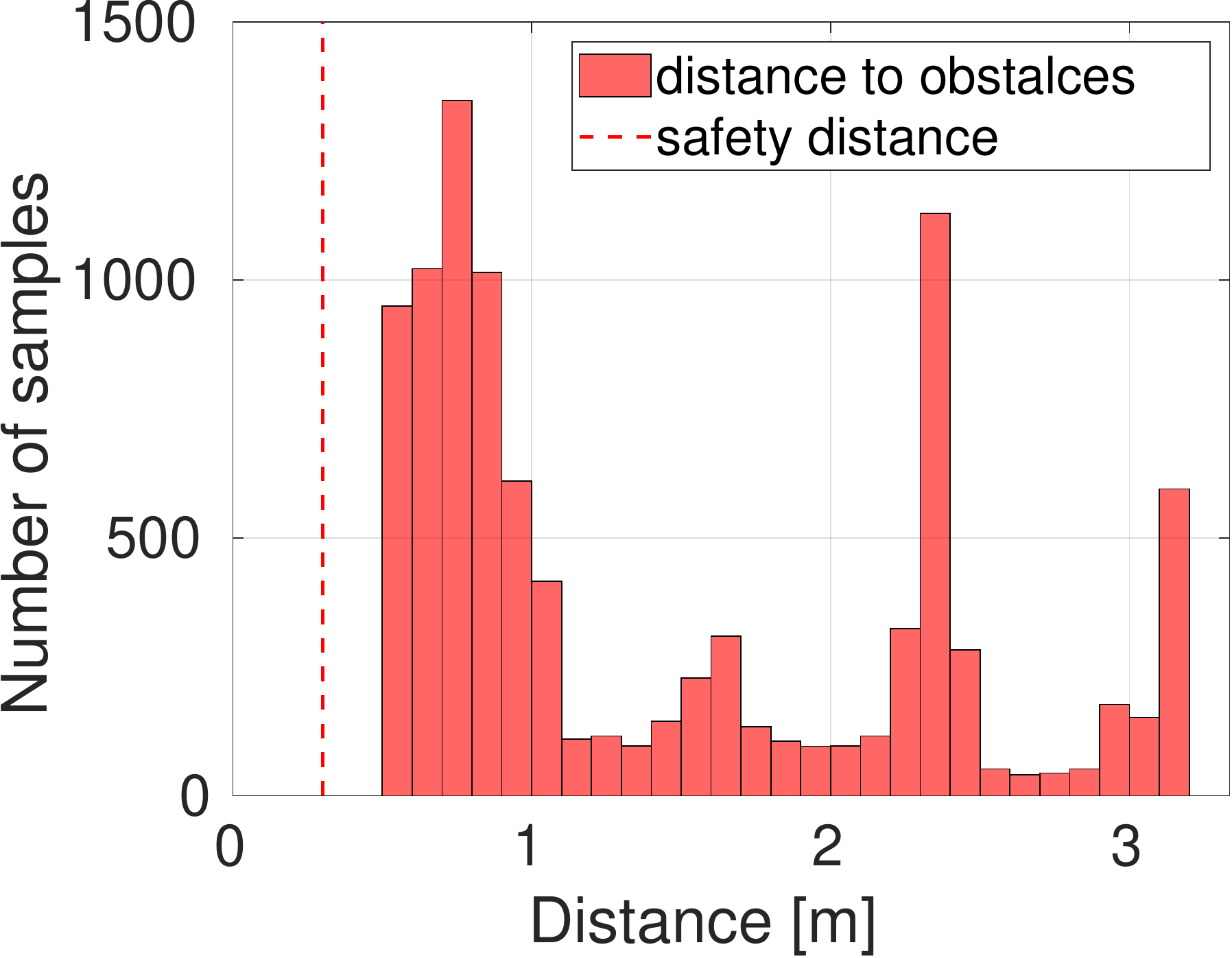}}

   \caption{Experimental results with two/three quadrotors with/without obstacles. (a) Histogram of inter-robot distance. (b) Histogram of distance between robots and obstacles. }%
   \label{fig:exp_quad_3_sta}
\end{figure}

\begin{figure*}[t]
   \centering
   \subfloat{\label{}
      \includegraphics[width=.24\textwidth]{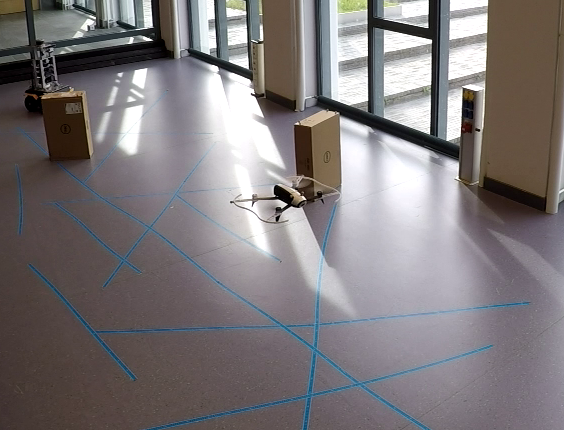}}
   \subfloat{\label{}
      \includegraphics[width=.24\textwidth]{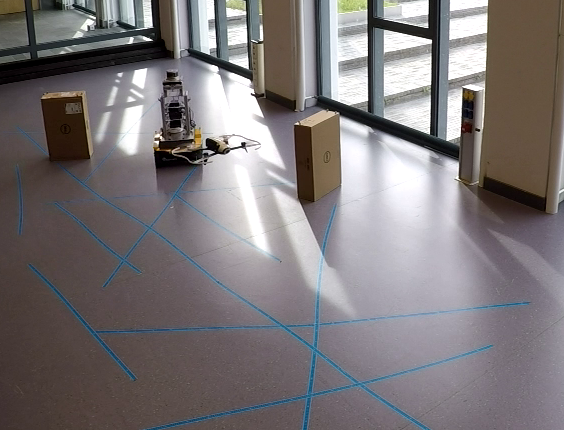}}
   \subfloat{\label{}
      \includegraphics[width=.24\textwidth]{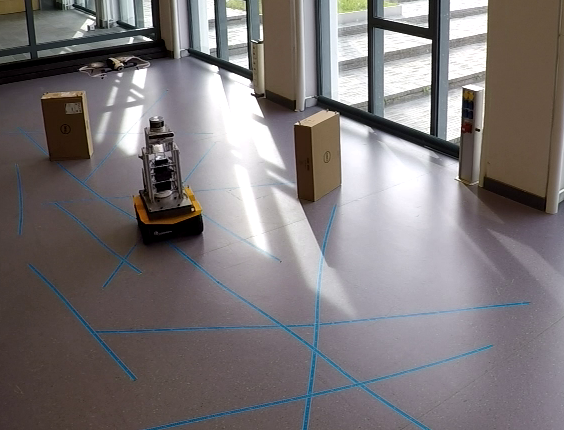}}
   \subfloat{\label{}
      \includegraphics[width=.24\textwidth]{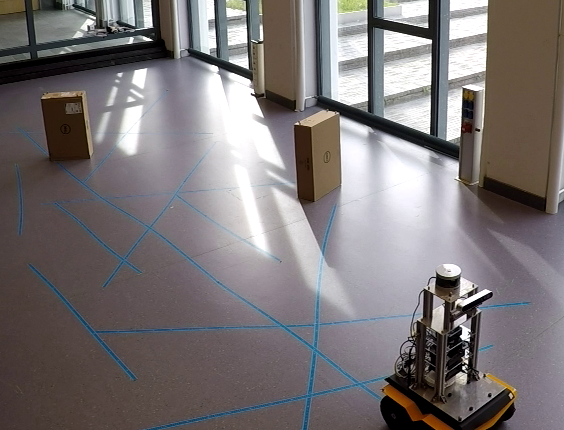}}
   \\
   \setcounter{subfigure}{0}
   \subfloat[$t$ = 0.1 s.]{\label{}
      \includegraphics[width=.24\textwidth]{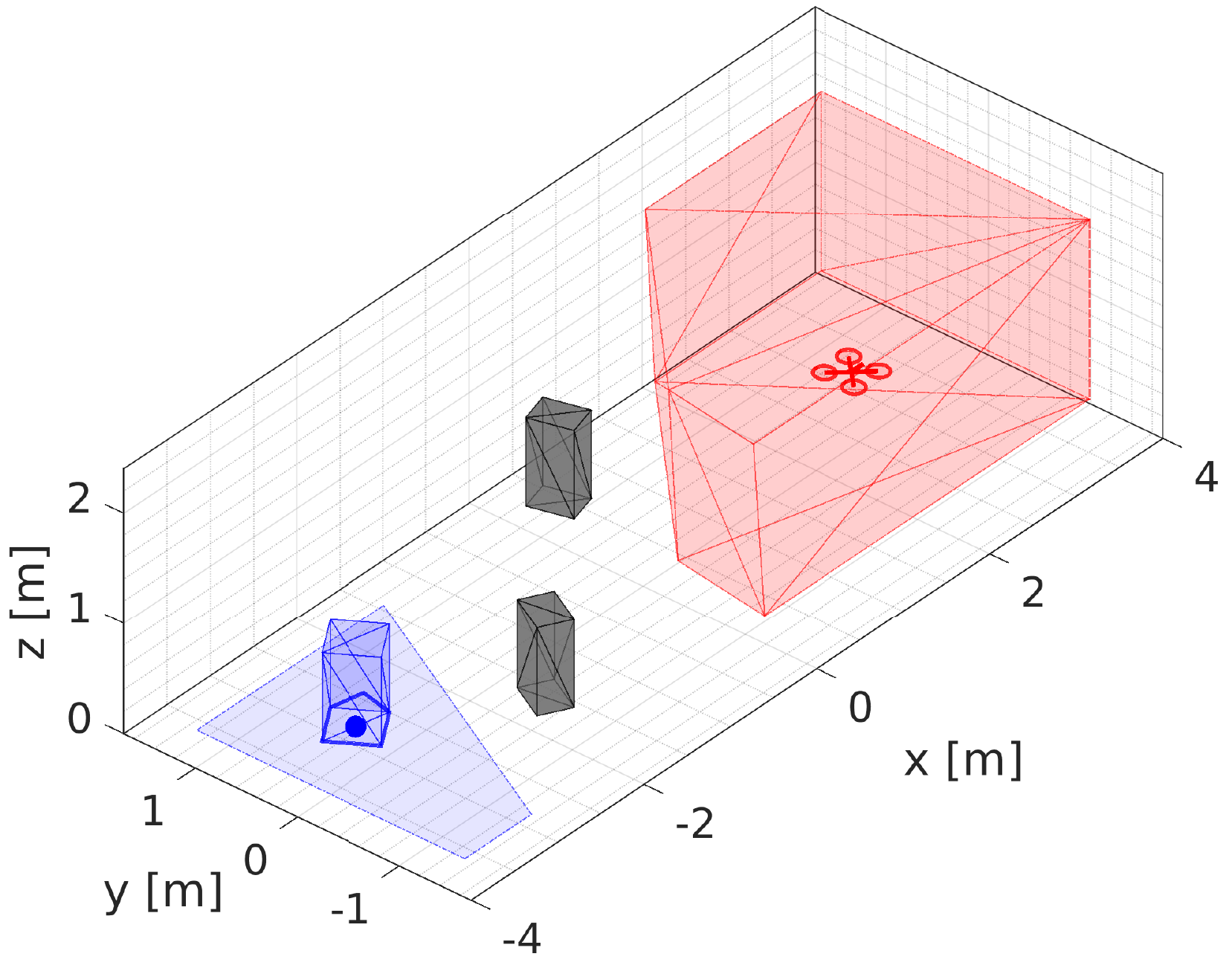}}
   \subfloat[$t$ = 4 s.]{\label{}
      \includegraphics[width=.24\textwidth]{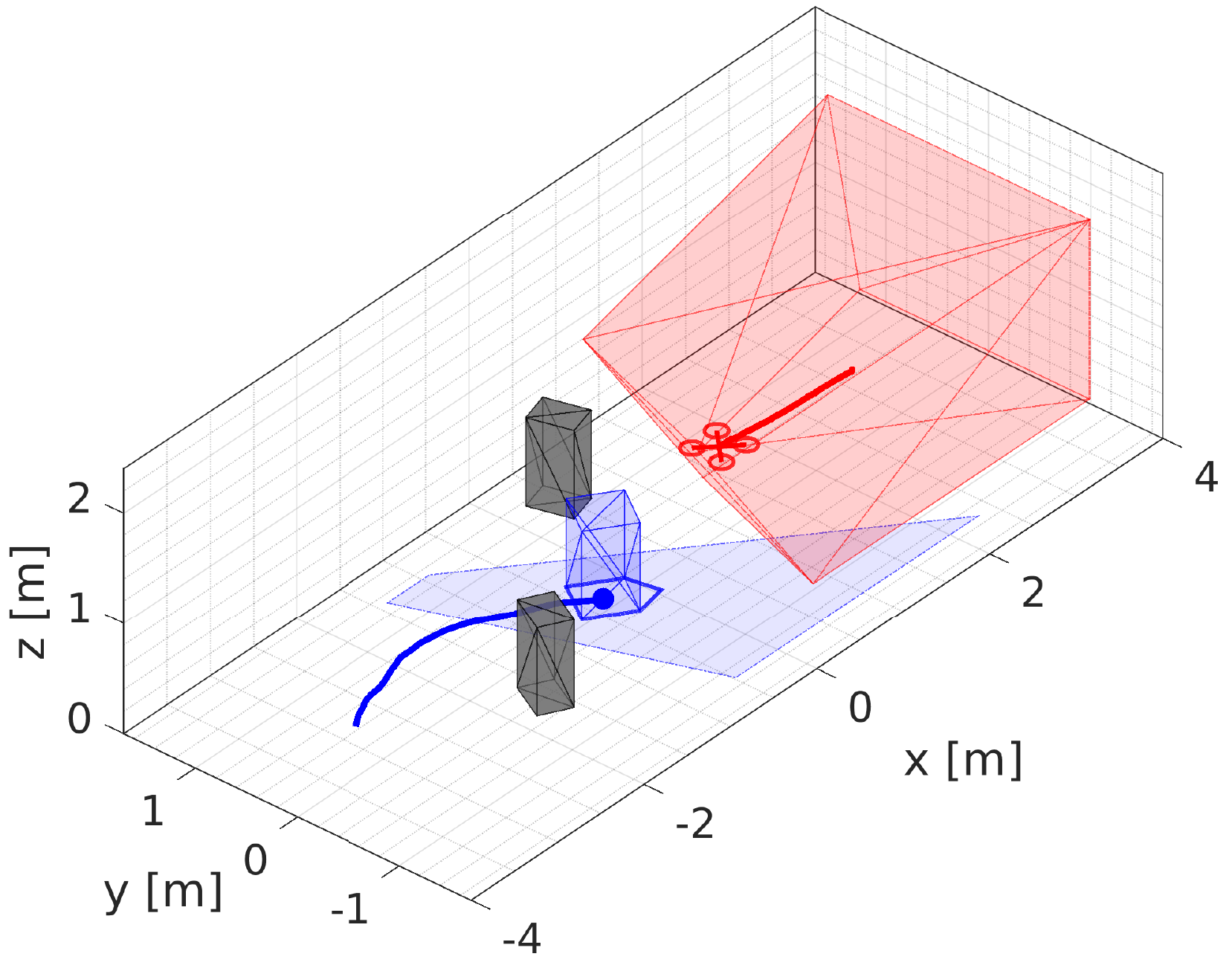}}
   \subfloat[$t$ = 6 s.]{\label{}
      \includegraphics[width=.24\textwidth]{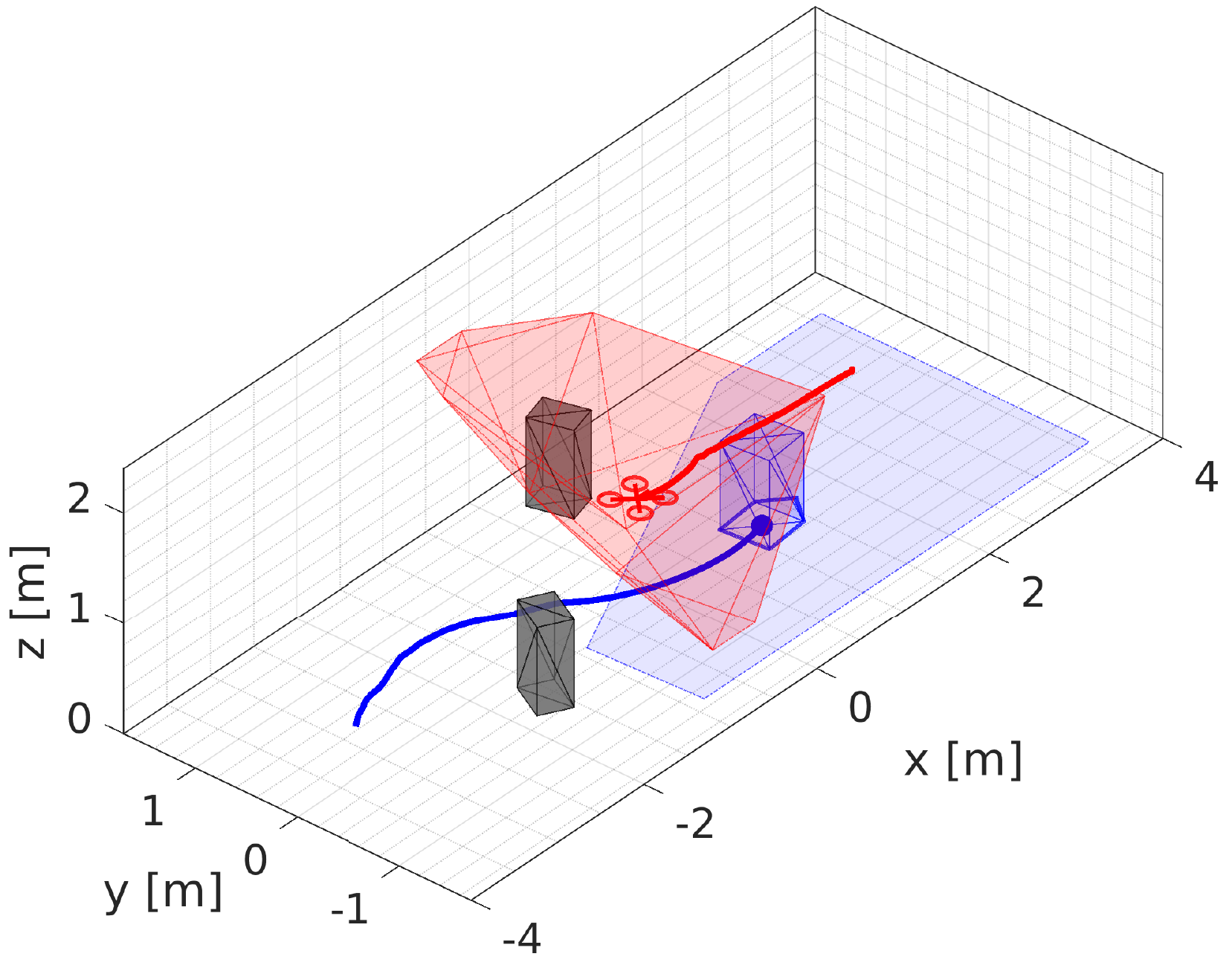}}
   \subfloat[$t$ = 12 s.]{\label{}
      \includegraphics[width=.24\textwidth]{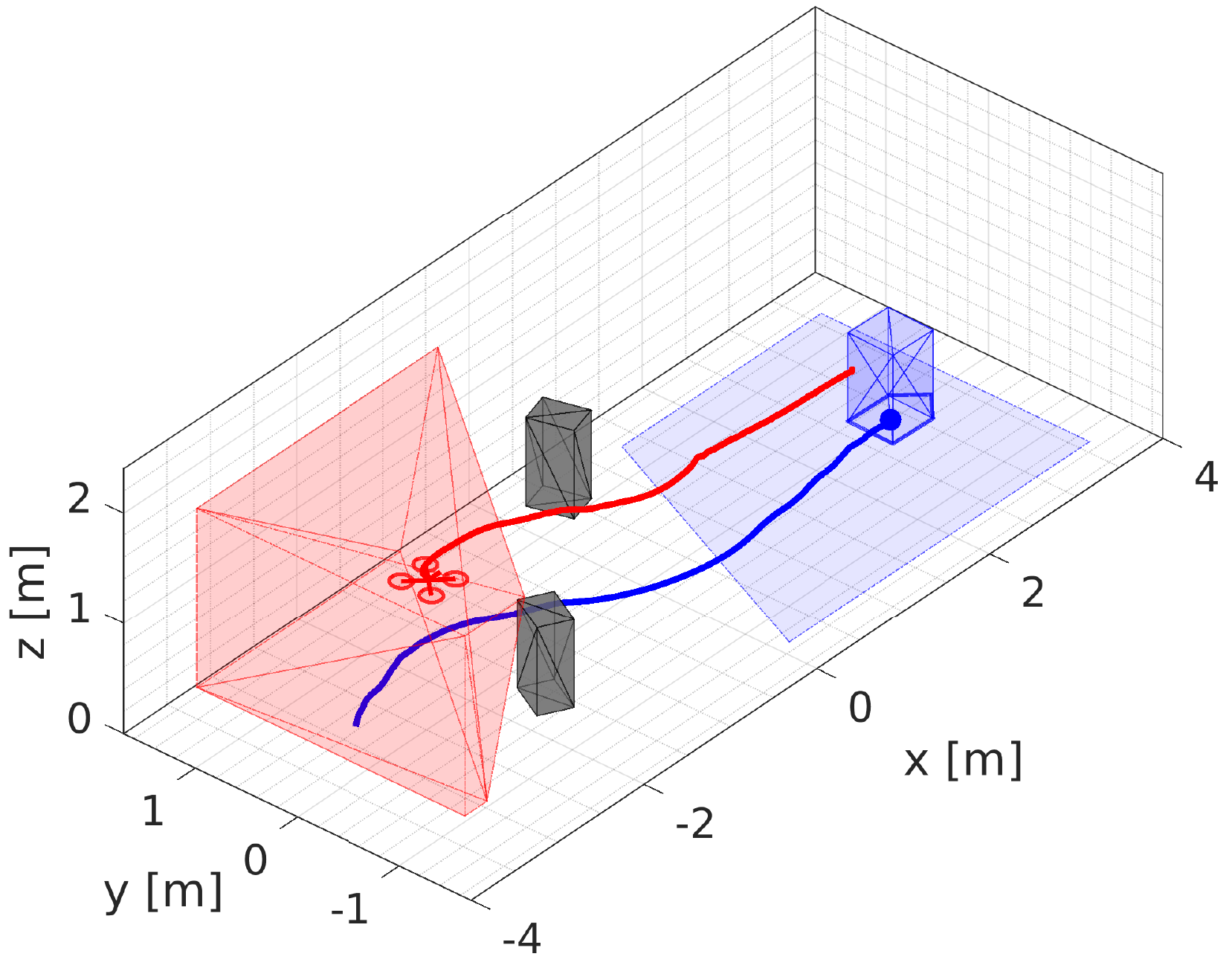}}

   \caption{Collision avoidance with a heterogeneous team of a differential-drive robot and a quadrotor. Top row: Snapshots of the experiment. Bottom row: Trajectories of the robots. }%
   \label{fig:exp_het}
\end{figure*}

\subsection{Experimental Results}

\subsubsection{Experiments with differential-drive robots in 2D}
We first validated our proposed approach with two differential-drive robots. In the experiment, two robots are required to swap their positions while avoiding two static obstacles in the environment. The robot safety radius is set as 0.3 m. We run the experiment four times. The two robots successfully navigated to their goals while avoiding each other as well as the obstacles in all runs. 

Fig. \ref{fig:exp_jackal} presents the results of one run. The top row of the figure shows a series of snapshots during the experiment, while the bottom row shows the robots' travelled trajectories and their corresponding B-UAVCs. It can be seen that each robot always keeps a very safe region (B-UAVC) taking into account its localization and sensing uncertainties. In Fig. \ref{fig:exp_jackal_sta} we cumulate the distance between the two robots (Fig. \ref{subfig:jackal_dis}) and distance between the robots and obstacles (Fig. \ref{subfig:jackal_dis_obs}) during the whole experiments. It can be seen that a minimum safe inter-robot distance of 0.6 m and a safe robot-obstacle distance of 0.3 m were maintained over all the runs. 

\subsubsection{Experiments with quadrotors in 3D}
We then performed experiments with a team of quadrotors in two scenarios: with and without static obstacles. The quadrotor safety radius is set as 0.3 m.

\paragraph{Scenario 1}
Two quadrotors swap their positions while avoiding two static obstacles in the environment. We performed the swapping action four times and Fig. \ref{fig:exp_quad_2} presents one run of the results.  

\paragraph{Scenario 2}
Three quadrotors fly in a confined space while navigating to different goal positions. The goal locations are randomly chosen such that the quadrotors' directions from initial positions towards goals are crossing. New goals are generated after all quadrotors reach their current goals. We run the experiment for a consecutive two minutes within which the goal of each quadrotor has been changed eight times. 

Fig. \ref{fig:exp_quad_3} presents a series of snapshots during the experiment. 
Fig. \ref{fig:exp_quad_3_sta} cumulates the inter-quadrotor distance in the experiments of both scenarios, and the distance between quadrotors and obstacles in Scenario 1. It can be seen that a minimum safety distance of 0.6 m among quadrotors and that of 0.3 m between quadrotors and obstacles were achieved during the whole experiments.

\subsubsection{Experiments with heterogeneous teams of robots}

We further tested our approach with one ground differential-drive robot and one quadrotor to show that it can be applied to heterogeneous robot teams. In the experiment, the ground robot only considers its motion and the obstacles in 2D (the ground plane) while ignoring the flying quadrotor. In contrast, the quadrotor considers both itself location and the ground robot's location as well as obstacles in 3D, in which it assumes the ground vehicle has a height of 0.6 m. To this end, the B-UAVC of the ground robot is a 2D convex region while that of the quadrotor is a 3D one. 

Fig. \ref{fig:exp_het} shows the results of the experiment. It can be seen that the two robots successfully reached their goals while avoiding each other and the static obstacles. Particularly at $t$ = 4 s, the quadrotor actively flies upward to avoid the ground robot. In Fig. \ref{fig:exp_hete_sta} we cumulate the distance between the two robots and the distance between robots and obstacles, which show that a safe inter-robot clearance of 0.6 m and that of 0.3 m between robots and obstacles were maintained during the experiment.

\begin{figure}[t]
    \centering
    \subfloat[]{\label{subfig:hete_dis_his}
        \includegraphics[width=.23\textwidth]{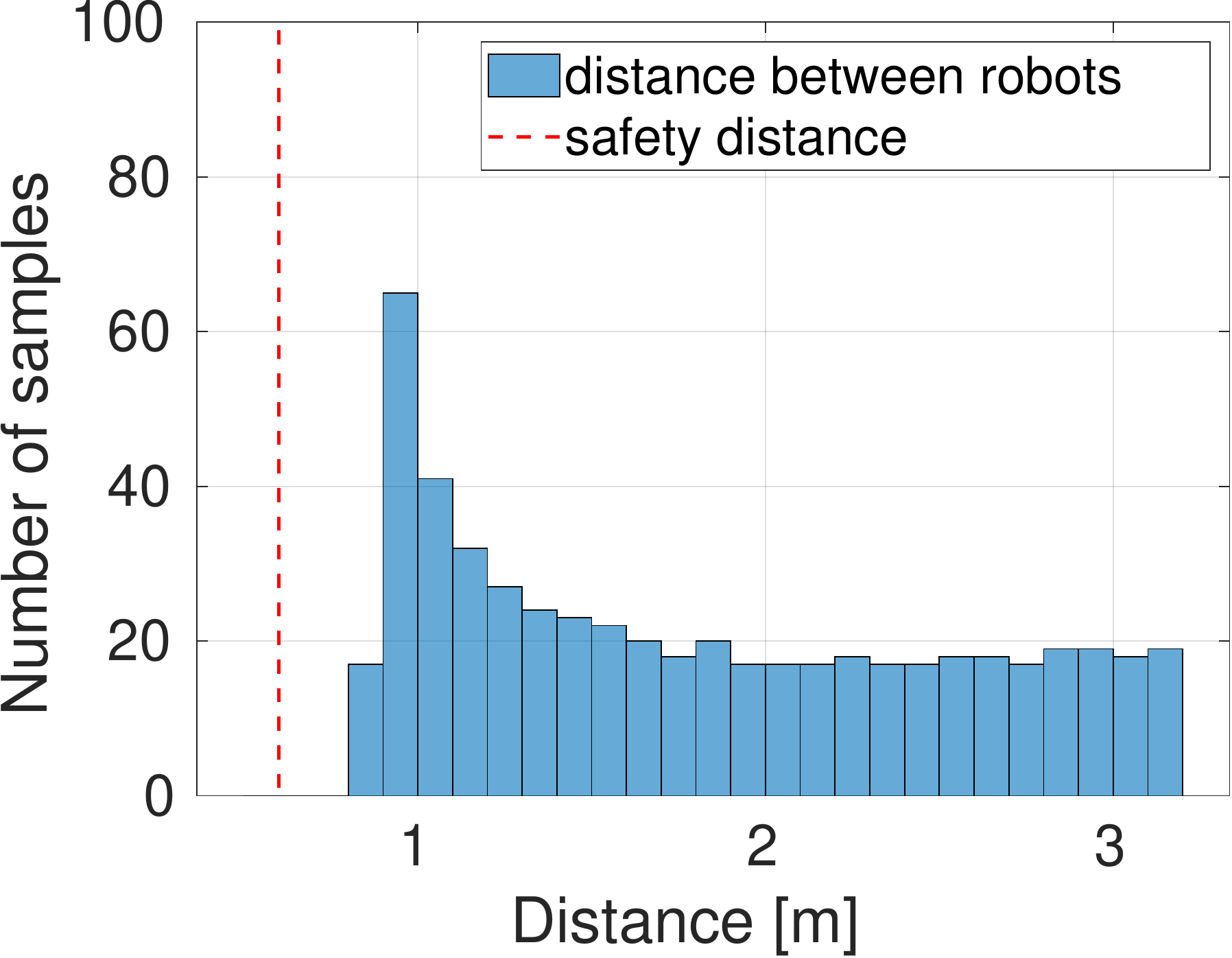}}
    \subfloat[]{\label{subfig:hete_obs_dis_his}
        \includegraphics[width=.23\textwidth]{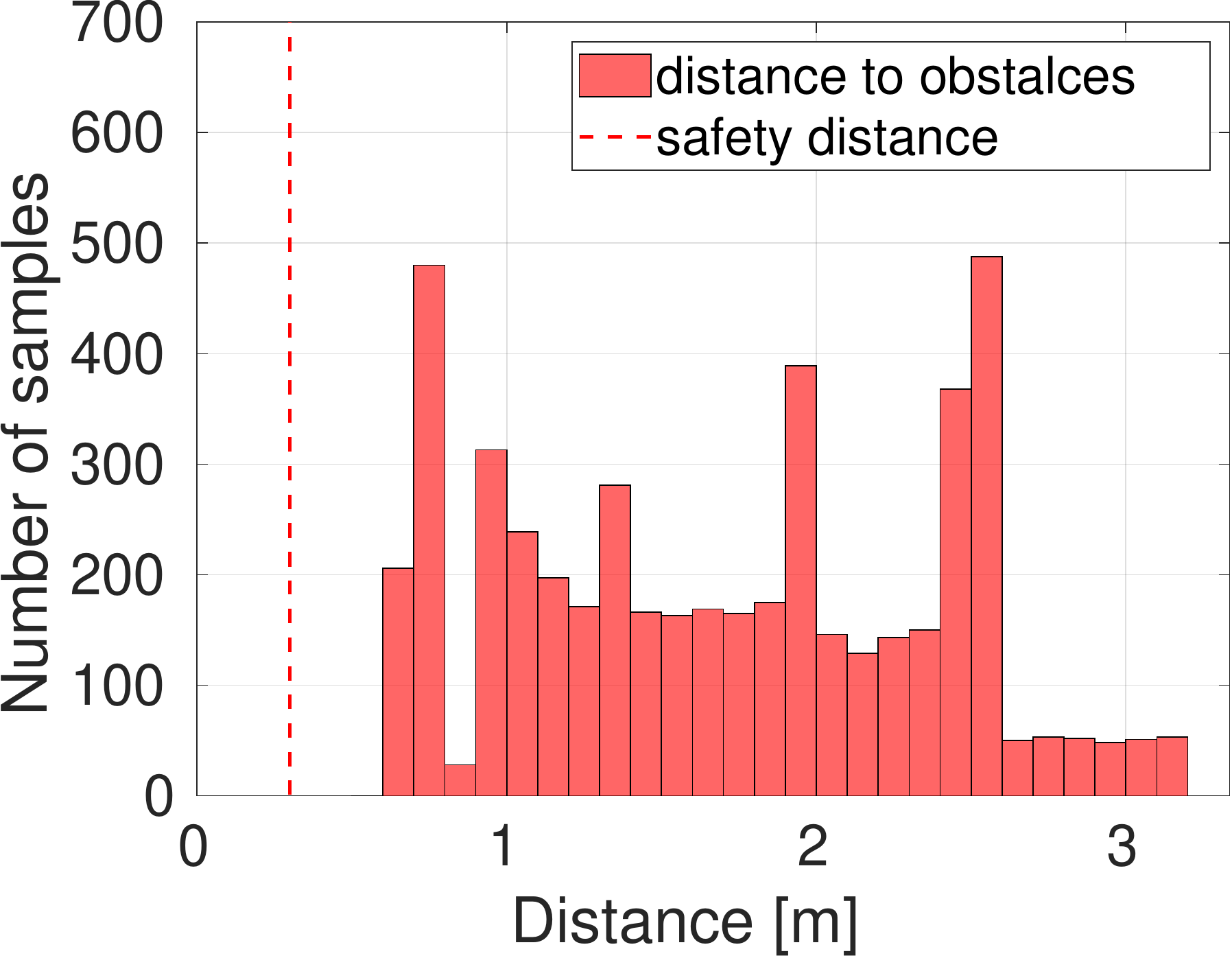}}

    \caption{Experimental results with a ground differential-drive robot and a quadrotor. (a) Histogram of inter-robot distance. (b) Histogram of distance between robots and obstacles. }%
    \label{fig:exp_hete_sta}
\end{figure}

\section{Conclusion}\label{sec:conclsuion}
In this paper we presented a decentralized and communication free multi-robot collision avoidance method that accounts for robot localization and sensing uncertainties. By assuming that the uncertainties are according to Gaussian distributions, we compute a chance-constrained buffered uncertainty-aware Voronoi cell (B-UAVC) for each robot among other robots and static obstacles. The probability of collision between robots and obstacles is guaranteed to be below a specified threshold by constraining each robot's motion to be within its corresponding B-UAVC. We apply the method to single-integrator, double-integrator, differential-drive, and general high-order dynamical multi-robot systems. 
In comparison with the BVC method, we showed that our method achieves robust safe navigation among a larger number of robots with noisy position measurements where the BVC approach will fail. 
In simulation with a team of quadrotors, we showed that our method achieves safer yet more conservative motions compared with the CCNMPC method, which is centralized and requires robots to communicate future trajectories. 
We also validated our method in extensive experiments with a team of ground vehicles, quadrotors, and heterogeneous robot teams in both obstacle-free and obstacles-clutter environments. 
Through simulations and experiments, two limitations of the proposed approach are also observed. 
The approach can achieve a high level of safety under robot localization and sensing uncertainty, however, it also leads to conservative behaviours of the robots, particulary for agile vehicles (quadrotors) in confined space. 
And, since the approach is local and efficient inter-robot coordination is not well investigated, deadlocks and livelocks may occure for large numbers of robots moving in complex environments.  

For future work, we plan to employ the proposed approach as a low-level robust collision-avoidance controller, and incorporate it with other higher-level multi-robot trajectory planning and coordination methods to achieve more efficient multi-robot navigation.

\rebuttal{
\section*{Appendix}
\def\thesection{\Alph{section}}
\setcounter{section}{0}

\section{Proofs of Lemmas and Theorems}

\subsection{Proof of Lemma \ref{lemma:probability_Do}}\label{appendix:proof_probability_Do}
\begin{proof}
    First we can write the random variable $\vd_o$ in an equivalent form $\vd_o = \Sigma_o^\prime\vd_o^\prime$, where $\vd_o^\prime\sim\Gau(0,I) \in \R^d$ and $\Sigma_o^{\prime}\Sigma_o^{\prime T} = \Sigma_o$.
    Note that $\vd_o^{\prime T}\vd_o^{\prime}$ is a chi-squared random variable with $d$ degrees of freedom. Hence, there is
    \begin{equation*}
        \pr(\vd_o^{\prime T}\vd_o^{\prime} \leq F^{-1}(1-\epsilon)) = 1 - \epsilon. 
    \end{equation*}
    Also note that $\Sigma_o^{-1} = (\Sigma_o^{\prime}\Sigma_o^{\prime T})^{-1} = \Sigma_o^{\prime T^{-1}}\Sigma_o^{\prime -1}$, thus $\vd_o^T\Sigma_o^{-1}\vd_o = \vd_o^{\prime T}\Sigma_o^{\prime T}\Sigma_o^{\prime T^{-1}}\Sigma_o^{\prime -1}\Sigma_o^\prime\vd_o^\prime = \vd_o^{\prime T}\vd_o^{\prime}$. Hence, it follows that $\pr(\vd_o^T\Sigma_o^{-1}\vd_o \leq F^{-1}(1-\epsilon)) = 1-\epsilon$. Thus, let $\cD_o = \{ \vd: \vd^T\Sigma_o^{-1}\vd \leq F^{-1}(1-\epsilon) \}$, there is $\pr(\vd_o \in\cD_o) = 1 - \epsilon$.
\end{proof}

\subsection{Proof of Theorem \ref{theorem:maximal_shadow}}\label{appendix:proof_maximal_shadow}
\begin{proof}
    We need to prove that the set $\cS_o$ contains the set $\cO_o$ with probability $1-\epsilon$. It is equivalent to that for any point in $\cO_o$, the set $\cS_o$ contains this point with probability $1-\epsilon$. Recall the definition of $\cO_o$, every $\vy \in \cO_o$ can be written as $\vx + \vd_o$ with some $\vx \in \hat{\cO}_o$. Also note the definition $\cS_o = \{\vx + \vd~|~\vx\in\hat{\cO}_o,\vd\in\cD_o\}$. Hence the probability that $\cS_o$ contains $\vy$ is equal to the probability that $\cD_o$ contains $\vd_o$. That is, $\pr(\vy \in \cS_o) = \pr(\vd_o \in \cD_o)= 1 - \epsilon, \forall \vy \in \cO_o$. Thus, $\pr(\cO_o \subseteq \cS_o) = 1-\epsilon$. $\cS_o$ is a maximal $\epsilon$-shadow of $\cO_o$.

\end{proof}

\section{Procedure to Compute the Best Linear Separator Between Two Gaussian Distributions}\label{appendix:best_linear_separator}
The objective is to solve the following minimax problem:
\begin{equation*}
    (\va_{ij}, b_{ij}) = \arg\underset{\va_{ij}\in\R^d,b_{ij}\in\R}{\min\max}(\pr_i, \pr_j),
\end{equation*}
where 
\begin{equation*}
    \begin{aligned}
        \pr_i(\va_{ij}^T\vp > b_{ij}) &= 1 - \Phi((b_{ij} - \va_{ij}^T\hp_i)/\sqrt{\va_{ij}^T\Sigma_i\va_{ij}}), \\ 
        \pr_j(\va_{ij}^T\vp \leq b_{ij}) &= 1 - \Phi((\va_{ij}^T\hp_j - b_{ij})/\sqrt{\va_{ij}^T\Sigma_j\va_{ij}}).
    \end{aligned}
\end{equation*}
Let $u_1 = \frac{b_{ij} - \va_{ij}^T\hp_i}{\sqrt{\va_{ij}^T\Sigma_i\va_{ij}}}$, $u_2 = \frac{\va_{ij}^T\hp_j - b_{ij}}{\sqrt{\va_{ij}^T\Sigma_j\va_{ij}}}$. As the function $\Phi(\cdot)$ is monotonic, the original minimax problem is equivalent to 
\begin{equation*}
    (\va_{ij}, b_{ij}) = \arg\underset{\va_{ij}\in\R^d,b_{ij}\in\R}{\max\min}(u_1, u_2).
\end{equation*}
We can write $u_1$ in the following form for a given $u_2$,
\begin{equation*}
    u_1 = \frac{\va_{ij}^T\hp_{ij} - u_2\sqrt{\va_{ij}^T\Sigma_j\va_{ij}}}{\sqrt{\va_{ij}^T\Sigma_i\va_{ij}}},
\end{equation*}
where $\hp_{ij} = \hp_j - \hp_i$. For each given $u_2$, $u_1$ needs to be maximized. Hence, we can differentiate the above equation with respect to $\va_{ij}$ and set the derivative to equal to zero, which leads to 
\begin{equation}\label{eq:procedure_a_ij}
    \va_{ij} = [t\Sigma_i + (1-t)\Sigma_j]^{-1}\hp_{ij},
\end{equation}
where $t \in (0,1)$ is a scaler. Thus according to definition of $u_1$ and $u_2$, we have 
\begin{equation}\label{eq:procedure_b_ij}
    b_{ij} = \va_{ij}^T\hp_i + t\va_{ij}^T\Sigma_i\va_{ij} = \va_{ij}^T\hp_j - (1-t)\va_{ij}^T\Sigma_j\va_{ij}.
\end{equation}
It is proved that $u_1 = u_2$ must be hold for the solution of the minimax problem \citep{Anderson1962}, which leads to 
\begin{equation}\label{eq:procedure_t}
    \va_{ij}^T[t^2\Sigma_i - (1-t)^2\Sigma_j]\va_{ij} = 0.
\end{equation}
Thus, one can first solve for $t$ by combining Eqs. (\ref{eq:procedure_a_ij}) and (\ref{eq:procedure_t}) via numerical iteration efficiently. Then $\va_{ij}$ and $b_{ij}$ can be computed using Eqs. (\ref{eq:procedure_a_ij}) and (\ref{eq:procedure_b_ij}).

\section{Deadlock Resolution Heuristic}\label{appendix:deadlock}
We detect and resolve deadlocks in a heuristic way in this paper. Let $\norm{\Delta\vp_i}$ be the position progress between two consecutive time steps of robot $i$, and $\Delta\vp_{\min}$ a predefined minimum allowable progress distance for the robot in $n_{\tn{dead}}$ time steps. If the robot has not reached its goal and $\Sigma_{n_{\tn{dead}}}\norm{\Delta\vp_i} \leq \Delta\vp_{\min}$, we consider the robot as in a deadlock situation. 
For the one-step controller, each robot must be at the ``projected goal'' $\vg_i^*$ when the system is in a deadlock configuration \citep{Zhou2017}. In this case, each robot chooses one of the nearby edges within its B-UAVC to move along. For receding horizon planning of high-order dynamical systems, the robot may get stuck due to a local minima of the trajectory optimization problem. In this case, we temporarily change the goal location $\vg_i$ of each robot by clockwise rotating it along the $z$ axis with $90\degree$, i.e.
\begin{equation}
    \vg_{i,\tn{temp}} = R_Z(-90\degree)(\vg_i - \hp_i) + \hp_i, 
\end{equation}
where $R_Z$ denotes the rotation matrix for rotations around $z$-axis. This temporary rotation will change the objective of the trajectory optimization problem, thus helping the robot to recover from a local minima. Once the robot recovers from stuck, its goal is changed back to $\vg_i$. 

Similar to most heuristic deadlock resolutions, the solutions presented here can not guarantee that all robots will eventually reach their goals since livelocks
(robots continuously repeat a sequence of behaviors that bring them from one deadlock situation to another one) 
may still occur.

\section{Quadrotor Dynamics Model}\label{appendix:quad_model}
We use the Parrot Bebop 2 quadrotor in our experiments. The state of the quadrotor is 
\begin{equation*}
    \vx = [\vp^T, \vv^T, \phi, \theta, \psi]^T \in \R^9,
\end{equation*}
where $\vp = [p_x, p_y, p_z]^T \in \R^3$ is the position, $\vv = [v_x, v_y, v_z]^T \in \R^3$ the velocity, and $\phi, \theta, \psi$ the roll, pitch and yaw angles of the quadrotor. The control inputs to the quadrotor are 
\begin{equation*}
    \vu = [\phi_c, \theta_c, v_{z_c}, \dot{\psi}_c]^T \in \R^4,
\end{equation*}
where $\phi_c$ and $\theta_c$ are commanded roll and pitch angles, $v_{z_c}$ the commanded velocity in vertical $z$ direction, and $\dot{\psi}_c$ the commanded yaw rate. 

The dynamics of the quadrotor position and velocity are 
\begin{equation*}
    \begin{cases}
        \dot\vp = \vv, \\ 
        \mat \dot{v}_x \\ \dot{v}_y \mate = R_Z(\psi)\mat \tan\theta \\ -\tan\phi \mate g - \mat k_{D_x}v_x \\ k_{D_y}v_y \mate, \\ 
        \dot{v}_z = \frac{1}{\tau_{v_z}}(k_{v_z}v_{z_c} - v_z),
    \end{cases}
\end{equation*}
where $g = 9.81~\tn{m}/\tn{s}^2$ is the Earth's gravity, $R_Z(\psi) = \mat \cos\psi &-\sin\psi \\ \sin\psi &\cos\psi \mate$ is the rotation matrix along the $z$-body axis, $k_{D_x}$ and $k_{D_y}$ the drag coefficient, $k_{v_z}$ and $\tau_{v_z}$ the gain and time constant of vertical velocity control. 

The attitude dynamics of the quadrotor are 
\begin{equation*}
    \begin{cases}
        \dot\phi = \frac{1}{\tau_{\phi}}(k_{\phi}\phi_c - \phi), \\
        \dot\theta = \frac{1}{\tau_{\theta}}(k_{\theta}\theta_c - \theta), \\
        \dot\psi = \dot{\psi}_c, 
    \end{cases}
\end{equation*}
where $k_{\phi}, k_{\theta}$ and $\tau_{\phi}, \tau_{\theta}$ are the gains and time constants of roll and pitch angles control respectively. 

We obtained the dynamics model parameters $k_{D_x} = 0.25$, $k_{D_y} = 0.33$, $k_{v_z}=1.2270$, $\tau_{v_z}=0.3367$, $k_{\phi}=1.1260$, $\tau_{\phi}=0.2368$, $k_{\theta}=1.1075$ and $\tau_{\theta}=0.2318$ by collecting real flying data and performing system identification. 

}

\bibliographystyle{spbasic}      
\balance
\bibliography{ref}   

\begin{thebibliography}{41}
\providecommand{\natexlab}[1]{#1}
\providecommand{\url}[1]{{#1}}
\providecommand{\urlprefix}{URL }
\expandafter\ifx\csname urlstyle\endcsname\relax
  \providecommand{\doi}[1]{DOI~\discretionary{}{}{}#1}\else
  \providecommand{\doi}{DOI~\discretionary{}{}{}\begingroup
  \urlstyle{rm}\Url}\fi
\providecommand{\eprint}[2][]{\url{#2}}

\bibitem[{Alonso-Mora et~al.(2018)Alonso-Mora, Beardsley, and
  Siegwart}]{Alonso-Mora2018}
Alonso-Mora J, Beardsley P, Siegwart R (2018) {Cooperative collision avoidance
  for nonholonomic robots}. \emph{IEEE Transactions on Robotics},
  34(2):404--420

\bibitem[{Anderson and Bahadur(1962)}]{Anderson1962}
Anderson TW, Bahadur RR (1962) {Classification into two multivariate normal
  distributions with different covariance matrices}. \emph{The Annals of
  Mathematical Statistics}, 33(2):420--431

\bibitem[{Andrews(1997)}]{Andrews1997}
Andrews LC (1997) \emph{{Special functions of mathematics for engineers}},
  vol~49. SPIE press

\bibitem[{Arslan and Koditschek(2019)}]{Arslan2019}
Arslan O, Koditschek DE (2019) {Sensor-based reactive navigation in unknown
  convex sphere worlds}. \emph{International Journal of Robotics Research},
  38(2-3):196--223

\bibitem[{Astolfi(1999)}]{Astolfi1999}
Astolfi A (1999) {Exponential stabilization of a wheeled mobile robot via
  discontinuous control}. \emph{Journal of Dynamic Systems, Measurement and
  Control, Transactions of the ASME}, 121(1):121--126

\bibitem[{Axelrod et~al.(2018)Axelrod, Kaelbling, and
  Lozano-P{\'{e}}rez}]{Axelrod2018}
Axelrod B, Kaelbling LP, Lozano-P{\'{e}}rez T (2018) {Provably safe robot
  navigation with obstacle uncertainty}. \emph{The International Journal of
  Robotics Research}, 37(13-14):1760--1774

\bibitem[{Bareiss and van~den Berg(2015)}]{Bareiss2015}
Bareiss D, van~den Berg J (2015) Generalized reciprocal collision avoidance.
  \emph{The International Journal of Robotics Research}, 34(12):1501--1514

\bibitem[{Van~den Berg et~al.(2008)Van~den Berg, Lin, and
  Manocha}]{VandenBerg2008}
Van~den Berg J, Lin M, Manocha D (2008) Reciprocal velocity obstacles for
  real-time multi-agent navigation. \emph{In: 2008 IEEE International
  Conference on Robotics and Automation (ICRA)}, IEEE, pp 1928--1935

\bibitem[{Blackmore et~al.(2011)Blackmore, Ono, and Williams}]{Blackmore2011}
Blackmore L, Ono M, Williams BC (2011) {Chance-constrained optimal path
  planning with obstacles}. \emph{IEEE Transactions on Robotics},
  27(6):1080--1094

\bibitem[{Breitenmoser and Martinoli(2016)}]{Breitenmoser2016}
Breitenmoser A, Martinoli A (2016) {On Combining Multi-robot Coverage and
  Reciprocal Collision Avoidance}. \emph{In: Springer Tracts in Advanced
  Robotics}, vol 112, Springer Japan, Tokyo, pp 49--64

\bibitem[{Chen et~al.(2015)Chen, Cutler, and How}]{Chen2015}
Chen Y, Cutler M, How JP (2015) {Decoupled multiagent path planning via
  incremental sequential convex programming}. \emph{In: 2015 IEEE International
  Conference on Robotics and Automation (ICRA)}, IEEE, pp 5954--5961

\bibitem[{Claes et~al.(2012)Claes, Hennes, Tuyls, and Meeussen}]{Claes2012}
Claes D, Hennes D, Tuyls K, Meeussen W (2012) {Collision avoidance under
  bounded localization uncertainty}. \emph{In: 2012 IEEE/RSJ International
  Conference on Intelligent Robots and Systems (IROS)}, IEEE, pp 1192--1198

\bibitem[{Dawson et~al.(2020)Dawson, Jasour, Hofmann, and
  Williams}]{Dawson2020IROS}
Dawson C, Jasour A, Hofmann A, Williams B (2020) {Provably Safe Trajectory
  Optimization in the Presence of Uncertain Convex Obstacles}. \emph{In: 2020
  IEEE/RSJ International Conference on Intelligent Robots and Systems (IROS)},
  IEEE, pp 6237--6244

\bibitem[{Deits and Tedrake(2015{\natexlab{a}})}]{Deits2015}
Deits R, Tedrake R (2015{\natexlab{a}}) {Computing large convex regions of
  obstacle-free space through semidefinite programming}. \emph{In: Springer
  Tracts in Advanced Robotics}, vol 107, pp 109--124

\bibitem[{Deits and Tedrake(2015{\natexlab{b}})}]{deits2015efficient}
Deits R, Tedrake R (2015{\natexlab{b}}) Efficient mixed-integer planning for
  uavs in cluttered environments. \emph{In: 2015 IEEE international conference
  on robotics and automation (ICRA)}, IEEE, pp 42--49

\bibitem[{Fiorini and Shiller(1998)}]{Fiorini1998}
Fiorini P, Shiller Z (1998) {Motion planning in dynamic environments using
  velocity obstacles}. \emph{The International Journal of Robotics Research},
  17(7):760--772

\bibitem[{Gopalakrishnan et~al.(2017)Gopalakrishnan, Singh, Kaushik, Krishna,
  and Manocha}]{Gopalakrishnan2017}
Gopalakrishnan B, Singh AK, Kaushik M, Krishna KM, Manocha D (2017) Prvo:
  Probabilistic reciprocal velocity obstacle for multi robot navigation under
  uncertainty. \emph{In: 2017 IEEE/RSJ International Conference on Intelligent
  Robots and Systems (IROS)}, IEEE, pp 1089--1096

\bibitem[{Hardy and Campbell(2013)}]{Hardy2013}
Hardy J, Campbell M (2013) {Contingency planning over probabilistic obstacle
  predictions for autonomous road vehicles}. \emph{IEEE Transactions on
  Robotics}, 29(4):913--929

\bibitem[{H\"onig et~al.(2018)H\"onig, Preiss, Kumar, Sukhatme, and
  Ayanian}]{Honig2018}
H\"onig W, Preiss JA, Kumar TK, Sukhatme GS, Ayanian N (2018) {Trajectory
  planning for quadrotor swarms}. \emph{IEEE Transactions on Robotics},
  34(4):856--869

\bibitem[{Kamel et~al.(2017)Kamel, Alonso-Mora, Siegwart, and
  Nieto}]{Kamel2017}
Kamel M, Alonso-Mora J, Siegwart R, Nieto J (2017) {Robust collision avoidance
  for multiple micro aerial vehicles using nonlinear model predictive control}.
  \emph{In: 2017 IEEE/RSJ International Conference on Intelligent Robots and
  Systems (IROS)}, IEEE, pp 236--243

\bibitem[{Kozlov et~al.(1980)Kozlov, Tarasov, and
  Khachiyan}]{kozlov1980polynomial}
Kozlov MK, Tarasov SP, Khachiyan LG (1980) The polynomial solvability of convex
  quadratic programming. \emph{USSR Computational Mathematics and Mathematical
  Physics}, 20(5):223--228

\bibitem[{Liu et~al.(2017)Liu, Watterson, Mohta, Sun, Bhattacharya, Taylor, and
  Kumar}]{Liu2017}
Liu S, Watterson M, Mohta K, Sun K, Bhattacharya S, Taylor CJ, Kumar V (2017)
  {Planning dynamically feasible trajectories for quadrotors using safe flight
  corridors in 3-d complex environments}. \emph{IEEE Robotics and Automation
  Letters}, 2(3):1688--1695

\bibitem[{Luis et~al.(2020)Luis, Vukosavljev, and Schoellig}]{Luis2020}
Luis CE, Vukosavljev M, Schoellig AP (2020) {Online trajectory generation with
  distributed model predictive control for multi-robot motion planning}.
  \emph{IEEE Robotics and Automation Letters}, 5(2):604--611

\bibitem[{Luo et~al.(2020)Luo, Sun, and Kapoor}]{luo2020multi}
Luo W, Sun W, Kapoor A (2020) Multi-robot collision avoidance under uncertainty
  with probabilistic safety barrier certificates. \emph{In: 2020 Advances in
  Neural Information Processing Systems (NeurIPS)}, vol~33

\bibitem[{Lyons et~al.(2012)Lyons, Calliess, and Hanebeck}]{Lyons2012}
Lyons D, Calliess J, Hanebeck UD (2012) {Chance constrained model predictive
  control for multi-agent systems with coupling constraints}. \emph{In: 2012
  American Control Conference (ACC)}, IEEE, pp 1223--1230

\bibitem[{Morgan et~al.(2016)Morgan, Subramanian, Chung, and
  Hadaegh}]{Morgan2016}
Morgan D, Subramanian GP, Chung SJ, Hadaegh FY (2016) {Swarm assignment and
  trajectory optimization using variable-swarm, distributed auction assignment
  and sequential convex programming}. \emph{International Journal of Robotics
  Research}, 35(10):1261--1285

\bibitem[{N{\"{a}}geli et~al.(2017)N{\"{a}}geli, Meier, Domahidi, Alonso-Mora,
  and Hilliges}]{Nageli2017multiple}
N{\"{a}}geli T, Meier L, Domahidi A, Alonso-Mora J, Hilliges O (2017)
  {Real-time planning for automated multi-view drone cinematography}. \emph{ACM
  Transactions on Graphics}, 36(4):1--10

\bibitem[{Okabe et~al.(2009)Okabe, Boots, Sugihara, and Chiu}]{Okabe2009}
Okabe A, Boots B, Sugihara K, Chiu SN (2009) \emph{Spatial tessellations:
  Concepts and applications of Voronoi diagrams}. John Wiley \& Sons

\bibitem[{Pierson et~al.(2020)Pierson, Schwarting, Karaman, and
  Rus}]{pierson2020weighted}
Pierson A, Schwarting W, Karaman S, Rus D (2020) Weighted buffered voronoi
  cells for distributed semi-cooperative behavior. \emph{In: 2020 IEEE
  International Conference on Robotics and Automation (ICRA)}, IEEE, pp
  5611--5617

\bibitem[{Schmerling and Pavone(2017)}]{Schmerling2017}
Schmerling E, Pavone M (2017) {Evaluating trajectory collision probability
  through adaptive importance sampling for safe motion planning}. \emph{In:
  Robotics: Science and Systems}, vol~13

\bibitem[{Serra-G\'{o}mez et~al.(2020)Serra-G\'{o}mez, Brito, Zhu, Chung, and
  Alonso-Mora}]{Serra2020}
Serra-G\'{o}mez A, Brito B, Zhu H, Chung JJ, Alonso-Mora J (2020) With whom to
  communicate: Learning efficient communication for multi-robot collision
  avoidance. \emph{In: 2020 IEEE/RSJ International Conference on Intelligent
  Robots and Systems (IROS)}, IEEE, pp 11770--11776

\bibitem[{Shim et~al.(2003)Shim, Kim, and Sastry}]{Shim2003}
Shim D, Kim H, Sastry S (2003) {Decentralized nonlinear model predictive
  control of multiple flying robots}. \emph{In: 2003 IEEE Conference on
  Decision and Control (CDC)}, IEEE, pp 3621--3626

\bibitem[{Tordesillas et~al.(2019)Tordesillas, Lopez, and
  How}]{tordesillas2019faster}
Tordesillas J, Lopez BT, How JP (2019) Faster: Fast and safe trajectory planner
  for flights in unknown environments. \emph{In: 2019 IEEE/RSJ International
  Conference on Intelligent Robots and Systems (IROS)}, IEEE, pp 1934--1940

\bibitem[{{Van Den Berg} et~al.(2011){Van Den Berg}, Guy, Lin, and
  Manocha}]{VanDenBerg2011}
{Van Den Berg} J, Guy SJ, Lin M, Manocha D (2011) {Reciprocal n-body collision
  avoidance}. \emph{In: Springer Tracts in Advanced Robotics}, vol~70, pp 3--19

\bibitem[{Wang and Schwager(2019)}]{Wang2019}
Wang M, Schwager M (2019) {Distributed collision avoidance of multiple robots
  with probabilistic buffered voronoi cells}. \emph{In: 2019 International
  Symposium on Multi-Robot and Multi-Agent Systems (MRS)}, IEEE, pp 169--175

\bibitem[{Zanelli et~al.(2020)Zanelli, Domahidi, Jerez, and
  Morari}]{Zanelli2020}
Zanelli A, Domahidi A, Jerez J, Morari M (2020) {FORCES NLP: an efficient
  implementation of interior-point methods for multistage nonlinear nonconvex
  programs}. \emph{International Journal of Control}, (1):13--29

\bibitem[{Zhou et~al.(2017)Zhou, Wang, Bandyopadhyay, and Schwager}]{Zhou2017}
Zhou D, Wang Z, Bandyopadhyay S, Schwager M (2017) {Fast, on-line collision
  avoidance for dynamic vehicles using buffered voronoi cells}. \emph{IEEE
  Robotics and Automation Letters}, 2(2):1047--1054

\bibitem[{Zhou et~al.(2018)Zhou, Tzoumas, Pappas, and
  Tokekar}]{zhou2018resilient}
Zhou L, Tzoumas V, Pappas GJ, Tokekar P (2018) Resilient active target tracking
  with multiple robots. \emph{IEEE Robotics and Automation Letters},
  4(1):129--136

\bibitem[{Zhu and Alonso-Mora(2019{\natexlab{a}})}]{Zhu2019MRS}
Zhu H, Alonso-Mora J (2019{\natexlab{a}}) B-uavc: Buffered uncertainty-aware
  voronoi cells for probabilistic multi-robot collision avoidance. \emph{In:
  2019 International Symposium on Multi-Robot and Multi-Agent Systems (MRS)},
  IEEE, pp 162--168

\bibitem[{Zhu and Alonso-Mora(2019{\natexlab{b}})}]{Zhu2019RAL}
Zhu H, Alonso-Mora J (2019{\natexlab{b}}) {Chance-constrained collision
  avoidance for mavs in dynamic environments}. \emph{IEEE Robotics and
  Automation Letters}, 4(2):776--783

\bibitem[{Zhu et~al.(2019)Zhu, Juhl, Ferranti, and Alonso-Mora}]{Zhu2019ICRA}
Zhu H, Juhl J, Ferranti L, Alonso-Mora J (2019) {Distributed multi-robot
  formation splitting and merging in dynamic environments}. \emph{In: 2019
  International Conference on Robotics and Automation (ICRA)}, IEEE, pp
  9080--9086

\end{thebibliography}

\end{document}